\documentclass[11pt,english]{article}

\usepackage[a4paper,
            bindingoffset=0.2in,
            left=1in,
            right=1in,
            top=1in,
            bottom=1in,
            footskip=.25in]{geometry}
\usepackage[utf8]{inputenc} % allow utf-8 input
\usepackage[T1]{fontenc}    % use 8-bit T1 fonts
\usepackage{hyperref}       % hyperlinks
\usepackage{url}            % simple URL typesetting
\usepackage{booktabs}       % professional-quality tables
\usepackage{amsfonts}       % blackboard math symbols
\usepackage{nicefrac}       % compact symbols for 1/2, etc.
\usepackage{microtype}      % microtypography
\usepackage{amsmath}
\usepackage{amssymb}
\usepackage{graphicx}
\usepackage{amsthm} % For theorems and definitions
\usepackage{graphicx}
\usepackage{multirow}
 \usepackage[
    backend=biber,
    style=numeric,
  ]{biblatex}

\usepackage[toc,page]{appendix}
\usepackage{xspace}
\usepackage{tikz}
\usepackage{algorithm}
\usepackage{algpseudocode}
\usepackage{algorithmicx}

\usepackage[utf8]{inputenc}
\usepackage{array}
\usepackage{textcomp}
\usepackage{subcaption}
\usepackage{xcolor}
\usepackage{colortbl}
\usepackage{booktabs}
\usetikzlibrary{arrows,shapes,positioning,calc,matrix}

\definecolor{titleblue}{RGB}{51,101,138}
\definecolor{constraintred}{RGB}{179,77,77}
\definecolor{guaranteegreen}{RGB}{76,140,85}
\definecolor{methodpurple}{RGB}{110,76,140}
\definecolor{datablue}{RGB}{70,130,180}
\definecolor{nodepurple}{RGB}{110,76,140}
\definecolor{lightblue}{RGB}{173,216,230}
\definecolor{lightgreen}{RGB}{144,238,144}
\definecolor{correct}{RGB}{34,139,34}
\definecolor{incorrect}{RGB}{220,20,60}

\bibliography{arxiv}
\newtheorem{definition}{Definition}
\newtheorem{theorem}{Theorem}

\newtheorem{corollary}{Corollary}

\newtheorem{problem}{Problem}

\usepackage{authblk}

\newcommand{\fr}{\texttt{FlowRec}\xspace}
\newcommand{\hfr}{\texttt{HFR}\xspace}

\title{Hierarchical Forecast Reconciliation on Networks: A Network Flow Optimization Formulation}
\author[]{Charupriya Sharma}
\author[]{ I\~naki Estella Aguerri}
\author[]{Daniel Guimarans}

\affil[]{Amazon}

\affil[ ]{\textit {\{charupr, iaguerri, guimd\}@amazon.com}}

\date{}

\begin{document}

\maketitle

% \begin{abstract}
% Hierarchical forecasting with reconciliation is crucial for organizations requiring coherent predictions across multiple aggregation levels. Current methods like minimum trace (MinT) are limited to tree structures and become computationally expensive for large-scale problems. We introduce \fr, which reformulates hierarchical forecast reconciliation as a network flow optimization problem, enabling handling of general network structures beyond traditional hierarchies.

% We prove that while reconciliation under the $\ell_0$ norm is NP-hard, \fr enables polynomial-time solutions for all $\ell_p$ norms with $p > 0$. Our approach achieves $O(n^2 \log n)$ complexity for sparse networks, improving upon MinT's $O(n^3)$, while allowing efficient local updates without requiring expensive recomputation. We show that \fr is a special case of MinT where the weight matrix is determined by network structure rather than requiring estimation, providing theoretical insight into its superior performance.

% Experiments on both simulated and real logistics networks demonstrate that \fr runs 5-40x faster than existing methods while reducing memory usage and improving accuracy. The framework naturally accommodates practical constraints like capacity limits and minimum flow requirements, making it particularly valuable for large scale applications.
% \end{abstract}

\begin{abstract}
The problem of hierarchical forecasting with reconciliation requires that we forecast values that are part of a hierarchy (e.g.~customer demand on a state, district or city level), and there is a relation between different forecast values (e.g.~all district forecasts should add up to the state forecast). State of the art forecasting provides no guarantee for these desired structural relationships. Reconciliation addresses this problem, which is crucial for organizations requiring coherent predictions across multiple aggregation levels. Current methods like minimum trace (MinT) are mostly limited to tree structures and are computationally expensive for large-scale problems. We introduce \fr, which reformulates hierarchical forecast reconciliation as a network flow optimization problem, enabling forecasting on generalized network structures and relationships beyond trees.

We present a rigorous complexity analysis of hierarchical forecast reconciliation under different loss functions. While reconciliation under the $\ell_0$ norm is NP-hard, we prove polynomial-time solvability for all $\ell_{p > 0}$ norms and, more generally, for any strictly convex and continuously differentiable loss function. For sparse networks, \fr achieves $O(n^2\log n)$ complexity, significantly improving upon MinT's $O(n^3)$. Furthermore, we prove that \fr extends MinT beyond tree structures to handle general networks, replacing MinT's error-covariance estimation step with direct network structural information, theoretically justifying  its superior computational efficiency.

A key novelty of our approach is its handling of dynamic scenarios: while traditional methods require recomputing both base forecasts and their reconciliation, \fr provides efficient localised updates with optimality guarantees. A monotonicity property ensures that when forecasts improve incrementally, the initial reconciliation remains optimal. We also establish efficient, error-bounded approximate reconciliation, enabling fast updates in time-critical applications.

Experiments on both simulated and real benchmarks demonstrate that \fr improves runtime by 3-40x, reduces memory usage by 5-7x and improves accuracy over state of the art. These results establish \fr as a powerful tool for large-scale hierarchical forecasting applications.

\end{abstract}

\section{Introduction} \label{sec:intro}
Hierarchical forecasting addresses settings where predictions are needed at multiple aggregation levels and must remain coherent - that is, the individual forecasts must sum to match their aggregates, and has the goal of forecasting the future values of the underlying variable relationships. For example, demand fulfillment for an entire district must equal the aggregated demand across all stores in that district. The exponential growth of digital infrastructure and global retail operations has created unprecedented challenges in such forecasting. For instance, retail operations managing thousands of stores face complex inventory decisions where store-level forecasts must align with district totals, regional projections, and national planning. The financial implications are substantial - forecast inconsistencies between these levels lead to significant inventory costs \cite{seeger2016bayesian, do2015demand} and resource misallocation.

Forecasting on hierarchical data is a challenge in practise due to inconsistencies in data availability at all levels as well as mathematical challenges in recovering multivariate dependencies. Although a few end-to-end methods \cite{rangapuram2021end, das2023dirichlet, olivares2024probabilistic}exist, they have no guarantees on accuracy and coherence for complex networks. Thus, hierarchical forecasting typically involves two steps: generating \textit{base forecasts}  at each level independently, then reconciling these forecasts to satisfy aggregation constraints. These structures appear naturally in many network-based systems: traffic flows must sum across routes to match road capacity, water networks must balance supply and demand, and cloud services must allocate computing resources across data centers. For example, in the logistics network shown in Figure \ref{fig:treevsnetwork}, base forecasts for stores, warehouses, and regions must be reconciled to ensure physical consistency - the sum of flows into each location must equal the sum of flows out. These flows can represent demand, purchased goods, or the number of trucks needed to move the goods.

\begin{figure}[h]
\centering

\tikzset{
    node_style/.style={circle, draw, minimum size=0.8cm, align=center},
    flow_edge/.style={->, >=latex, thick},
    hier_edge/.style={dashed, gray, thick},
    flow_label/.style={fill=white, inner sep=1pt},
    region_node/.style={node_style, fill=gray!20},
    dc_node/.style={node_style, fill=lightblue},
    store_node/.style={node_style, fill=lightgreen},
    true_value/.style={text=correct},
    false_value/.style={text=incorrect}
    legend_node/.style={circle, draw, minimum size=0.4cm},
    legend_store/.style={legend_node, fill=lightgreen},
    legend_dc/.style={legend_node, fill=lightblue},
    legend_wh/.style={legend_node},
    legend_region/.style={legend_node, fill=gray!20}
}

\begin{tikzpicture}[node distance=2cm, scale=0.5, transform shape]
    % Left Panel Title
    \node[align=center] at (-4,5.25) {\textbf{Tree Reconciliation}\\\textbf{(Incorrect)}};
    
    % Right Panel Title
    \node[align=center] at (4,5.25) {\textbf{Network Flow}\\\textbf{(Correct)}};

    % Left Panel: Tree Structure
    \node[node_style] (Total) at (-4,4) {T:1000};
    \node[region_node] (RA) at (-5,2) {RA:600};
    \node[region_node] (RB) at (-3,2) {RB:450};
    \node[dc_node] (W1) at (-6,0) {W1:300};
    \node[dc_node] (DC1) at (-4,0) {D1:280};
    \node[dc_node] (DC2) at (-2.5,0) {D2:250};
    \node[dc_node] (W2) at (-1,0) {W2:200};
    \node[store_node] (S1) at (-6,-3) {S1:150\\\textcolor{incorrect}{(280)}};
    \node[store_node] (S2) at (-4,-3) {S2:280\\\textcolor{incorrect}{(400)}};
    \node[store_node] (S3) at (-1,-3) {S3:170\\\textcolor{incorrect}{(250)}};

    % Tree edges
    \draw[hier_edge] (Total) -- (RA);
    \draw[hier_edge] (Total) -- (RB);
    \draw[hier_edge] (RA) -- (W1);
    \draw[hier_edge] (RA) -- (DC1);
    \draw[hier_edge] (RB) -- (DC2);
    \draw[hier_edge] (RB) -- (W2);
    \draw[hier_edge] (W1) -- (S1);
    \draw[hier_edge] (DC1) -- (S2);
    \draw[hier_edge] (W2) -- (S3);

    % Right Panel: Network Structure
    \node[node_style] (Total2) at (4,4) {T:1000};
    \node[region_node] (RA2) at (3,2) {RA:580};
    \node[region_node] (RB2) at (5,2) {RB:420};
    \node[dc_node] (W12) at (2,0) {W1:300};
    \node[dc_node] (DC12) at (4,0) {D1:280};
    \node[dc_node] (DC22) at (5.5,0) {D2:180};
    \node[dc_node] (W22) at (7,0) {W2:170};
    \node[store_node] (S12) at (2,-3) {S1:\textcolor{correct}{280}};
    \node[store_node] (S22) at (4,-3) {S2:\textcolor{correct}{400}};
    \node[store_node] (S32) at (7,-3) {S3:\textcolor{correct}{250}};

    % Network flows
    \draw[flow_edge, gray] (W12) -- node[flow_label, above, sloped] {150} (S12);
    \draw[flow_edge, gray] (W12) -- node[flow_label, above, sloped] {150} (S22);
    \draw[flow_edge, blue] (DC12) -- node[flow_label, above, sloped] {130} (S12);
    \draw[flow_edge, blue] (DC12) -- node[flow_label, above, sloped] {150} (S22);
    \draw[flow_edge, blue] (DC22) -- node[flow_label, above, sloped] {100} (S22);
    \draw[flow_edge, blue] (DC22) -- node[flow_label, above, sloped] {80} (S32);
    \draw[flow_edge, guaranteegreen] (W22) -- node[flow_label, above, sloped] {170} (S32);

    % Hierarchical relationships in network
    \draw[hier_edge] (Total2) -- (RA2);
    \draw[hier_edge] (Total2) -- (RB2);
    \draw[hier_edge] (RA2) -- (W12);
    \draw[hier_edge] (RA2) -- (DC12);
    \draw[hier_edge] (RB2) -- (DC22);
    \draw[hier_edge] (RB2) -- (W22);

    % Comparison Table
    \node[anchor=north west] at (8.5,1.5) {\includegraphics[width=.6\textwidth]{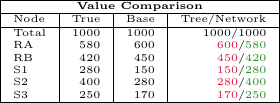}
    };

  %   % Legend
  %   \begin{scope}[xshift=8.5cm, yshift=2.5cm]
  %    \draw[hier_edge] (0,1.5) -- (1,1.5); \node[anchor=west] (text) at (1,1.5) {Hierarchical Relationship};
  %    \draw[flow_edge, gray] (0,1) -- (1,1); \node[anchor=west] (text) at (1,1.0) {WH1 Flow};
  %    \draw[flow_edge, blue] (0,0.5) -- (1,0.5); \node[anchor=west] (text) at (1,0.5) {DC1 Flow};
  %    \draw[flow_edge, lightgreen] (0,0) -- (1,0); \node[anchor=west] (text) at (1,0.0) {RB Flow};
  % \end{scope}

  % Add three-column legend with tikz styling

% Legend in three columns
\begin{scope}[yshift=-8cm]
    % Node Types Column
    \node[anchor=south west] at (-4,2) {\textbf{Node Types:}};
    \node[store_node,scale=0.7] at (-4,1.5) {};
    \node[anchor=west] at (-3.5,1.5) {S: Store};
    \node[dc_node,scale=0.7] at (-4,1) {};
    \node[anchor=west] at (-3.5,1) {D: Distribution Center};
    \node[dc_node,scale=0.7] at (-4,0.5) {};
    \node[anchor=west] at (-3.5,0.5) {W : Warehouse};
    \node[region_node,scale=0.7] at (-4,0) {};
    \node[anchor=west] at (-3.5,0) {RA,RB : Region A,B};
    \node[node_style,scale=0.7] at (-4,-0.5) {};
    \node[anchor=west] at (-3.5,-0.5) {T : Totals};

    % Flow Types Column
    \node[anchor=south west] at (1,2) {\textbf{Flow Types:}};
    \draw[hier_edge] (1,1.5) -- (2,1.5); 
    \node[anchor=west] at (2.2,1.5) {Hierarchical Relationship};
    \draw[flow_edge, gray] (1,1) -- (2,1); 
    \node[anchor=west] at (2.2,1) {Warehouse Flow};
    \draw[flow_edge, blue] (1,0.5) -- (2,0.5); 
    \node[anchor=west] at (2.2,0.5) {Distribution Flow};
    \draw[flow_edge, guaranteegreen] (1,0) -- (2,0); 
    \node[anchor=west] at (2.2,0) {Region  Flow};

    % Applications Column
    \node[anchor=south west] at (7,2) {\textbf{Applications:}};
    \node[anchor=west] at (7,1.5) {Forecasts:};
    \node[anchor=west] at (8.8,1.5) {Demand};
    \node[anchor=west] at (8.8,1) {Product Movement};
    \node[anchor=west] at (7.8,0.5) {Uses:};
    \node[anchor=west] at (8.8,0.5) {truck Reservation};
    \node[anchor=west] at (8.8,0) {Route Planning};
    \node[anchor=west] at (8.8,-0.5) {Inventory Control};
\end{scope}

\end{tikzpicture}

\bigskip

\begin{subfigure}{\textwidth}
\centering
 \includegraphics[width=.6\textwidth]{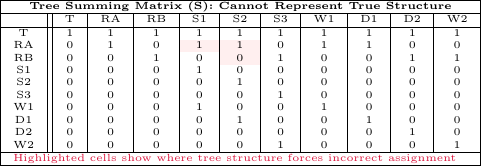}
\end{subfigure}

\bigskip

\begin{subfigure}{\textwidth}
\centering
 \includegraphics[width=.6\textwidth]{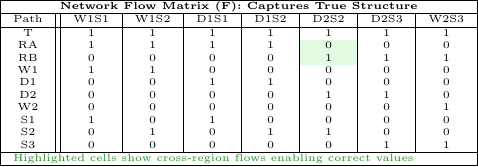}
\end{subfigure}

\caption{Comparison of Tree and Network Reconciliation. Base forecasts deviate from true values. Tree structure cannot recover true values due to single-parent constraint. Network structure allows multiple paths, enabling recovery of true values through flow conservation.} \label{fig:treevsnetwork}
\end{figure}

% While the state-of-the-art minimum trace (MinT)~\cite{wickramasuriya2019optimal} method combines forecasts in tree structures, it cannot handle the complex network relationships common in modern applications, where flows often follow multiple paths between nodes. Consider Figure \ref{fig:treevsnetwork}. The tree network (left), imposes an incorrect structure (tree summing matrix) on the base forecast, failing to recover the true values (right) due to the stores receiving flow from multiple warehouse and distribution center nodes, which is a relationship the general network structure (center) can correctly model (network flow matrix).

While the state-of-the-art minimum trace (MinT)~\cite{wickramasuriya2019optimal} method combines forecasts in tree structures by minimizing the trace of the forecast error covariance matrix, it cannot handle the complex network relationships, where flows often follow multiple paths between nodes. Consider Figure \ref{fig:treevsnetwork}. MinT's covariance estimation assumes each node has exactly one parent. The resulting tree network (left), imposes an incorrect structure (tree summing matrix) on the base forecast, failing to recover the true values (right) due to the stores receiving flow from multiple warehouse and distribution centers, which the general network structure (center) can correctly model (network flow matrix).

We present \fr, a novel approach that reformulates hierarchical forecast reconciliation as a network flow problem. Our key insight, as in Figure \ref{fig:treevsnetwork}, is that many forecasting constraints naturally map to flow conservation in networks. \fr generalizes hierarchical relationships beyond trees to arbitrary network structures and provides several fundamental advantages:

\begin{enumerate}
\item We provide a rigorous study of the complexity of loss functions for forecast reconciliation and prove that while reconciliation under the $\ell_0$ norm is NP-hard, \fr enables polynomial-time forecast reconciliation for all $\ell_{p>0}$ and strictly convex and continuously differentiable losses in general.

\item We show \fr is a generalization of MinT from trees to networks, where the weight matrix is determined by network structure. For sparse networks, \fr achieves optimal runtime of $O(n^2\log n)$ against $O(n^3)$ for MinT.

\item \fr supports efficient localised dynamic updates for network and data updates, and allows an $\varepsilon$-relaxation of structural constraints with provable error bounds, enabling fast reconciliation for online settings.

\item We prove a monotonicity property ensuring optimality preservation under improving forecasts, particularly valuable for online settings where data quality improves over time.

\item Empirically, we demonstrate significant speedups with reduced memory usage and improved accuracy across both simulated networks and standard forecasting benchmarks, including 3-40x performance improvements over existing methods like MinT.
\end{enumerate}

\fr generalizes traditional hierarchical structures like trees to networks while improving computational efficiency, as any tree is a specific type of graph. In the comparison table of Figure \ref{fig:treevsnetwork}, \fr can recover true values that tree-based methods cannot represent. Complex real-world constraints that were previously difficult or impossible to incorporate are natural components of \fr.

The remainder of this paper is organized as follows. Section \ref{sec:lit} provides background on hierarchical forecasting and network flows. Section \ref{sec:defs} introduces our \fr framework. Section \ref{sec:props} presents our main theoretical results on computational complexity, accuracy and establishes \fr's relationship to MinT. Section \ref{sec:fast} provides results on dynamic updates and approximations. Section \ref{sec:experiments} presents experimental results on both simulated and real-world data. Future work and conclusions are detailed in Section \ref{sec:future} and \ref{sec:conclude}.

\section{Background and Related Work}\label{sec:lit}

Base forecasting methods provide the foundation for hierarchical forecasting. While classical approaches like exponential smoothing \cite{hyndman2008forecasting} and traditional time series models like ARIMA \cite{box2015time} have been widely used, recent deep learning approaches \cite{sak2014long, rangapuram2018deep, salinas2020deepar, oreshkin2019n, zeng2023transformers} show significant improvements in complex hierarchical settings. The fundamental challenge remains that forecasts generated independently at different levels violate aggregation constraints \cite{wickramasuriya2019optimal}.

Traditional reconciliation approaches include bottom-up methods (aggregating base-level forecasts upward), middle-out methods (using a middle level as base), and top-down methods (disaggregating from the top level). The optimal combination approach \cite{hyndman2011optimal} provided a framework to combine forecasts from all levels using ordinary least squares, which MinT \cite{wickramasuriya2019optimal} enhanced by minimizing the trace of the reconciled forecast error covariance matrix.

Network flows offer a natural solution to these challenges, though they have not been extensively explored in forecasting contexts. Since the introduction of the maximum flow problem \cite{ford1956maximal}, the field has seen significant advances. Early work \cite{kantorovich1942translocation} established foundations for min-cost flow, leading to a series of breakthroughs \cite{christiano2011electrical, madry2013navigating, madry2016computing, van2020bipartite, king1994faster, olver2020simpler, lacki2012single, abboud2020new}. Recent work has achieved almost-linear time algorithms for maximum flow and minimum-cost flow \cite{dong2022nested, chen2022maximum, van2023deterministic}, with growing interest in learning-augmented settings \cite{davies2023predictive, polak2024learning, davies2024warm}.

Multicommodity flows have seen parallel development, from early work \cite{fulkerson1961network, edmonds1972theoretical} through advances in maximum concurrent flow \cite{shahrokhi90} and max-flow min-cut theorems \cite{leighton95}. Significant progress in approximation algorithms \cite{leighton1999multicommodity, fleischer00, benczur1996approximating, karakostas02, garg2007faster, madry10} has culminated in almost-linear time algorithms for undirected graphs \cite{chen2022maximum} and new techniques for directed cases \cite{van2023deterministic, vdBZ23}.

Dynamic settings \cite{DBLP:journals/jea/HanauerHS22, ford1958constructing, gale1959transient, skutella2009introduction} have seen particular advancement in decremental problems \cite{henzinger2014sublinear, henzinger2015improved} and incremental algorithms \cite{chen2024almost, van2024incremental} achieving almost-linear time complexity.

These advances in network algorithms, particularly in dynamic and multicommodity settings, provide the theoretical foundation for addressing the challenges of hierarchical time series, especially in handling complex network structures and dynamic updates.

\section{Hierarchical Forecasting on Networks with Flows} \label{sec:defs}

% [NEW] Enhanced opening paragraph with precise terminology
We now formalize hierarchical forecasting and reconciliation on networks introduced in Section \ref{sec:intro}, building on  definitions from \cite{panagiotelis2021forecast}. Given time series observations on a network $G=(V,E)$ with paths $\mathcal{P}$, we need to generate and reconcile forecasts while respecting hierarchical relationships represented by complex networks. %The fundamental challenge is ensuring that forecasts remain coherent across all aggregation levels while maintaining flow conservation.

% [NOTATION CHANGES from original] 
% - $\mathcal{T}$ for time index set (was T) for set consistency
% - $(\mathbf{y}_t)_{t\in \mathcal{T}}$ keeping bold for vector-valued time series
% - $\mathcal{B}$ for bottom-level set (was B)
% - Added explicit dimensions for all matrices
\begin{definition}[Hierarchical Time Series] \label{def:hier-time}
For a set of time steps $\mathcal{T}$, a hierarchical time series $(\mathbf{y}_t)_{t\in \mathcal{T}}$ is a collection of observations of $n$ variables $V$. More formally:
\begin{enumerate}
    \item $(\mathbf{y}_t)_{t\in \mathcal{T}}$ is a series of $n$-dimensional vectors $\mathbf{y}_t \in \mathbb{R}^n$
    \item Bottom-level variables $\mathcal{B}$, $|\mathcal{B}| = n_b \leq n$ cannot be expressed by linear aggregation of other variables. And $\beta: \mathbb{R}^n \to \mathbb{R}^{n_b}$ extracts bottom-level observations $\beta(\mathbf{v}_t) \mapsto \mathbf{b}_t$
    \item Aggregation constraints are linear and constant over time: $\mathbf{y}_t = \mathbf{S}\mathbf{b}_t \quad \forall t \in \mathcal{T}$
    \item The coherent subspace $\mathfrak{s} \subset \mathbb{R}^n$ is the linear subspace where aggregation constraints hold: $\mathfrak{s} = \{\mathbf{y} \in \mathbb{R}^n | \mathbf{y} = \mathbf{S}\beta(\mathbf{y})\}$, where $\mathbf{S} \in \mathbb{R}^{n\times n_b}$ encodes aggregation constraints.
\end{enumerate}

\end{definition}

% [NEW] Added motivation connecting to Figure \ref{fig:treevsnetwork}
In Figure~\ref{fig:treevsnetwork}, stores represent bottom-level variables $\mathcal{B}$, while the network structures represent aggregations $S$. Next, we define a forecast that satisfies aggregation constraints.

% [NOTATION CHANGES]
% - Standardized forecast notation to use T+h|T consistently
% - All vectors in bold
% - Clear distinction between forecast time T and observation time t
\begin{definition}[Point Forecast, Base Forecast, Coherent Forecast, Reconciliation]  \label{def:point}
For a hierarchical time series $(\mathbf{y}_t)_{t\in \mathcal{T}}$:
\begin{enumerate}
    \item A point forecast $\hat{\mathbf{y}}_{T+h|T} \in \mathbb{R}^n$ is a vector of predictions for all variables $V$ at time $T+h$, made at time $T$.
    \item A base forecast $\hat{\mathbf{y}}_{T+h|T}$ is an initial point forecast, typically generated using standard time series models (e.g., ARIMA), which may not satisfy aggregation constraints.
    \item A point forecast $\hat{\mathbf{y}}_{T+h|T}$ is coherent for aggregation constraint matrix $\mathbf{S}$ if $\hat{\mathbf{y}}_{T+h|T} \in \mathfrak{s}$.
    \item Reconciling an incoherent point forecast $\hat{\mathbf{y}}_{T+h|T}$ to a coherent point forecast $\tilde{\mathbf{y}}_{T+h|T} = \psi(\hat{\mathbf{y}}_{T+h|T})$ is done via a reconciling mapping $\psi: \mathbb{R}^n \to \mathfrak{s}$.
\end{enumerate}
\end{definition}

Now we can define hierarchical forecasting: generating $h$ future predictions that are both coherent (satisfying aggregation constraints) and accurate (minimizing some loss function) with respect to future actual values $\mathbf{y}_{T+h}$.
\begin{problem}[Hierarchical Forecasting] \label{prob:hier-forecast}
Given a hierarchical time series $(\mathbf{y}_t)_{t\in \mathcal{T}}$, a loss function $L: \mathbb{R}^n \times \mathbb{R}^n \to \mathbb{R}^+_0$, and a forecast horizon $h$, the hierarchical forecasting problem for each time step $k=1,\ldots,h$ is:
\begin{equation}
\min_{\hat{\mathbf{y}}_{T+k|T}} L(\mathbf{y}_{T+k}, \hat{\mathbf{y}}_{T+k|T}) \quad \text{subject to} \quad \hat{\mathbf{y}}_{T+k|T} = \mathbf{S}\hat{\mathbf{b}}_{T+k|T}
\end{equation}
where $\mathbf{S} \in \mathbb{R}^{n \times n_b}$ encodes aggregation constraints and $\hat{\mathbf{b}}_{T+k|T}$ represents bottom-level forecasts.
\end{problem}

% [NEW] Added transition to flow-specific structures
The matrix $\mathbf{S}$ can encode any network structure, including trees. However, network-structured data presents unique challenges for maintaining coherence across complex hierarchical relationships. Consider Figure~\ref{fig:treevsnetwork}'s network representation: each node's flow must equal the sum of all flows on its incoming and outgoing edges, while each edge's flow must equal the sum of all path flows traversing that edge. For instance, a store might receive inventory from multiple warehouses simultaneously, a scenario impossible to represent in tree-based hierarchies. This limitation motivates our flow aggregated time series formulation, which extends Definition \ref{def:flow-hier-time} to naturally incorporate these complex network constraints.

% [NOTATION CHANGES]
% - All vectors and matrices in bold
% - Explicit dimensions for block matrices
% - Consistent set notation for paths $\mathcal{P}$
% \begin{definition}[Flow Aggregated Time Series] \label{def:flow-hier-time}
% Let $G=(V,E)$ be a network with paths $\mathcal{P}$. A hierarchical flow time series $\mathbf{y}_t \in \mathbb{R}^n$, where $n = |\mathcal{P}| + |V| + |E|$, consists of observations:
% \[ \mathbf{y}_t = \begin{pmatrix} \mathbf{n}_t \\ \mathbf{e}_t \\ \mathbf{p}_t \end{pmatrix} \]
% where $\mathbf{n}_t \in \mathbb{R}^{|V|}$, $\mathbf{e}_t \in \mathbb{R}^{|E|}$, and $\mathbf{p}_t \in \mathbb{R}^{|\mathcal{P}|}$ represent node, edge, and path flows respectively. Flow aggregation constraints are given by:
% \[ \mathbf{S} = \begin{pmatrix} \mathbf{V}' \\ \mathbf{E}' \\ \mathbf{I}_{|\mathcal{P}|} \end{pmatrix} \in \mathbb{R}^{n \times |\mathcal{P}|} \]
% where $\mathbf{V}' \in \mathbb{R}^{|V| \times |\mathcal{P}|}$ and $\mathbf{E}' \in \mathbb{R}^{|E| \times |\mathcal{P}|}$ are the vertex-path and edge-path incidence matrices.
% \end{definition}

\begin{definition}[Flow Aggregated Time Series] \label{def:flow-hier-time}
Let $G=(V,E)$ be a network with paths $\mathcal{P}$. A hierarchical flow time series $\mathbf{y}_t \in \mathbb{R}^n$, where $n = |\mathcal{P}| + |V| + |E|$, consists of observations:
\[ \mathbf{y}_t = \begin{pmatrix} \mathbf{v}_t \\ \mathbf{e}_t \\ \mathbf{p}_t \end{pmatrix} \]
where $\mathbf{v}_t \in \mathbb{R}^{|V|}$, $\mathbf{e}_t \in \mathbb{R}^{|E|}$, and $\mathbf{p}_t \in \mathbb{R}^{|\mathcal{P}|}$ represent node, edge, and path flows respectively. Flow aggregation constraints are given by:
\[ \mathbf{S} = \begin{pmatrix} \mathbf{V}' \\ \mathbf{E}' \\ \mathbf{I}_{|\mathcal{P}|} \end{pmatrix} \in \mathbb{R}^{n \times |\mathcal{P}|} \text{ where $\mathbf{V}' \in \mathbb{R}^{|V| \times |\mathcal{P}|}$ is the vertex-path incidence matrix}\]
  defined by
$ v'_{ij} = \begin{cases} 1 & \text{if vertex } i \text{ appears in path } j \\ 0 & \text{otherwise} \end{cases} $
and $\mathbf{E}' \in \mathbb{R}^{|E| \times |\mathcal{P}|}$ is the edge-path incidence matrix defined by
$ e'_{ij} = \begin{cases} 1 & \text{if edge } i \text{ appears in path } j \\ 0 & \text{otherwise} \end{cases} $.
\end{definition}
% [NOTATION CHANGES]
% - Consistent function space notation
% - Bold vectors throughout
% - Explicit domain and codomain for mappings
Now, we extend the concept of forecast coherence to flow aggregated time series.
\begin{definition}[Flow Point Forecast] \label{def:flow-point}
A coherent flow point forecast is a coherent point forecast where the aggregation constraint matrix $\mathbf{S}$ is a flow aggregation matrix. 
\end{definition}

% [NEW] Added transition to problem formulation
Having established the structure of flow-based hierarchical forecasting, we now formulate the general forecast reconciliation problem on networks. This novel formulation combines standard hierarchical time series and flow constraints.

% \begin{problem}[Hierarchical Flow Reconciliation (\hfr)] \label{prob:hfr}
% Given a loss function $L: \mathbb{R}^n \times \mathbb{R}^n \to \mathbb{R}^+_0$, base forecasts $\hat{\mathbf{y}}_{T+k|T}$ for $k=1,\ldots,h$, 
% and a flow aggregation matrix $\mathbf{S} \in \mathbb{R}^{n \times |\mathcal{P}|}$,
% the forecast reconciliation problem for each time step $k$ is:
% \begin{equation}\label{eq:opt_problem}
% \min_{\tilde{\mathbf{y}}_{T+k|T}} L(\hat{\mathbf{y}}_{T+k|T}, \tilde{\mathbf{y}}_{T+k|T}) \quad \text{subject to} \quad \tilde{\mathbf{y}}_{T+k|T} = \mathbf{S}\tilde{\mathbf{b}}_{T+k|T}
% \end{equation}
% Additional constraints may include box-constraints:
% %\begin{equation}\label{eq:box_constraints}
% $\boldsymbol{\ell} \leq \tilde{\mathbf{y}}_{T+k|T} \leq \mathbf{u}$
% %\end{equation}
% where $\boldsymbol{\ell}, \mathbf{u} \in \mathbb{R}^n$ are lower and upper bounds respectively.
% \end{problem}

\begin{problem}[Hierarchical Flow Reconciliation (\hfr)] \label{prob:hfr}
Given a loss function $L: \mathbb{R}^n \times \mathbb{R}^n \to \mathbb{R}^+_0$, base forecasts $\hat{\mathbf{y}}_{T+k|T}$ for $k=1,\ldots,h$, 
and a flow aggregation matrix $\mathbf{S} \in \mathbb{R}^{n \times |\mathcal{P}|}$,
find a reconciling mapping $\psi: \mathbb{R}^n \to \mathfrak{s}$ that for each time step $k$ solves:
\begin{equation}\label{eq:opt_problem}
\min_{\tilde{\mathbf{y}}_{T+k|T}} L(\hat{\mathbf{y}}_{T+k|T}, \tilde{\mathbf{y}}_{T+k|T}) \quad \text{subject to} \quad \tilde{\mathbf{y}}_{T+k|T} = \mathbf{S}\tilde{\mathbf{b}}_{T+k|T}
\end{equation}
where $\tilde{\mathbf{y}}_{T+k|T} = \psi(\hat{\mathbf{y}}_{T+k|T})$ and additional constraints may include box-constraints $\boldsymbol{\ell} \leq \tilde{\mathbf{y}}_{T+k|T} \leq \mathbf{u}$ with lower and upper bounds $\boldsymbol{\ell}, \mathbf{u} \in \mathbb{R}^n$.
\end{problem}
% [NEW] Added connecting paragraph to next section
\hfr captures both hierarchical structure and complex constraints like capacity limits.  %In Section \ref{sec:props}, we will analyze the computational complexity of this problem and present efficient algorithms for its solution.

\section{Properties of the \hfr Formulation and the \fr Framework} \label{sec:props}
\begin{algorithm}[h]
\caption{Flow-Based Hierarchical Forecast Reconciliation (\fr)}\label{algo:flowrec}
\begin{algorithmic}[1]
\Require Network $G=(V,E)$, observations $(\mathbf{y}_t)_{t\in \mathcal{T}}$, loss function $L$
\Ensure Reconciled forecasts $\tilde{\mathbf{y}}_{T+k|T}$ for $k=1,\ldots,h$

\State Initialize $\mathcal{P} \gets$ paths with observations in vector $(\mathbf{y}_t)_{t\in \mathcal{T}}$
\State Initialize $\mathbf{V} \gets$ vertex-path incidence matrix (Definition \ref{def:flow-hier-time})
\State Initialize $\mathbf{E} \gets$ edge-path incidence matrix (Definition \ref{def:flow-hier-time})
\State Form flow aggregation matrix: $\mathbf{S} \gets \begin{bmatrix} \mathbf{V} \\ \mathbf{E} \\ \mathbf{I} \end{bmatrix}$

\State Generate base forecasts $\hat{\mathbf{y}}_{T+k|T}$ using any suitable forecasting method
\State Solve \hfr instance: $\min_{\tilde{\mathbf{y}}} L(\tilde{\mathbf{y}}, \hat{\mathbf{y}}_{T+k|T})$ subject to $\mathbf{S} = \mathbf{b}\tilde{\mathbf{y}}$

\end{algorithmic}
\end{algorithm}

The \fr framework (Algorithm \ref{algo:flowrec}) converts an instance of the hierarchical forecasting problem to an instance of \hfr and solves that instance (Theorem \ref{thm:computation}). %Reconciliation methods like MinT are alternatives to \fr. 

%We analyze the theoretical properties of \hfr and \fr and show several advantages over alternative methods. While the problem of forecast reconciliation with arbitrary constraints is not known to have an optimal polynomial-time solution, we show that \fr enables one for commonly used loss functions like $l_{p>0}$ norms. Proofs are in Appendix \ref{app:props}.

Preprocessing steps 1-4 use $G$, which encodes any hierarchical relationship between the variables and uses that to formulate an instance of \hfr. Note that Step 5 can be implemented using various forecasting methods depending on the application (see Section \ref{sec:experiments}). Our analysis focuses on Step 6 (reconciliation). Methods like MinT provide an alternative to \fr, but our theoretical results show several advantages of \fr (Proofs in Appendix \ref{app:props}):

\begin{enumerate}
    \item \textbf{Computational Complexity of \hfr:} We prove polynomial-time solvability for all $\ell_{p>0}$ as well as general convex and continuously differentiable losses. A runtime lower bound proves optimality of \fr for sparse networks.
    
    \item \textbf{Relationship of \fr to MinT:} We show \fr is a generalization of MinT from trees to graphs with superior computational efficiency.
    
    \item \textbf{Flow Conservation:} \fr exploits network structure encoded in $\mathbf S$ to ensure coherence, avoiding the need for variance-covariance estimation required by MinT.
    
    \item \textbf{Multiple Computation Methods:} We provide two ways to compute the reconciled forecasts, and use hardness results to guide the choice of which method to use according to the given loss functions for the best performance.
\end{enumerate}

\subsection{Hardness of \hfr}

The choice of loss function in hierarchical forecasting affects both computational tractability and practical utility. Mean squared error (MSE), root mean squared error (RMSE), and mean absolute error (MAE) — corresponding to $\ell_2$, $\sqrt \ell_2$ , and $\ell_1$ norms respectively — dominate forecasting practice due to their convexity and differentiability (for MSE/RMSE) or robustness to outliers (for MAE). However, discrete resource allocation often demands step-wise loss functions. For example, in logistics, when determining the number of trucks needed for delivery, the truck capacity reservation leads to:
\[L(y, \hat{y}) = \begin{cases} 
0 & \text{if } |y - \hat{y}| \leq a \\
c_1 & \text{if } a < |y - \hat{y}| \leq b \\
c_2 & \text{if } b < |y - \hat{y}| \leq c \\
c_3 & \text{if } |y - \hat{y}| > c
\end{cases}\]
 where thresholds $a,b,c$ represent various truck sizes and costs $c_i$ reflect the penalties for under or over-estimating the required number of trucks at each size.
Table~\ref{tab:common_loss_3} presents common loss functions, from $\ell_p$ norms to Huber loss, which combines differentiability near zero with linear growth for large errors. Our analysis shows \hfr is NP-hard for the $\ell_0$ norm but polynomial-time solvable for $\ell_p$ norms with $p > 0$, suggesting continuous relaxations when discrete constraints prove intractable. The runtime lower bound of $\Omega(m \log n)$ later proves \fr's optimality for sparse networks, which are standard in applications.

% In the following, we focus on the hardness of finding the coherent forecast that minimizes the loss compared to a given base forecast for common loss functions. Table~\ref{tab:common_loss_3} gives an overview of the most common loss functions in forecasting.
% For applied settings, the loss function should represent the impact of the forecasting error. In the context of logistic, forecasting errors lead to incorrect procurement, e.g.~the amount of trucks needed to transport customer orders. As the number of trucks is integral, this makes step-wise loss functions appealing, although these functions are less common in machine learning due to them being non-differentiable. These step-functions are in general of the form  \[L(y, \hat{y}) = \begin{cases} 
% 0 & \text{if } |y - \hat{y}| \leq a \\
% c_1 & \text{if } a < |y - \hat{y}| \leq b \\
% c_2 & \text{if } b < |y - \hat{y}| \leq c \\
% c_3 & \text{if } |y - \hat{y}| > c
% \end{cases}\]

\begin{table}[thb]
\caption{Common Loss Functions in Forecasting}
\centering
\small
\begin{tabular}{p{2.15cm}p{3.0cm}|p{2.2cm}p{4.45cm}}
\toprule
Mean Squared Error (MSE) & 
$\frac{1}{n} \sum_{i=1}^n (y_i - \hat{y}_i)^2$ &
Mean Absolute Error (MAE) & 
$\frac{1}{n} \sum_{i=1}^n |y_i - \hat{y}_i|$ \\
\midrule
Root Mean Squared Error (RMSE) & 
$\sqrt{\frac{1}{n} \sum_{i=1}^n (y_i - \hat{y}_i)^2}$ &
Huber Loss & 
$\begin{cases} 
\frac{1}{2}(y - \hat{y})^2 & \text{if } |y - \hat{y}| \leq \delta \\
\delta|y - \hat{y}| - \frac{1}{2}\delta^2 & \text{otherwise}
\end{cases}$ \\
\bottomrule
\end{tabular}
\label{tab:common_loss_3}
\end{table}
% \begin{table}[thb]
% \caption{Common Loss Functions in Forecasting}
% \centering
% \small
% \begin{tabular}{p{5cm}p{7cm}}
% \textbf{Loss Function} & \textbf{Definition} \\
% \toprule
% Mean Squared Error (MSE) & 
% $MSE = \frac{1}{n} \sum_{i=1}^n (y_i - \hat{y}_i)^2$ \\
% \midrule
% Root Mean Squared Error (RMSE) & 
% $RMSE = \sqrt{\frac{1}{n} \sum_{i=1}^n (y_i - \hat{y}_i)^2}$ \\
% \midrule
% Mean Absolute Error (MAE) & 
% $MAE = \frac{1}{n} \sum_{i=1}^n |y_i - \hat{y}_i|$ \\
% \hline
% Huber Loss & 
% $L_\delta(y, \hat{y}) = \begin{cases} 
% \frac{1}{2}(y - \hat{y})^2 & \text{for } |y - \hat{y}| \leq \delta \\
% \delta|y - \hat{y}| - \frac{1}{2}\delta^2 & \text{otherwise}
% \end{cases}$ \\
% %\hline
% %Binary Cross-Entropy & 
% %$L(y, \hat{y}) = -[y \log(\hat{y}) + (1-y) \log(1-\hat{y})]$ \\
% %\hline
% %Categorical Cross-Entropy & 
% %$L(y, \hat{y}) = -\sum_{i=1}^C y_i \log(\hat{y}_i)$ \\
% %\hline
% %Hinge Loss & 
% %$L(y, \hat{y}) = \max(0, 1 - y\hat{y})$ \\
% %\hline
% % Quantile Loss & 
% % $L_q(y, \hat{y}) = \max(q(y - \hat{y}), (q-1)(y - \hat{y}))$ \\
% \bottomrule
% \end{tabular}

% \label{tab:common_loss_3}
% \end{table}

For all of these loss-functions, the set of feasible solutions is convex (with or without box-constraints). Consequently, the hardness of \hfr depends on the loss-function. 
For step-functions, \hfr becomes NP-hard, as they introduce an integrality aspect. For our hardness proof, we use the $||\cdot||_0$-norm as a loss function. The $||\cdot||_0$ norm in $\mathbb{R}^n$ is defined as the number of non-zero entries, i.e. 
$ ||x||_0 = \sum_{i=0}^n \iota_i, \quad  \iota_i = \begin{cases} 0 & \text{if } x_i = 0,\\ 1 & \text{otherwise.}\end{cases}$

\begin{theorem}[NP-Hardness for $\ell_0$]
Minimizing the loss of the $\|\cdot\|_0$ norm for \hfr is NP-hard, where for $\mathbf{x} \in \mathbb{R}^n$:
$ \|\mathbf{x}\|_0 = |\{i : x_i \neq 0\}| $
\end{theorem}

% [In main text: Brief proof sketch]

% [For appendix: Full proof]
% \begin{proof}
% We prove the claim by reduction from Exact-1-in-3-SAT. Given an instance with $n$ variables and $m$ clauses where each clause is of the form $(l_{j,1} \lor l_{j,2} \lor l_{j,3})$, we construct a forecast reconciliation instance as follows:

% For each variable $x_i$ in the SAT instance, create a forecast variable $x_i$. For each clause $(l_{j,1} \lor l_{j,2} \lor l_{j,3})$, create a linear equation:
% \[ (-1)^{\alpha_1}x_{j,1} + (-1)^{\alpha_2}x_{j,2} + (-1)^{\alpha_3}x_{j,3} + \sigma_j = 1 \]
% where $\alpha_i = 0$ if $l_{j,i}$ is positive, and $\alpha_i = 1$ if $l_{j,i}$ is negative.

% If there exists a variable assignment satisfying Exact-1-in-3-SAT, then setting $\sigma_j = 0$ for all $j$ gives reconciled forecasts with zero $\ell_0$ loss. Conversely, any reconciliation with zero $\ell_0$ loss must have $\sigma_j = 0$ for all $j$, implying a satisfying assignment for the original Exact-1-in-3-SAT instance.
% \end{proof}

% [NOTATION CHANGES from original]
% - Added vector notation for forecasts
% - Standardized subscripts
% - Explicit statement of problem construction

\begin{theorem}[Polynomial-Time Solvability for $\ell_{p > 0}$]
For any $\ell_p$ norm with $p > 0$, \hfr can be solved in polynomial time, even with additional linear constraints (e.g.~upper/lower bounds) or weighted objectives. For $p>1$, the solution is unique. Specifically:
\begin{enumerate}
    \item For $p=1$ (MAE), \hfr reduces to linear programming
    \item For $p=2$ (MSE), \hfr reduces to quadratic programming
    \item For general $p > 0$, \hfr can be solved via convex optimization
\end{enumerate}
\end{theorem}

%For more specialised losses like Huber loss, we can still achieve polynomial time solutions. 
\begin{theorem}[Generalized Loss Functions]
Consider loss functions of form:
$L(\mathbf{x}, \mathbf{y}) = \sum_{i=1}^n f(|x_i - y_i|)$.
For strictly convex, continuously differentiable $f$ with $f(0)=0$, \hfr has a unique solution computable in polynomial time, even with additional linear constraints (e.g.~upper/lower bounds) or weighted objectives.

\end{theorem}
\begin{theorem}[Runtime Lower Bound] \label{thm:lower-bound}
Any algorithm to solve \hfr on networks with $m$ edges and $n$ nodes requires $\Omega(m \log n)$ comparisons, even for the $\ell_2$ loss function.
\end{theorem}

\subsection{Flow Conservation}

We now show that \fr naturally ensures coherence through flow conservation.

% [NOTATION CHANGES]
% - Consistent bold vectors/matrices
% - Calligraphic sets
% - Added explicit time indices
% - Standardized flow notation
\begin{theorem}[Flow Conservation] \label{thm:flow_conservation}
Let $\hat{\mathbf{y}}_{T+h|T} \in \mathbb{R}^n$ be a point forecast of hierarchical flow time series in a network $G=(V,E,\mathcal{P})$, and $\mathbf{S} \in \mathbb{R}^{n\times|\mathcal{P}|}$ the flow aggregation matrix. Then:
\begin{enumerate}
    \item $\mathbf{S}$ is a reconciling mapping
    \item The edge values $\tilde{\mathbf{y}}_{e,T+h|T}$ of
    $ \tilde{\mathbf{y}}_{T+h|T} = \mathbf{S}\hat{\mathbf{y}}_{T+h|T} $
    form a flow in $G=(V,E)$ satisfying supplies and demands:
    \[ b_v = \sum_{P \in \mathcal{P}:O(P)=v} \tilde{\mathbf{y}}_{P,T+h|T} - \sum_{P \in \mathcal{P}:D(P)=v} \tilde{\mathbf{y}}_{P,T+h|T}, \quad \forall v \in V \]
\end{enumerate}
where $\tilde{\mathbf{y}}_{P,T+h|T}$ is the reconciled value for path $P$, and $O(P), D(P) \in V$ refer to the origin and destination nodes of path $P$, respectively.
\end{theorem}

% \begin{proof}
% We need to prove two statements:

% Step 1: $\mathbf{S}$ is a reconciling mapping.
% By definition of $\mathbf{S}$, we have:
% \[ \mathbf{S}\hat{\mathbf{p}}_{T+h|T} = \begin{pmatrix} \mathbf{V}' \\ \mathbf{E}' \\ \mathbf{I} \end{pmatrix} \hat{\mathbf{p}}_{T+h|T} = \begin{pmatrix} \sum_{P \in \mathcal{P}:v\in P} \hat{\mathbf{p}}_{P,T+h|T} \\ \sum_{P \in \mathcal{P}:e\in P} \hat{\mathbf{p}}_{P,T+h|T} \\ \hat{\mathbf{p}}_{T+h|T} \end{pmatrix} \in \mathfrak{s} \subset \mathbb{R}^n \]

% Step 2: The values $\tilde{\mathbf{y}}_{e,T+h|T}$ form a flow in $G=(V,E)$ with supplies and demands $b_v$.
% We need to prove:
% \[ \sum_{e=(·,v)\in E} \tilde{\mathbf{y}}_{e,T+h|T} - \sum_{e=(v,·)\in E} \tilde{\mathbf{y}}_{e,T+h|T} = b_v, \quad \forall v \in V \]

% Note this holds for nodes $v \in V$ with $b_v = 0$ (intermediate nodes); thus, proving this equation proves flow conservation. We have:
% \begin{align*}
% &\sum_{e=(·,v)\in E} \tilde{\mathbf{y}}_{e,T+h|T} - \sum_{e=(v,·)\in E} \tilde{\mathbf{y}}_{e,T+h|T} \\
% &= \sum_{e=(·,v)\in E} \sum_{P \in \mathcal{P}:e\in P} \tilde{\mathbf{y}}_{P,T+h|T} - \sum_{e=(v,·)\in E} \sum_{P \in \mathcal{P}:e\in P} \tilde{\mathbf{y}}_{P,T+h|T}
% \end{align*}

% We can eliminate any path $P \in \mathcal{P}$ containing both an edge of form $e=(·,v)$ and $e'=(v,·)$, as these terms cancel. The remaining terms belong to paths having $v$ as origin or destination, giving:
% \[ \sum_{P \in \mathcal{P}:O(P)=v} \tilde{\mathbf{y}}_{P,T+h|T} - \sum_{P \in \mathcal{P}:D(P)=v} \tilde{\mathbf{y}}_{P,T+h|T} = b_v \]

% This completes the proof.
% \end{proof}

\subsection{Computing Reconciled Forecasts}

We now present two equivalent methods for computing the reconciled forecasts. These methods offer different computational trade-offs according to the given loss function.

% [NOTATION CHANGES]
% - Vectors in bold
% - Calligraphic for sets
% - Explicit dimensions for spaces/matrices
% - Consistent time indices
\begin{theorem}[Computation Methods]\label{thm:computation}
Let $\hat{\mathbf{y}}_{T+h|T} \in \mathbb{R}^n$ be a base forecast for network $G=(V,E,\mathcal{P})$. We can compute the reconciled forecast $\tilde{\mathbf{y}}_{T+h|T}$ in two equivalent ways:

\noindent 1. \textbf{Orthogonal projection} of $\hat{\mathbf{y}}_{T+h|T}$ to $\mathfrak{s}$:
Compute an orthonormal basis $\mathcal{E} = \{\mathbf{e}_1,\ldots,\mathbf{e}_{|\mathcal{P}|}\}$ of $\mathfrak{s}$ and project $\hat{\mathbf{y}}_{T+h|T}$ onto $\mathfrak{s}$:
    $ \tilde{\mathbf{y}}_{T+h|T} = \sum_{i=1}^{|\mathcal{P}|} \langle \hat{\mathbf{y}}_{T+h|T}, \mathbf{e}_i \rangle \mathbf{e}_i $\\
2. \textbf{Minimum Reconciling Flow}:
$ \min_{\mathbf{s}_{\mathcal{P}}, \mathbf{s}_E, \mathbf{s}_V} f(\mathbf{s}_{\mathcal{P}}, \mathbf{s}_E, \mathbf{s}_V) $ subject to:
\begin{align*}
\mathbf{s}_{\mathcal{P}} &\geq \hat{\mathbf{y}}_{\mathcal{P}} - \tilde{\mathbf{y}}_{\mathcal{P}}, \quad
\mathbf{s}_{\mathcal{P}} \geq \tilde{\mathbf{y}}_{\mathcal{P}} - \hat{\mathbf{y}}_{\mathcal{P}} \quad \forall P \in \mathcal{P} \\
\mathbf{s}_E &\geq \hat{\mathbf{y}}_E - \tilde{\mathbf{y}}_E, \quad 
\mathbf{s}_E \geq \tilde{\mathbf{y}}_E - \hat{\mathbf{y}}_E, \quad \sum_{P \in \mathcal{P}:e \in P} \tilde{\mathbf{y}}_P = \tilde{\mathbf{y}}_e \quad \forall e \in E \\
\mathbf{s}_V &\geq \hat{\mathbf{y}}_V - \tilde{\mathbf{y}}_V, \quad
\mathbf{s}_V \geq \tilde{\mathbf{y}}_V - \hat{\mathbf{y}}_V, \quad \sum_{P \in \mathcal{P}:v \in P} \tilde{\mathbf{y}}_P = \tilde{\mathbf{y}}_v \quad \forall v \in V 
\end{align*}

% 3. Direct computation using flow aggregation matrix:
% \[ \tilde{\mathbf{y}}_{T+h|T} = \hat{\mathbf{y}}_{T+h|T} - \mathbf{S}^T(\mathbf{S}\mathbf{S}^T)^{-1}(\mathbf{S}\hat{\mathbf{y}}_{T+h|T} - \mathbf{b}) \]
 \end{theorem}

The $\ell_1$ and $\ell_2$ norms admit efficient solutions by leveraging network structure. For $\ell_1$, the absolute differences map directly to edge costs in a min-cost flow problem, solvable in $O(n^2\log n)$ time for sparse networks, matching the lower bound from Theorem \ref{thm:lower-bound} times $O(n)$. The $\ell_2$ case offers a choice: orthogonal projection exploits low-rank structure when available, while quadratic min-cost flow utilizes network sparsity. However, for $\ell_p$ norms with $p \neq 1,2$ and generalized losses $f(|x_i - y_i|)$, these direct methods may fall short. Non-linear objectives prevent straightforward flow reductions, dense matrix operations negate sparsity benefits, and maintaining error bounds becomes challenging. These cases necessitate more nuanced approaches that balance network structure exploitation, iterative optimization techniques, and efficient update mechanisms to handle the inherent non-linearity while preserving the computational advantages of the network formulation.

% \begin{proof}
% We prove each method yields the optimal reconciled forecasts:

% Step 1: Orthogonal projection correctness.
% The coherent subspace $\mathfrak{s}$ is the range of $\mathbf{S}$. The orthogonal projection onto $\mathfrak{s}$ minimizes the L2 distance to $\hat{\mathbf{y}}_{T+h|T}$ while ensuring the result lies in $\mathfrak{s}$.

% Step 2: Minimum Reconciling Flow equivalence.
% The optimization problem is equivalent to:
% \[ \min_{\tilde{\mathbf{y}}} \|\tilde{\mathbf{y}} - \hat{\mathbf{y}}\|_2^2 \text{ subject to flow conservation} \]
% The slack variables measure absolute deviations, and the constraints ensure flow conservation.

% % Step 3: Direct computation correctness.
% % This is the explicit solution to the normal equations of the constrained least squares problem:
% % \[ \min_{\tilde{\mathbf{y}}} \|\tilde{\mathbf{y}} - \hat{\mathbf{y}}\|_2^2 \text{ subject to } \mathbf{S}\tilde{\mathbf{y}} = \mathbf{b} \]

% The equivalence follows from the uniqueness of the solution to the reconciliation problem proven in the MinT equivalence theorem.
% \end{proof}

\subsection{Relationship to MinT}

The state-of-the-art MinT reconciles forecasts on tree based hierarchies by minimizing the trace of the reconciled forecast error ($n \times n$) covariance matrix. We show that \fr can be viewed as a special case where the weight matrix corresponds to the network structure, allowing significant computational improvements and generalisation to complex networks.

% [NOTATION CHANGES]
% - Matrices in bold: $\mathbf{S}$, $\mathbf{W}$, $\boldsymbol{\Sigma}$
% - Vectors in bold: $\hat{\mathbf{y}}$, $\tilde{\mathbf{y}}$, $\mathbf{b}$
% - Problems labeled as (P₁), (P₂) for clarity
\begin{theorem}[MinT Equivalence]
Let $\mathbf{S} \in \mathbb{R}^{n\times m}$ be a hierarchical aggregation matrix and $\mathbf{W} \in \mathbb{R}^{n\times n}$ be a symmetric positive definite matrix. For any base forecast $\hat{\mathbf{y}}$ and target aggregation values $\mathbf{b}$, the following optimization problems:

$(P_1):
 \min_{\tilde{\mathbf{y}}} (\tilde{\mathbf{y}} - \hat{\mathbf{y}})^T \mathbf{W} (\tilde{\mathbf{y}} - \hat{\mathbf{y}}) \quad \text{subject to} \quad \mathbf{S}\tilde{\mathbf{y}} = \mathbf{b} $

$(P_2):
 \min_{\tilde{\mathbf{y}}} (\tilde{\mathbf{y}} - \hat{\mathbf{y}})^T \boldsymbol{\Sigma}^{-1} (\tilde{\mathbf{y}} - \hat{\mathbf{y}}) \quad \text{subject to} \quad \mathbf{S}\tilde{\mathbf{y}} = \mathbf{b} $

where $\boldsymbol{\Sigma}$ is the forecast error covariance matrix, have the following properties:
\begin{enumerate}
    \item Both problems have a unique solution of the form: 
    $ \tilde{\mathbf{y}} = \hat{\mathbf{y}} - \mathbf{W}^{-1}\mathbf{S}^T(\mathbf{S}\mathbf{W}^{-1}\mathbf{S}^T)^{-1}(\mathbf{S}\hat{\mathbf{y}} - \mathbf{b}) $.
    \item When $\mathbf{W} = \boldsymbol{\Sigma}^{-1}$, the problems $(P_1)$ and $(P_2)$ are equivalent.
    \item \fr is a special case where $\mathbf{W} = \mathbf{I}$ and $\mathbf{S} = \begin{pmatrix} \mathbf{V} \\ \mathbf{E} \\ \mathbf{I} \end{pmatrix}$ with $\mathbf{V}$ as vertex-path incidence matrix, $\mathbf{E}$ as edge-path incidence matrix, and $\mathbf{I}$ as identity matrix for paths.
\end{enumerate}
\end{theorem}

% [After MinT Equivalence theorem and proofs]

\begin{corollary}[Computational Complexity of \fr and MinT]
Let $G=(V,E,\mathcal{P})$ be a network with $|V| = n$ vertices and $|E| = m$ edges where $m \in [n-1, n^2]$. Then:
\begin{enumerate}
    \item For sparse networks i.e. $m = O(n)$: \fr has $O(n^2 \log n)$, MinT has $O(n^3)$ operations.    
    \item For dense networks i.e. $m = O(n^2)$: both have $O(n^4)$ operations.
\end{enumerate}
%\fr preserves network flow properties without requiring explicit path enumeration.
\end{corollary}

\section{Fast Updates and Approximations}\label{sec:fast}

Modern forecasting demands rapid adaptation to network changes and swift operational updates. While traditional methods necessitate complete recomputation of expensive base forecasts as well as reconciliation, \fr enables efficient localized updates to the reconciled forecast. We address three key scenarios: network expansion, disruption, and approximation of constraints. \fr leverages classical network flow techniques~\cite{ahuja1993network} and approximation algorithms~\cite{williamson2011design}, allowing unprecedented efficiency for adapting hierarchical forecasting to these scenarios. Full proofs are provided in Appendix \ref{app:fast}.

\subsection{Incremental Updates in Expanding Networks}

When new paths emerge in a forecasting network, we must integrate these additions while preserving existing reconciliations. We prove that optimal updates can be computed by considering only affected paths, a capability unique to \fr.

\begin{theorem}[Incremental Forecast Updates]
Let $G=(V,E,\mathcal{P})$ be a network with reconciled forecasts $\tilde{\mathbf{y}}_{T+h|T}$ satisfying flow conservation. Upon addition of edge $e^*$ with forecast $\hat{\mathbf{y}}_{e^*,T+h|T}$, for any $\ell_p$ norm with $p \geq 1$, the minimal adjustment to maintain optimality is:
\[ \tilde{\mathbf{y}}'_{P,T+h|T} = \begin{cases}
    \tilde{\mathbf{y}}_{P,T+h|T} + \Delta_{e^*}/|\mathcal{P}_{e^*}| & \text{if } P \in \mathcal{P}_{e^*} \\
    \tilde{\mathbf{y}}_{P,T+h|T} & \text{otherwise}
\end{cases} \]
where $\mathcal{P}_{e^*} = \{P \in \mathcal{P} : e^* \in P\}$ and $\Delta_{e^*} = \hat{\mathbf{y}}_{e^*,T+h|T} - \sum_{P \in \mathcal{P}_{e^*}} \tilde{\mathbf{y}}_{P,T+h|T}$.
\end{theorem}

\subsection{Incremental and Monotonic Updates for New Data}

In practical settings, new data often affects only a small network subset. We prove that under certain conditions, local forecast changes require only local reconciliation updates, a significant advancement over traditional methods requiring complete recomputation.

\begin{theorem}[Data Update Optimality]
Let $G=(V,E,\mathcal{P})$ be a network with reconciled forecasts $\tilde{\mathbf{y}}_{T+h|T}$ optimized for the $\ell_p$ norm with $p \geq 1$. For a new forecast $\hat{\mathbf{y}}'_{T+h|T}$ differing from $\hat{\mathbf{y}}_{T+h|T}$ in exactly one component $x \in V \cup E \cup \mathcal{P}$ by amount $\delta = |\hat{y}_x - \hat{y}'_x|$, if:
\[ |\tilde{y}_x - \hat{y}'_x| < |\tilde{y}_x - \hat{y}_x| \text{ then $\tilde{\mathbf{y}}_{T+h|T}$ remains optimal for $\hat{\mathbf{y}}'_{T+h|T}$.}\]
\end{theorem}

This result yields a powerful monotonicity property which is particularly valuable in online settings with frequent small updates trending towards improved accuracy.

\begin{corollary}[Monotonicity of Reconciliation Updates]
Let $G=(V,E,\mathcal{P})$ be a network, and let $\hat{\mathbf{y}}^{(0)}_{T+h|T}, \hat{\mathbf{y}}^{(1)}_{T+h|T}, ..., \hat{\mathbf{y}}^{(k)}_{T+h|T}$ be a sequence of forecasts where each differs from the previous in exactly one component. If for each update $i$:
$ |\tilde{y}_x^{(i-1)} - \hat{y}_x^{(i)}| < |\tilde{y}_x^{(i-1)} - \hat{y}_x^{(i-1)}| $, then the initial reconciliation $\tilde{\mathbf{y}}^{(0)}_{T+h|T}$ remains optimal throughout the sequence.
\end{corollary}

% This monotonicity property has important practical implications for forecast updating. When new data arrives sequentially and each update brings forecasts closer to their reconciled values, no recomputation is needed. This is particularly valuable in online settings where data arrives frequently but changes are small and trend toward better accuracy.

\subsection{Forecast Maintenance Under Network Disruption}

Network disruptions necessitate forecast adjustments while maintaining coherence. \fr provides explicit bounds on forecast changes and efficient computation methods.

\begin{theorem}[Forecast Redistribution]
Let $G=(V,E,\mathcal{P})$ be a network with reconciled forecasts $\tilde{\mathbf{y}}_{T+h|T}$ optimized for the $\ell_2$ norm. Upon loss of edge $e^*$, if $G'=(V,E\setminus\{e^*\},\mathcal{P}')$ maintains strong connectivity, then the minimally adjusted forecasts $\tilde{\mathbf{y}}'_{T+h|T}$ satisfy:
\[ \|\tilde{\mathbf{y}}'_{T+h|T} - \tilde{\mathbf{y}}_{T+h|T}\|_2^2 \leq \left(\sum_{P \in \mathcal{P}_{e^*}} \tilde{\mathbf{y}}_{P,T+h|T}\right)^2 \text{where $\mathcal{P}_{e^*}$ denotes paths using $e^*$.} \]

\end{theorem}

\subsection{Efficient Approximate Reconciliation}

For applications prioritizing speed over exact coherence, we provide a principled relaxation with bounded incoherence.

\begin{theorem}[Approximate Forecast Reconciliation]
Let $G=(V,E,\mathcal{P})$ be a network with base forecasts $\hat{\mathbf{y}}_{T+h|T}$. For accuracy threshold $\varepsilon > 0$, there exists an $\varepsilon$-relaxed reconciliation $\tilde{\mathbf{y}}_{\varepsilon,T+h|T}$ computable in $O(m \log(\frac{1}{\varepsilon}) \log n)$ time satisfying:
\[ \left|\sum_{P \in \mathcal{P}:e \in P} \tilde{\mathbf{y}}_{\varepsilon,P,T+h|T} - \tilde{\mathbf{y}}_{\varepsilon,e,T+h|T}\right| \leq \varepsilon \quad \forall e \in E \]
with bounded deviation from exact reconciliation:
$ \|\tilde{\mathbf{y}}_{\varepsilon,T+h|T} - \tilde{\mathbf{y}}_{T+h|T}\|_2 \leq \sqrt{\varepsilon|E|}\|\tilde{\mathbf{y}}_{T+h|T}\|_2 $.
\end{theorem}

\subsection{Summary of Computational Requirements}

Table~\ref{tab:complexity} summarizes \fr's computational requirements, demonstrating significant improvements over MinT's $O(n^3)$ time and $O(n^2)$ memory complexities.

\begin{table}[h]
\caption{Computational Requirements of \fr Operations}
\label{tab:complexity}
\centering
\begin{tabular}{|l|l|l|l|}
\hline
Operation & Time Complexity & Memory & Notes\\
\hline
Data Update & $O(1)$ & $O(1)$ & Single component change\\
\hline
Edge Addition & $O(|\mathcal{P}_{e^*}|)$ & $O(|\mathcal{P}_{e^*}|)$ & $\mathcal{P}_{e^*}$: paths using new edge\\
\hline
Edge Removal & $O(|\mathcal{P}_{e^*}|(m + n\log n))$ & $O(|\mathcal{P}_{e^*}|d)$ & $d$: maximum path length\\
\hline
$\varepsilon$-Relaxation & $O(m \log(\frac{1}{\varepsilon}) \log n)$ & $O(m)$ & $\varepsilon$: accuracy parameter\\
\hline
\end{tabular}
\end{table}

For sparse networks $(m = O(n))$, \fr is much faster than MinT's $O(n^3)$, with substantially lower memory requirements, especially for localized operations. The $\varepsilon$-relaxation offers a tunable accuracy-speed trade-off, enabling rapid approximate updates when needed.

\section{Experimental Evaluation}\label{sec:experiments}

We evaluate \fr on both simulated and real-world datasets, validating our theoretical results while demonstrating practical advantages in computational efficiency and reconciliation accuracy. We compare \fr, implemented as described in Theorem \ref{thm:computation} in Python, with the optimization problem being solved by the interior point method based solver for quadratic programs Clarabel \cite{goulart2024clarabel} to other reconciliation approaches implemented with Nixtla \footnote{\url{https://nixtlaverse.nixtla.io/}}. All experiments were run on an AWS g6.8xlarge instance with 32 virtual CPUs, 128GB RAM, and an NVIDIA A10G GPU with 24GB memory running Ubuntu 22.04 LTS.

Base forecasts are generated from the Nixtla methods, with hyper-parameter tuning: 
\begin{enumerate}
    \item \texttt{Naive:} All forecasts have the value of the last observation.
    \item \texttt{AutoARIMA} \cite{box2015time}: Selects ARIMA (AutoRegressive Integrated Moving Average) model, combining historical patterns with trend and seasonality components.

    \item \texttt{Simple Exponential Smoothing (SES): } It uses a weighted average of all past observations where the weights decrease exponentially into the past. 

\item \texttt{DeepAR} \cite{salinas2020deepar}:
Deep learning approach using autoregressive recurrent neural networks. 
\item \texttt{NBEATS} \cite{oreshkin2019n}:
Neural architecture with residual links and basis expansion. 

\item \texttt{LSTM} \cite{sak2014long}:
Recurrent neural network architecture with gates to control information flow through time and learns long-term dependencies.

\item \texttt{DLinear} \cite{zeng2023transformers}:
Decomposes time series into trend and seasonal components. 
Applies simple linear layer to each component separately before combining for final forecast.

\end{enumerate}

For reconciliation of base forecasts, we compare \fr against two standard methods:
\begin{enumerate}
    \item \texttt{Bottom-Up (BU):}  Forecasts are generated independently for the most disaggregated level and then aggregated upwards to obtain consistent forecasts for higher levels.
    \item \texttt{Minimum-Trace (MinT): } It aims to minimize the trace of the covariance matrix of reconciled forecast errors. MinT can be implemented with different estimation techniques, and here we have used OLS (Ordinary Least Squares) and  a non-negativity constraint.

\end{enumerate}

\subsection{Simulated Network Experiments}

Our experimental framework follows standard network flow literature construction, where each test instance comprises source, sink, and intermediate nodes connected by edges, except between source-sink pairs. The framework encompasses 100 random instances with network density ranging uniformly from $n$ to $n^2$ edges, each limited to 50 nodes for comparative analysis with existing methods, as larger dense networks often lead to convergence failures for MinT and some base models. Following standard forecasting practice, we generate ground truth values using flow conservation and add Gaussian noise to create base forecasts.

Table~\ref{tab:perf_comparison} presents the aggregated computational performance of \fr compared to MinT and BU reconciliation for AutoARIMA and DeepAR. The choice of the base forecasting method shows limited impact on \fr, with minor changes for MinT and BU. The results strongly support our theoretical complexity analysis (Theorem~\ref{thm:computation}). For sparse networks, \fr achieves runtime improvements of 3-40x over MinT, aligning with the predicted $O(n^2\log n)$ vs $O(n^3)$ complexity difference. Even for dense graphs, \fr maintains a significant improvement (notice the standard deviation), showing the impact of the formulation and the resulting ability to leverage highly efficient solvers. The memory usage reported by the tracemalloc utility further validates our analysis, with MinT requiring substantially more memory due to its covariance matrix operations, although results are implementation dependant.

% \begin{table}[t!]
% \caption{Performance Comparison of Reconciliation Methods (Mean and Standard Deviation)}
% \label{tab:perf_comparison}
% \centering
% \begin{tabular}{llrr|rr}
% \toprule
% Loss & Method & \multicolumn{2}{c}{AutoARIMA} & \multicolumn{2}{c}{DeepAR} \\
% \cmidrule(lr){3-4} \cmidrule(lr){5-6}
% & & Runtime (s) & Memory (MB) & Runtime (s) & Memory (MB) \\
% \midrule
% RMSE & \fr   & \textbf{0.063 ± 0.002} & \textbf{1.2 ± 0.1} & \textbf{0.063 ± 0.002} & \textbf{1.1 ± 0.1} \\
%      & MinT  & 0.172 ± 0.050 & 5.6 ± 1.5 & 0.232 ± 0.060 & 5.6 ± 1.5 \\
%      & BU  & 0.088 ± 0.010 & 1.8 ± 0.2 & 0.085 ± 0.010 & 1.9 ± 0.2 \\
% \midrule
% MAE  & \fr   & \textbf{0.058 ± 0.001} & \textbf{1.0 ± 0.1} & \textbf{0.058 ± 0.001} & \textbf{1.0 ± 0.1} \\
%      & MinT  & 0.165 ± 0.045 & 5.6 ± 1.5 & 0.225 ± 0.055 & 5.6 ± 1.5 \\
%      & BU  & 0.085 ± 0.010 & 1.8 ± 0.2 & 0.082 ± 0.010 & 1.9 ± 0.2 \\
% \bottomrule
% \end{tabular}
% \end{table}
\begin{table}[t!]
\caption{Performance Comparison (Mean ± Standard Deviation)}
\label{tab:perf_comparison}
\small
\setlength{\tabcolsep}{3pt}  % reduce column spacing
\begin{tabular}{l|rr|rr||rr|rr}
\toprule
& \multicolumn{4}{c||}{RMSE} & \multicolumn{4}{c}{MAE} \\
\cmidrule(lr){2-5} \cmidrule(lr){6-9}
& \multicolumn{2}{c|}{ARIMA} & \multicolumn{2}{c||}{DeepAR} & \multicolumn{2}{c|}{ARIMA} & \multicolumn{2}{c}{DeepAR} \\[-0.3ex]
Method & Time(s) & MB & Time(s) & MB & Time & MB & Time(s) & MB \\
\midrule
\fr & \textbf{.063±.002} & \textbf{1.2} & \textbf{.063±.002} & \textbf{1.1} & \textbf{.058±.001} & \textbf{1.0} & \textbf{.058±.001} & \textbf{1.0} \\
MinT & .172±.050 & 5.6 & .232±.060 & 5.6 & .165±.045 & 5.6 & .225±.055 & 5.6 \\
BU & .088±.010 & 1.8 & .085±.010 & 1.9 & .085±.010 & 1.8 & .082±.010 & 1.9 \\
\bottomrule
\end{tabular}
\end{table}

Figure~\ref{fig:overall-rmse} illustrates \fr's superior reconciliation overall accuracy for RMSE for various base forecasting methods. BU reconciliation consistently increases the error of the base forecast. MinT shows better performance than BU,  predicting values close to the base forecast, sometimes with a mild improvement. This can be explained by the enforced non-negativity constraint,  which forces all negative forecast numbers to zero thus limiting outliers. However,  \fr can be  at least five times more accurate than the base forecasts, irrespective of the base forecast model and the forecast hierarchy level (node, edge, path or overall). \fr's ability to improve upon base forecasts, an atypical capability for reconciliation methods, can be attributed to its effective utilization of network structure, as predicted by our flow conservation principles (Theorem~\ref{thm:flow_conservation}).  Figures showing similar results on MAE, as well as results over nodes, edges and paths are given in Appendix \ref{app:exp-sim}. 

\begin{figure}
    \centering
    \includegraphics[scale=0.45]{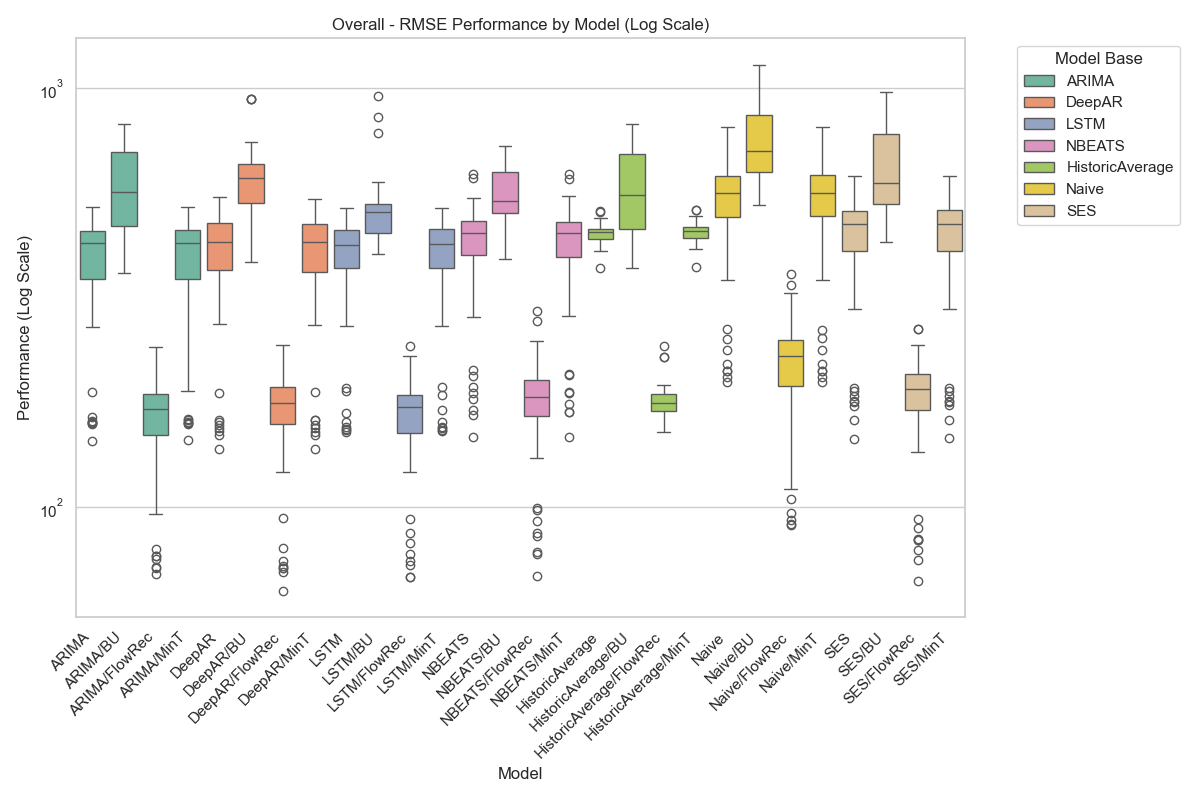}
    \caption{Total RMSE for forecasting and reconciliation for simulated data, shown across multiple base forecasting methods. \fr (third bar for each base model) shows the lowest RMSE}.%, while bottom-up shows the highest. MinT shows similar performance to the unreconciled base forecast. }
    \label{fig:overall-rmse}
\end{figure}

The poor performance of BU reconciliation can be understood through our theoretical framework. In a network $G = (V, E, \mathcal{P})$, BU aggregates path-level forecasts:

\[ \tilde{e}_{t+h} = \sum_{P \in \mathcal{P}: e \in P} \hat{P}_{t+h} \quad \text{for all } e\in E \quad \text{and}  \quad \tilde{v}_{t+h} = \sum_{P \in \mathcal{P}: v \in P} \hat{P}_{t+h} \quad \text{for all } v\in V \]

This approach relies solely on path-level forecasts, which our flow conservation theorem (Theorem~\ref{thm:flow_conservation}) shows to be constrained by both edge and node-level flows. By disregarding these higher-level constraints, BU fails to leverage the full network structure. Moreover, in large networks where the number of paths grows as $O(|N|^k)$ for paths of length $k$, individual path forecasts become increasingly sparse and potentially unreliable.

Now, consider a single base forecasting model like AutoARIMA, which is one of the our performers as a base forecast. Figure \ref{fig:overall-density1} shows the RMSE plotted against network density (average number of neighbours for a node with 1 being fully connected) for AutoARIMA plotted for the base and reconciled forecasts of MinT and \fr. BU is the worst performing, and thus omitted. We note that while performance degrades as density increases for each case, however \fr continues to perform the best, and also widens its performance advantage with increasing density compared to base AutoARIMA and MinT. Figure \ref{fig:overall-length1} shows the RMSE plotted against the maximum path length for AutoARIMA plotted for the base and reconciled forecasts of MinT and \fr. BU is the worst performing, and thus omitted. We note that while performance degrades as density increases for each case, however \fr continues to perform the best, and also widens its performance advantage with increasing density compared to base AutoARIMA and MinT. This shows \fr is applicable to a broad range of network structures. Complete results with MAE are given in Appendix \ref{app:exp-sim}.

\begin{figure}
    \centering
    \includegraphics[scale=0.35]{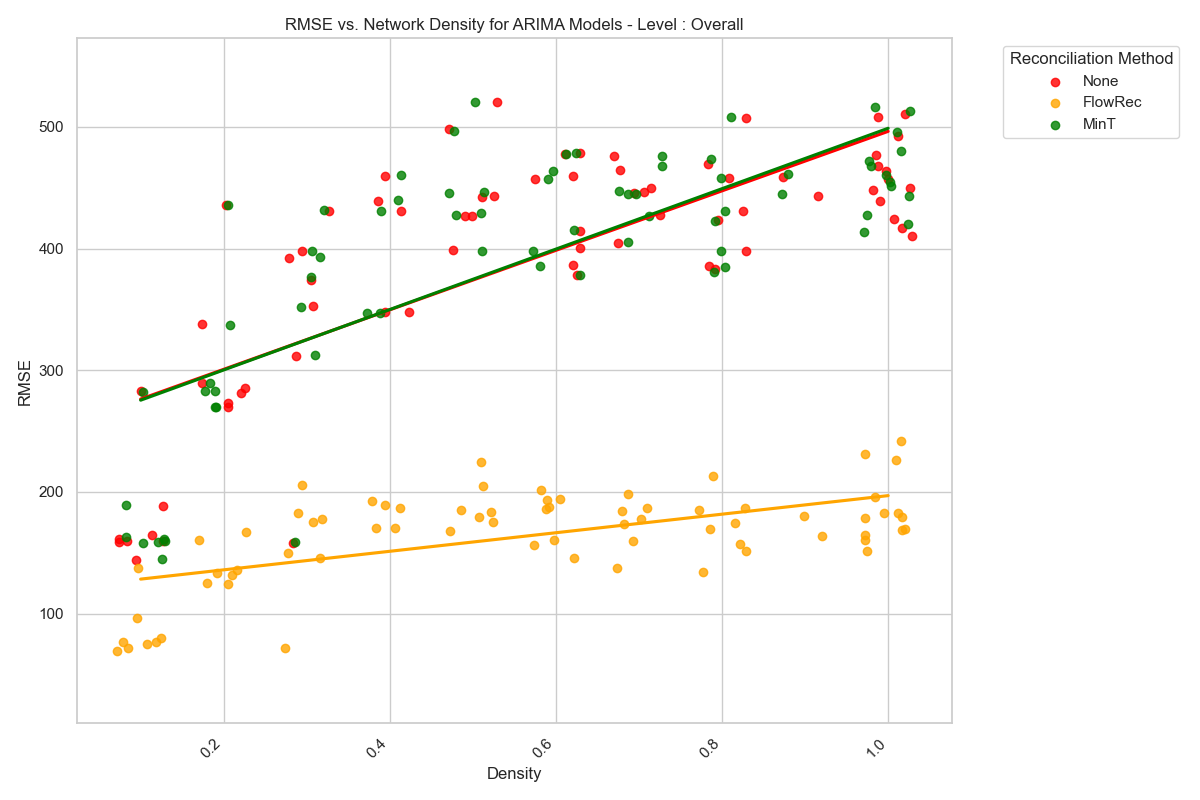}
    \caption{RMSE for AutoARIMA, plotted against density ( node degree, 1 means fully connected) }
    \label{fig:overall-density1}
\end{figure}

\begin{figure}
    \centering
    \includegraphics[scale=0.35]{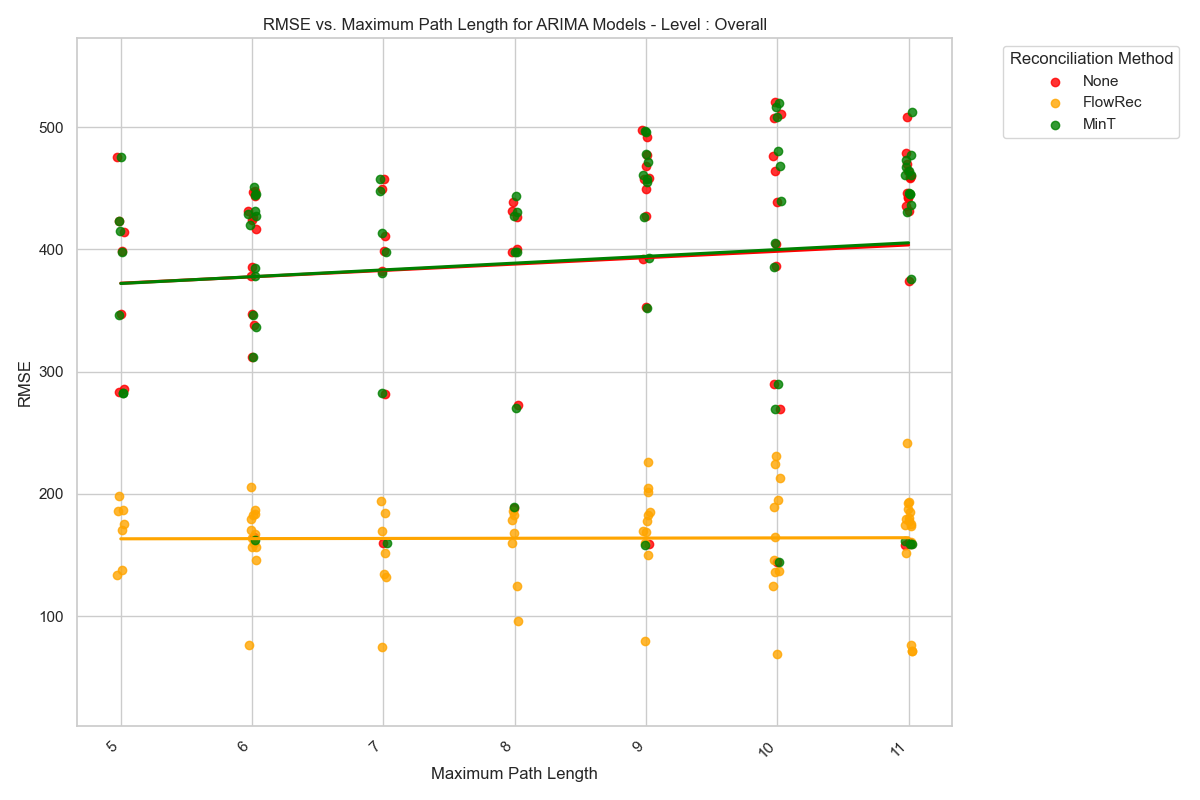}
    \caption{RMSE for AutoARIMA, plotted against the maximum path length}
    \label{fig:overall-length1}
\end{figure}

\subsection{Experiments on Real Benchmarks} \label{sec:exp_real}
To validate \fr's performance on established benchmarks, we conducted experiments on hierarchical datasets (Table \ref{tab:datasets}) widely used in forecasting literature~\cite{rangapuram2021end, olivares2024probabilistic, kamarthi2022profhit}. These datasets include Australian labour and tourism data, SF Bay area traffic, and Wikipedia article views. While lacking explicit flow conservation constraints and being mostly tree-like due to implementation standards of forecasting models, these datasets provide a rigorous test of \fr's generalizability to standard hierarchical reconciliation tasks.

\begin{table}[t!]
\caption{Datasets summary, $h$ being forecast horizon}\label{tab:datasets}
\centering
\begin{tabular}{lrrrrrr}
\toprule
Dataset & Total & Bottom & Aggregated & Levels & Observations & $h$ \\
\midrule
Tourism & 89 & 56 & 33 & 4 & 36 & 8 \\
Tourism-L (grouped) & 555 & 76,304 & 175 & 4,5 & 228 & 12 \\
Labour & 57 & 32 & 25 & 4 & 514 & 8 \\
Traffic & 207 & 200 & 7 & 4 & 366 & 1 \\
Wiki & 199 & 150 & 49 & 5 & 366 & 1 \\
\bottomrule
\end{tabular}
\end{table}
%\section{Experiments on Real Data} \label{sec:exp_real}
% We also validated our approach using a selection of hierarchical datasets \cite{rangapuram2021end, olivares2024probabilistic, kamarthi2022profhit}, extensively used for forecasting and reconciliation in academic literature. The data was taken directly from the hosted versions in the Nixtla package. We show RMSE results on Australian Monthly Labour (Labour), SF Bay Area Daily Traffic, Quarterly Australian Tourism Visits (Tourism(Small)), Monthly Australian Tourism visits (Tourism(Large)), and daily Wikipedia article views (Wiki2). While these datasets lack inherent network structure, they provide a rigorous test of FlowRec's performance on general hierarchical reconciliation tasks. We focus our analysis on Tourism(Small), with complete results for all datasets and metrics provided in Appendix \ref{app:exp-real}.

Figure~\ref{fig:rmse-tours} presents RMSE results on Tourism(Small). \fr achieves accuracy matching MinT, with both methods improving upon base forecasts across four different forecasting models. In contrast, BU shows inconsistent performance, matching FlowRec and MinT in two cases but significantly underperforming in the remaining five. This pattern persists across other datasets and evaluation metrics (see Appendix \ref{app:exp-real}).

Computationally, \fr maintains its theoretical advantages, executing 18x faster than MinT while reducing memory usage by 7x. These efficiency gains stem from FlowRec's ability to exploit the implicit network structure it imposes on hierarchical relationships.

\begin{figure}
    \centering
    \includegraphics[scale=0.43]{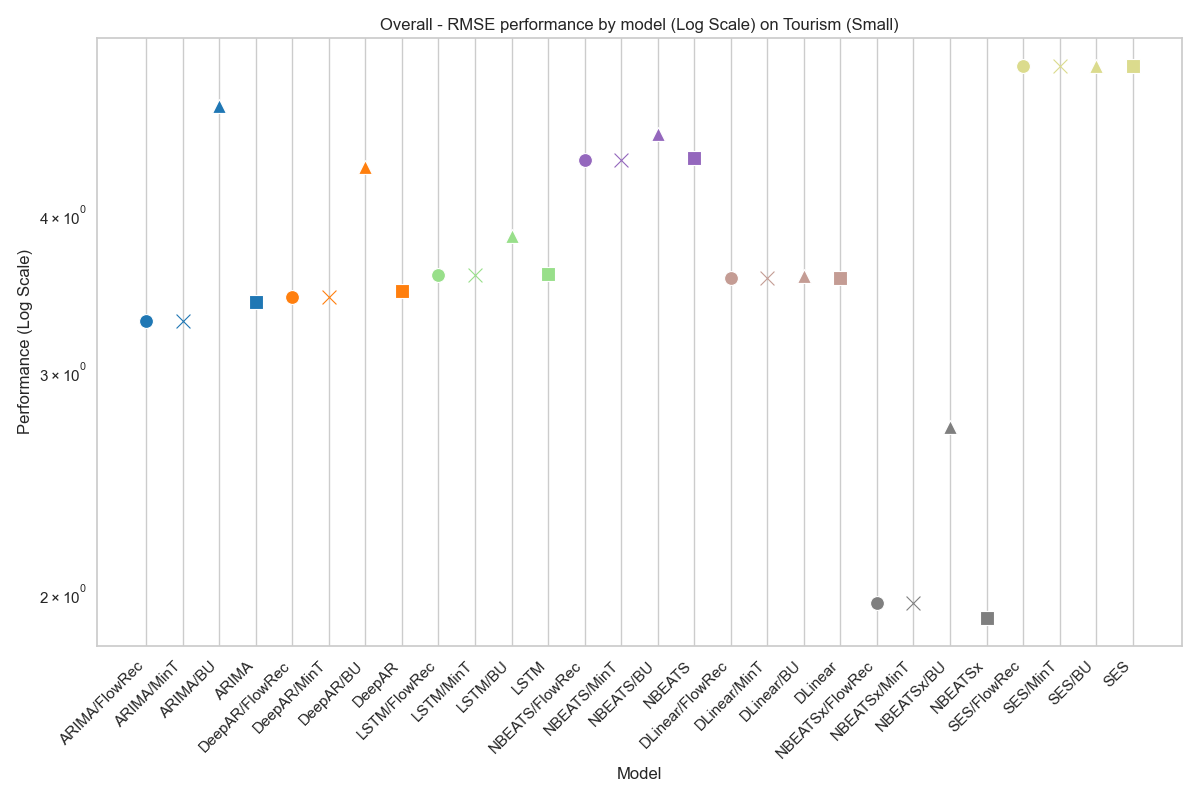}
    \caption{RMSE results for forecasting and reconciliation models across all hierarchies.}
    \label{fig:rmse-tours}
\end{figure}

\section{Future Research Directions} \label{sec:future}

While \fr provides a foundation for generalized hierarchical forecasting, several important theoretical and practical challenges remain open. We identify three primary directions.

First, the framework could be extended to sophisticated loss functions arising in resource procurement. Asymmetric losses are particularly relevant, as under-forecasting typically incurs higher costs than over-forecasting through premium-priced last-minute procurement. Similarly, the discrete nature of resources (e.g. numbers of trucks) suggests investigating approximation algorithms for step-wise losses, despite the NP-hardness of the $\ell_0$ norm.

Second, real-world forecasting systems often possess rich prior knowledge that could enhance reconciliation quality. While \fr handles basic box constraints, incorporating complex capacity constraints while maintaining computational efficiency remains open. Moreover, historical forecast performance varies across networks, suggesting weighted variants of our flow formulation could better capture prediction confidence.

Third, the temporal aspects of resource procurement introduce complex dependencies. The ability to adjust resources varies with forecast horizon, suggesting time-varying constraints or multi-stage optimization. While our monotonicity results provide insight into forecast updates, handling general recourse actions with committed resources remains challenging.

Additionally, investigating connections to fundamental graph concepts could yield valuable insights. Spectral graph theory might provide new perspectives on forecast stability, while network design principles could inform optimal hierarchical structures for reconciliation. Cut-based approaches might offer alternative views of forecast disagreement across hierarchies.

\section{Conclusion}\label{sec:conclude}

We have presented \fr, establishing three key theoretical advances in hierarchical forecast reconciliation. First, we characterize computational complexity across loss functions, proving NP-hardness for $\ell_0$ while achieving optimal $O(n^2\log n)$ complexity for $\ell_{p>0}$ norms in sparse networks, improving upon MinT's $O(n^3)$. Second, we prove \fr's equivalence to MinT while eliminating covariance estimation through network structure. Third, we establish novel dynamic properties: monotonicity preserving optimality under improving forecasts, and bounded-error approximate reconciliation enabling efficient updates.

These contributions open several promising directions. The framework naturally extends to asymmetric and discrete losses for resource procurement, enables incorporation of capacity constraints and confidence measures, and provides foundations for handling rolling horizons in dynamic settings. The connection between network structure and reconciliation suggests broader implications for network design and online forecasting scenarios.

Our experimental results confirm these theoretical advantages translate to practice: 3-40x faster computation, 5-7x reduced memory usage, and improved accuracy across diverse datasets. These results establish FlowRec as a significant advance in hierarchical forecasting, particularly for large-scale applications requiring frequent updates.

\printbibliography

\appendix
\printbibliography
\section{Properties of the \hfr Formulation and the \fr Framework} \label{app:props}

\subsection{Hardness of \hfr}

\begin{theorem}[NP-Hardness for $\ell_0$]
Minimizing the loss of the $\|\cdot\|_0$ norm in \hfr is NP-hard, where for $\mathbf{x} \in \mathbb{R}^n$:
\[ \|\mathbf{x}\|_0 = |\{i : x_i \neq 0\}| \]
\end{theorem}

% [For appendix: Full proof]
\begin{proof}
We prove the claim by reduction from Exact-1-in-3-SAT. Given an instance with $n$ variables and $m$ clauses where each clause is of the form $(l_{j,1} \lor l_{j,2} \lor l_{j,3})$, we construct a forecast reconciliation instance as follows:

For each variable $x_i$ in the SAT instance, create a forecast variable $x_i$. For each clause $(l_{j,1} \lor l_{j,2} \lor l_{j,3})$, create a linear equation:
\[ (-1)^{\alpha_1}x_{j,1} + (-1)^{\alpha_2}x_{j,2} + (-1)^{\alpha_3}x_{j,3} + \sigma_j = 1 \]
where $\alpha_i = 0$ if $l_{j,i}$ is positive, and $\alpha_i = 1$ if $l_{j,i}$ is negative.

If there exists a variable assignment satisfying Exact-1-in-3-SAT, then setting $\sigma_j = 0$ for all $j$ gives reconciled forecasts with zero $\ell_0$ loss. Conversely, any reconciliation with zero $\ell_0$ loss must have $\sigma_j = 0$ for all $j$, implying a satisfying assignment for the original Exact-1-in-3-SAT instance.
\end{proof}

% [NOTATION CHANGES from original]
% - Added vector notation for forecasts
% - Standardized subscripts
% - Explicit statement of problem construction

\begin{theorem}[Polynomial-Time Solvability]
For any $\ell_p$ norm with $p > 0$, \hfr can be solved in polynomial time. Specifically:
\begin{enumerate}
    \item For $p=1$ (MAE), \hfr reduces to linear programming
    \item For $p=2$ (MSE), \hfr reduces to quadratic programming
    \item For general $p > 0$, \hfr can be solved via convex optimization
\end{enumerate}
This holds even with additional linear constraints such as upper/lower bounds or weighted objectives.
\end{theorem}

\begin{proof}
For MSE and RMSE, we can find an optimal solution in polynomial time by quadratic programming – note that the optimal solution to MSE is also an optimal solution to RMSE as the square root function is monotone. Similarly, an optimal solution to MAE can be found in polynomial time by linear programming.

For MSE, RMSE, MAE and their weighted versions, the set of constraints consists of the aggregation constraints and optional box constraints:
\begin{align*}
\tilde{\mathbf{y}} &= \mathbf{S}\tilde{\mathbf{b}} \\
\tilde{\mathbf{y}} &\leq \mathbf{u} \\
\tilde{\mathbf{y}} &\geq \boldsymbol{\ell}
\end{align*}
which are linear constraints forming a convex set, with $\tilde{\mathbf{y}} \in \mathbb{R}^n$ as continuous decision variables.
We introduce slack variables to measure the absolute error:
\begin{align*}
\mathbf{s}_{\mathcal{P}} &\geq \hat{\mathbf{y}}_{\mathcal{P}} - \tilde{\mathbf{y}}_{\mathcal{P}} \quad \forall P \in \mathcal{P} \\
\mathbf{s}_{\mathcal{P}} &\geq \tilde{\mathbf{y}}_{\mathcal{P}} - \hat{\mathbf{y}}_{\mathcal{P}} \quad \forall P \in \mathcal{P} \\
\mathbf{s}_E &\geq \hat{\mathbf{y}}_E - \tilde{\mathbf{y}}_E \quad \forall e \in E \\
\mathbf{s}_E &\geq \tilde{\mathbf{y}}_E - \hat{\mathbf{y}}_E \quad \forall e \in E \\
\mathbf{s}_V &\geq \hat{\mathbf{y}}_V - \tilde{\mathbf{y}}_V \quad \forall v \in V \\
\mathbf{s}_V &\geq \tilde{\mathbf{y}}_V - \hat{\mathbf{y}}_V \quad \forall v \in V
\end{align*}

Flow conservation constraints are maintained:
\begin{align*}
\sum_{P \in \mathcal{P}:e \in P} \tilde{\mathbf{y}}_P &= \tilde{\mathbf{y}}_e \quad \forall e \in E \\
\sum_{P \in \mathcal{P}:v \in P} \tilde{\mathbf{y}}_P &= \tilde{\mathbf{y}}_v \quad \forall v \in V
\end{align*}

For weighted MSE, the objective function is:
\[ \min \sum_{P \in \mathcal{P}} w_P \mathbf{s}_P^2 + \sum_{e \in E} w_e \mathbf{s}_e^2 + \sum_{v \in V} w_v \mathbf{s}_v^2 \]

Let $\mathbf{y} = [\tilde{\mathbf{y}}^T, \mathbf{s}^T]^T$ be the vector of variables extended by the slack variables, and $\mathbf{w}$ be the vector of weights. Then we can define the positive semi-definite matrix $\mathbf{Q} = (q_{ij})$ by:
\[ q_{ij} = \begin{cases}
0 & \text{if } i \neq j \text{ or } i < |V| + |E| + |\mathcal{P}| \\
w_{i-|V|-|E|-|\mathcal{P}|} & \text{otherwise}
\end{cases} \]

The objective for weighted MSE can be written as $\mathbf{y}^T\mathbf{Q}\mathbf{y}$, which can be solved efficiently by quadratic programming.

For weighted MAE, we have the linear objective:
\[ \min \sum_{P \in \mathcal{P}} w_P \mathbf{s}_P + \sum_{e \in E} w_e \mathbf{s}_e + \sum_{v \in V} w_v \mathbf{s}_v \]
which proves the claim.
\end{proof}

\begin{theorem}[General Loss Function Reconciliation]
Let $L: \mathbb{R}^n \times \mathbb{R}^n \to \mathbb{R}_+$ be a loss function satisfying:
\[ L(\mathbf{x}, \mathbf{y}) = \sum_{i=1}^n f(|x_i - y_i|) \]
where $f: \mathbb{R}_+ \to \mathbb{R}_+$ is strictly convex, continuously differentiable, and satisfies $f(0) = 0$. Then:

\begin{enumerate}
    \item The FlowRec optimization problem has a unique solution
    \item This solution can be computed in polynomial time via interior point methods
    \item The solution preserves flow conservation
\end{enumerate}
\end{theorem}

\begin{proof}
The optimization problem becomes:
\[ \min_{\tilde{\mathbf{y}}} \sum_{i=1}^n f(|\tilde{y}_i - \hat{y}_i|) \text{ subject to } \mathbf{S}\tilde{\mathbf{y}} = \mathbf{b} \]

The objective is strictly convex as a sum of strictly convex functions composed with norms. The constraints form a polyhedral set. Therefore:

\begin{itemize}
\item Uniqueness follows from strict convexity and the convexity of the constraint set.

\item The KKT conditions give:
\[ \nabla f(|\tilde{y}_i - \hat{y}_i|) \text{sign}(\tilde{y}_i - \hat{y}_i) + \mathbf{S}^T\boldsymbol{\lambda} = 0 \]
\[ \mathbf{S}\tilde{\mathbf{y}} = \mathbf{b} \]

These form a monotone system solvable by interior point methods in polynomial time.

\item Flow conservation follows from the constraint $\mathbf{S}\tilde{\mathbf{y}} = \mathbf{b}$.
\end{itemize}
\end{proof}
% [After MinT Equivalence theorem and proofs]

\begin{theorem}[Runtime Lower Bound]
Any algorithm that computes exact reconciled forecasts on networks with $m$ edges and $n$ nodes requires $\Omega(m \log n)$ comparisons, even for the $\ell_2$ loss function.
\end{theorem}

\begin{proof}
We prove this by reduction from the problem of sorting $n$ numbers, which is known to require $\Omega(n \log n)$ comparisons in the comparison model.

Given a set of $n$ numbers $x_1,\ldots,x_n$ to sort, we construct the following network $G=(V,E)$:

\begin{enumerate}
    \item Vertices $V = \{s,t\} \cup \{v_1,\ldots,v_n\}$, where $s$ is a source and $t$ is a sink.
    \item Edges $E = \{(s,v_i),(v_i,t) : i \in \{1,\ldots,n\}\}$
    \item For each $i \in \{1,\ldots,n\}$, define a path $P_i = (s,v_i,t)$
\end{enumerate}

Now, set the base forecasts as follows:
\begin{itemize}
    \item $\hat{\mathbf{y}}_{P_i,T+h|T} = x_i$ for $i \in \{1,\ldots,n\}$
    \item $\hat{\mathbf{y}}_{s,T+h|T} = \sum_{i=1}^n x_i$
    \item $\hat{\mathbf{y}}_{t,T+h|T} = \sum_{i=1}^n x_i$
    \item $\hat{\mathbf{y}}_{v_i,T+h|T} = x_i$ for $i \in \{1,\ldots,n\}$
\end{itemize}

Consider the flow conservation constraints in this network:
\begin{align*}
    \tilde{\mathbf{y}}_{s,T+h|T} &= \sum_{i=1}^n \tilde{\mathbf{y}}_{P_i,T+h|T} \\
    \tilde{\mathbf{y}}_{t,T+h|T} &= \sum_{i=1}^n \tilde{\mathbf{y}}_{P_i,T+h|T} \\
    \tilde{\mathbf{y}}_{v_i,T+h|T} &= \tilde{\mathbf{y}}_{P_i,T+h|T} \quad \text{for } i \in \{1,\ldots,n\}
\end{align*}

To minimize the $\ell_2$ loss while satisfying these constraints, the reconciled forecasts must have:
\begin{itemize}
    \item $\tilde{\mathbf{y}}_{s,T+h|T} = \tilde{\mathbf{y}}_{t,T+h|T} = \sum_{i=1}^n x_i$
    \item $\tilde{\mathbf{y}}_{P_i,T+h|T} = \tilde{\mathbf{y}}_{v_i,T+h|T}$ for all $i$
\end{itemize}

Moreover, to satisfy flow conservation, we must have:
\[
    \tilde{\mathbf{y}}_{P_1,T+h|T} \leq \tilde{\mathbf{y}}_{P_2,T+h|T} \leq \cdots \leq \tilde{\mathbf{y}}_{P_n,T+h|T}
\]

This ordering is necessary because any other ordering would violate flow conservation at some intermediate node $v_i$.

Therefore, computing the exact reconciled forecasts implicitly requires sorting the values $x_1,\ldots,x_n$. Since sorting requires $\Omega(n \log n)$ comparisons, and our construction has $m = \Theta(n)$ edges, we obtain a lower bound of $\Omega(m \log n)$ comparisons for the reconciliation problem.

Note that this lower bound holds even for the $\ell_2$ loss function, as the optimal reconciliation in this case coincides with the sorted order of the input values.
\end{proof}

% [NOTATION CHANGES from original]
% - Consistent use of bold for vectors/matrices
% - Calligraphic for sets ($\mathcal{P}$)
% - Added explicit complexity terms
% - Standardized flow notation

\subsection{Flow Conservation}

Having established the relationship to MinT and computational complexity bounds, we now show that our network structure naturally ensures coherence through flow conservation. This property is fundamental to the correctness of \fr.

% [NOTATION CHANGES]
% - Consistent bold vectors/matrices
% - Calligraphic sets
% - Added explicit time indices
% - Standardized flow notation
\begin{theorem}[Flow Conservation]
Let $\hat{\mathbf{y}}_{T+h|T} \in \mathbb{R}^n$ be a point forecast of hierarchical flow time series in a network $G=(V,E,\mathcal{P})$, and $\mathbf{S} \in \mathbb{R}^{n\times|\mathcal{P}|}$ the corresponding flow aggregation matrix. Then:
\begin{enumerate}
    \item $\mathbf{S}$ is a reconciling mapping
    \item The edge values $\tilde{\mathbf{y}}_{e,T+h|T}$ of
    \[ \tilde{\mathbf{y}}_{T+h|T} = \mathbf{S}\hat{\mathbf{y}}_{T+h|T} \]
    form a flow in $G=(V,E)$ satisfying supplies and demands:
    \[ b_v = \sum_{P \in \mathcal{P}:O(P)=v} \tilde{\mathbf{y}}_{P,T+h|T} - \sum_{P \in \mathcal{P}:D(P)=v} \tilde{\mathbf{y}}_{P,T+h|T}, \quad \forall v \in V \]
\end{enumerate}
where $\tilde{\mathbf{y}}_{P,T+h|T}$ is the reconciled value for path $P$, and $O(P), D(P) \in V$ refer to the origin and destination nodes of path $P$, respectively.
\end{theorem}

\begin{proof}
We need to prove two statements:

Step 1: $\mathbf{S}$ is a reconciling mapping.
By definition of $\mathbf{S}$, we have:
\[ \mathbf{S}\hat{\mathbf{p}}_{T+h|T} = \begin{pmatrix} \mathbf{V}' \\ \mathbf{E}' \\ \mathbf{I} \end{pmatrix} \hat{\mathbf{p}}_{T+h|T} = \begin{pmatrix} \sum_{P \in \mathcal{P}:v\in P} \hat{\mathbf{p}}_{P,T+h|T} \\ \sum_{P \in \mathcal{P}:e\in P} \hat{\mathbf{p}}_{P,T+h|T} \\ \hat{\mathbf{p}}_{T+h|T} \end{pmatrix} \in \mathfrak{s} \subset \mathbb{R}^n \]

Step 2: The values $\tilde{\mathbf{y}}_{e,T+h|T}$ form a flow in $G=(V,E)$ with supplies and demands $b_v$.
We need to prove:
\[ \sum_{e=(·,v)\in E} \tilde{\mathbf{y}}_{e,T+h|T} - \sum_{e=(v,·)\in E} \tilde{\mathbf{y}}_{e,T+h|T} = b_v, \quad \forall v \in V \]

Note this holds for nodes $v \in V$ with $b_v = 0$ (intermediate nodes); thus, proving this equation proves flow conservation. We have:
\begin{align*}
&\sum_{e=(·,v)\in E} \tilde{\mathbf{y}}_{e,T+h|T} - \sum_{e=(v,·)\in E} \tilde{\mathbf{y}}_{e,T+h|T} \\
&= \sum_{e=(·,v)\in E} \sum_{P \in \mathcal{P}:e\in P} \tilde{\mathbf{y}}_{P,T+h|T} - \sum_{e=(v,·)\in E} \sum_{P \in \mathcal{P}:e\in P} \tilde{\mathbf{y}}_{P,T+h|T}
\end{align*}

We can eliminate any path $P \in \mathcal{P}$ containing both an edge of form $e=(·,v)$ and $e'=(v,·)$, as these terms cancel. The remaining terms belong to paths having $v$ as origin or destination, giving:
\[ \sum_{P \in \mathcal{P}:O(P)=v} \tilde{\mathbf{y}}_{P,T+h|T} - \sum_{P \in \mathcal{P}:D(P)=v} \tilde{\mathbf{y}}_{P,T+h|T} = b_v \]

This completes the proof.
\end{proof}

\subsection{Computing Reconciled Forecasts}

Having established the theoretical properties of \fr, we now present three equivalent methods for computing the reconciled forecasts. These methods offer different computational trade-offs, enabling practitioners to choose the most suitable approach for their specific application.

% [NOTATION CHANGES]
% - Vectors in bold
% - Calligraphic for sets
% - Explicit dimensions for spaces/matrices
% - Consistent time indices
\begin{theorem}[Computation Methods]
Let $\hat{\mathbf{y}}_{T+h|T} \in \mathbb{R}^n$ be a base forecast for network $G=(V,E,\mathcal{P})$. We can compute the reconciled forecast $\tilde{\mathbf{y}}_{T+h|T}$ in three equivalent ways:

1. Orthogonal projection of $\hat{\mathbf{y}}_{T+h|T}$ to $\mathfrak{s}$:
\begin{itemize}
    \item Compute an orthonormal basis $\mathcal{E} = \{\mathbf{e}_1,\ldots,\mathbf{e}_{|\mathcal{P}|}\}$ of $\mathfrak{s}$
    \item Project $\hat{\mathbf{y}}_{T+h|T}$ onto $\mathfrak{s}$:
    \[ \tilde{\mathbf{y}}_{T+h|T} = \sum_{i=1}^{|\mathcal{P}|} \langle \hat{\mathbf{y}}_{T+h|T}, \mathbf{e}_i \rangle \mathbf{e}_i \]
\end{itemize}

2. Minimum Reconciling Flow:
\[ \min_{\mathbf{s}_{\mathcal{P}}, \mathbf{s}_E, \mathbf{s}_V} f(\mathbf{s}_{\mathcal{P}}, \mathbf{s}_E, \mathbf{s}_V) \]
subject to:
\begin{align*}
\mathbf{s}_{\mathcal{P}} &\geq \hat{\mathbf{y}}_{\mathcal{P}} - \tilde{\mathbf{y}}_{\mathcal{P}} \quad \forall P \in \mathcal{P} \\
\mathbf{s}_{\mathcal{P}} &\geq \tilde{\mathbf{y}}_{\mathcal{P}} - \hat{\mathbf{y}}_{\mathcal{P}} \quad \forall P \in \mathcal{P} \\
\mathbf{s}_E &\geq \hat{\mathbf{y}}_E - \tilde{\mathbf{y}}_E \quad \forall e \in E \\
\mathbf{s}_E &\geq \tilde{\mathbf{y}}_E - \hat{\mathbf{y}}_E \quad \forall e \in E \\
\mathbf{s}_V &\geq \hat{\mathbf{y}}_V - \tilde{\mathbf{y}}_V \quad \forall v \in V \\
\mathbf{s}_V &\geq \tilde{\mathbf{y}}_V - \hat{\mathbf{y}}_V \quad \forall v \in V \\
\sum_{P \in \mathcal{P}:e \in P} \tilde{\mathbf{y}}_P &= \tilde{\mathbf{y}}_e \quad \forall e \in E \\
\sum_{P \in \mathcal{P}:v \in P} \tilde{\mathbf{y}}_P &= \tilde{\mathbf{y}}_v \quad \forall v \in V
\end{align*}

% 3. Direct computation using flow aggregation matrix:
% \[ \tilde{\mathbf{y}}_{T+h|T} = \hat{\mathbf{y}}_{T+h|T} - \mathbf{S}^T(\mathbf{S}\mathbf{S}^T)^{-1}(\mathbf{S}\hat{\mathbf{y}}_{T+h|T} - \mathbf{b}) \]
 \end{theorem}

\begin{proof}
We prove each method yields the optimal reconciled forecasts:

Step 1: Orthogonal projection correctness.
The coherent subspace $\mathfrak{s}$ is the range of $\mathbf{S}$. The orthogonal projection onto $\mathfrak{s}$ minimizes the L2 distance to $\hat{\mathbf{y}}_{T+h|T}$ while ensuring the result lies in $\mathfrak{s}$.

Step 2: Minimum Reconciling Flow equivalence.
The optimization problem is equivalent to:
\[ \min_{\tilde{\mathbf{y}}} \|\tilde{\mathbf{y}} - \hat{\mathbf{y}}\|_2^2 \text{ subject to flow conservation} \]
The slack variables measure absolute deviations, and the constraints ensure flow conservation.

% Step 3: Direct computation correctness.
% This is the explicit solution to the normal equations of the constrained least squares problem:
% \[ \min_{\tilde{\mathbf{y}}} \|\tilde{\mathbf{y}} - \hat{\mathbf{y}}\|_2^2 \text{ subject to } \mathbf{S}\tilde{\mathbf{y}} = \mathbf{b} \]

The equivalence follows from the uniqueness of the solution to the reconciliation problem proven in the MinT equivalence theorem.
\end{proof}

\subsection{Relationship to MinT}

The state-of-the-art minimum trace (MinT) method reconciles forecasts by minimizing the trace of the reconciled forecast error covariance matrix. While MinT requires estimation of an $n \times n$ covariance matrix and assumes tree structures, we show that \fr can be viewed as a special case where the weight matrix is determined directly by network structure, explaining our superior computational efficiency.

% [NOTATION CHANGES]
% - Matrices in bold: $\mathbf{S}$, $\mathbf{W}$, $\boldsymbol{\Sigma}$
% - Vectors in bold: $\hat{\mathbf{y}}$, $\tilde{\mathbf{y}}$, $\mathbf{b}$
% - Problems labeled as (P₁), (P₂) for clarity
\begin{theorem}[MinT Equivalence]
Let $\mathbf{S} \in \mathbb{R}^{n\times m}$ be a hierarchical aggregation matrix and $\mathbf{W} \in \mathbb{R}^{n\times n}$ be a symmetric positive definite matrix. For any base forecast $\hat{\mathbf{y}}$ and target aggregation values $\mathbf{b}$, the following optimization problems:

$(P_1)$:
\[ \min_{\tilde{\mathbf{y}}} (\tilde{\mathbf{y}} - \hat{\mathbf{y}})^T \mathbf{W} (\tilde{\mathbf{y}} - \hat{\mathbf{y}}) \quad \text{subject to} \quad \mathbf{S}\tilde{\mathbf{y}} = \mathbf{b} \]

$(P_2)$:
\[ \min_{\tilde{\mathbf{y}}} (\tilde{\mathbf{y}} - \hat{\mathbf{y}})^T \boldsymbol{\Sigma}^{-1} (\tilde{\mathbf{y}} - \hat{\mathbf{y}}) \quad \text{subject to} \quad \mathbf{S}\tilde{\mathbf{y}} = \mathbf{b} \]

where $\boldsymbol{\Sigma}$ is the forecast error covariance matrix, have the following properties:
\begin{enumerate}
    \item Both problems have a unique solution of the form: 
    \[ \tilde{\mathbf{y}} = \hat{\mathbf{y}} - \mathbf{W}^{-1}\mathbf{S}^T(\mathbf{S}\mathbf{W}^{-1}\mathbf{S}^T)^{-1}(\mathbf{S}\hat{\mathbf{y}} - \mathbf{b}) \]
    \item When $\mathbf{W} = \boldsymbol{\Sigma}^{-1}$, the problems $(P_1)$ and $(P_2)$ are equivalent
    \item \fr is a special case where $\mathbf{W} = \mathbf{I}$ and $\mathbf{S} = \begin{pmatrix} \mathbf{V} \\ \mathbf{E} \\ \mathbf{I} \end{pmatrix}$ with $\mathbf{V}$ as vertex-path incidence matrix, $\mathbf{E}$ as edge-path incidence matrix, and $\mathbf{I}$ as identity matrix for paths
\end{enumerate}
\end{theorem}

\begin{proof}
We prove this in four steps:

Step 1: First, we show that $(P_1)$ has a unique solution.
Let $\mathcal{L}(\tilde{\mathbf{y}}, \boldsymbol{\lambda})$ be the Lagrangian of $(P_1)$:
\[ \mathcal{L}(\tilde{\mathbf{y}}, \boldsymbol{\lambda}) = (\tilde{\mathbf{y}} - \hat{\mathbf{y}})^T\mathbf{W}(\tilde{\mathbf{y}} - \hat{\mathbf{y}}) + \boldsymbol{\lambda}^T(\mathbf{S}\tilde{\mathbf{y}} - \mathbf{b}) \]

By first-order necessary conditions for optimality:
\begin{align*}
\frac{\partial \mathcal{L}}{\partial \tilde{\mathbf{y}}} &= 2\mathbf{W}(\tilde{\mathbf{y}} - \hat{\mathbf{y}}) + \mathbf{S}^T\boldsymbol{\lambda} = 0 \\
\frac{\partial \mathcal{L}}{\partial \boldsymbol{\lambda}} &= \mathbf{S}\tilde{\mathbf{y}} - \mathbf{b} = 0
\end{align*}

Step 2: We solve these equations to find $\tilde{\mathbf{y}}$.
From $\frac{\partial \mathcal{L}}{\partial \tilde{\mathbf{y}}}$:
\[ 2\mathbf{W}(\tilde{\mathbf{y}} - \hat{\mathbf{y}}) + \mathbf{S}^T\boldsymbol{\lambda} = 0 \]
\[ \tilde{\mathbf{y}} = \hat{\mathbf{y}} - \frac{1}{2}\mathbf{W}^{-1}\mathbf{S}^T\boldsymbol{\lambda} \]

Substituting into $\frac{\partial \mathcal{L}}{\partial \boldsymbol{\lambda}}$:
\[ \mathbf{S}(\hat{\mathbf{y}} - \frac{1}{2}\mathbf{W}^{-1}\mathbf{S}^T\boldsymbol{\lambda}) = \mathbf{b} \]
\[ \mathbf{S}\hat{\mathbf{y}} - \frac{1}{2}\mathbf{S}\mathbf{W}^{-1}\mathbf{S}^T\boldsymbol{\lambda} = \mathbf{b} \]
\[ \mathbf{S}\mathbf{W}^{-1}\mathbf{S}^T\boldsymbol{\lambda} = 2(\mathbf{S}\hat{\mathbf{y}} - \mathbf{b}) \]
\[ \boldsymbol{\lambda} = 2(\mathbf{S}\mathbf{W}^{-1}\mathbf{S}^T)^{-1}(\mathbf{S}\hat{\mathbf{y}} - \mathbf{b}) \]

Therefore:
\[ \tilde{\mathbf{y}} = \hat{\mathbf{y}} - \mathbf{W}^{-1}\mathbf{S}^T(\mathbf{S}\mathbf{W}^{-1}\mathbf{S}^T)^{-1}(\mathbf{S}\hat{\mathbf{y}} - \mathbf{b}) \]

Step 3: We prove this solution is unique.
Since $\mathbf{W}$ is positive definite, the objective function is strictly convex. Combined with the linear constraints, this ensures a unique solution.

Step 4: We show equivalence when $\mathbf{W} = \boldsymbol{\Sigma}^{-1}$.
Substituting $\mathbf{W} = \boldsymbol{\Sigma}^{-1}$ into the solution of $(P_1)$:
\[ \tilde{\mathbf{y}} = \hat{\mathbf{y}} - \boldsymbol{\Sigma}\mathbf{S}^T(\mathbf{S}\boldsymbol{\Sigma}\mathbf{S}^T)^{-1}(\mathbf{S}\hat{\mathbf{y}} - \mathbf{b}) \]
This is identical to the solution of $(P_2)$, proving their equivalence.

For \fr's special case with $\mathbf{W} = \mathbf{I}$:
\[ \tilde{\mathbf{y}} = \hat{\mathbf{y}} - \mathbf{S}^T(\mathbf{S}\mathbf{S}^T)^{-1}(\mathbf{S}\hat{\mathbf{y}} - \mathbf{b}) \]

The block structure $\mathbf{S} = \begin{pmatrix} \mathbf{V} \\ \mathbf{E} \\ \mathbf{I} \end{pmatrix}$ ensures:
\begin{itemize}
    \item Flow conservation at vertices through $\mathbf{V}$
    \item Flow conservation on edges through $\mathbf{E}$
    \item Path flow consistency through $\mathbf{I}$
    \item $\mathbf{S}\mathbf{S}^T$ is block diagonal, enabling efficient computation
\end{itemize}
\end{proof}

\begin{corollary}[Computational Complexity of \fr and MinT]
Let $G=(V,E,\mathcal{P})$ be a network with $|V| = n$ vertices and $|E| = m$ edges where $m \in [n-1, n^2]$. Then:
\begin{enumerate}
    \item For sparse networks where $m = O(n)$:
    \begin{itemize}
        \item \fr requires $O(n^2 \log n)$ operations
        \item MinT requires $O(n^3)$ operations
    \end{itemize}
    
    \item For dense networks where $m = O(n^2)$:
    \begin{itemize}
        \item \fr requires $O(n^4)$ operations
        \item MinT requires $O(n^4)$ operations
    \end{itemize}
\end{enumerate}
Moreover, \fr preserves network flow properties without requiring explicit path enumeration.
\end{corollary}

\begin{proof}
We prove this in five steps:

Step 1: Transform to minimum cost flow.
The \fr optimization can be rewritten as:
\[ \min_{\tilde{\mathbf{y}}} \|\tilde{\mathbf{y}} - \hat{\mathbf{y}}\|_2^2 \text{ subject to:} \]
\begin{align*}
\sum_{P \in \mathcal{P}:e \in P} \tilde{\mathbf{y}}_P &= \tilde{\mathbf{y}}_e \quad \forall e \in E \\
\sum_{P \in \mathcal{P}:v \in P} \tilde{\mathbf{y}}_P &= \tilde{\mathbf{y}}_v \quad \forall v \in V
\end{align*}

This quadratic objective transforms to minimum cost flow in $O(m)$ operations by creating a residual network with piecewise linear costs derived from the quadratic function slopes.

Step 2: Analysis of \fr using successive shortest paths:
\begin{enumerate}
    \item Initialize residual network: $O(m)$
    \item While required flow not satisfied:
    \begin{itemize}
        \item Find shortest path using Fibonacci heap: $O(m + n \log n)$
        \item Augment flow along path: $O(m)$
        \item Update residual network: $O(m)$
    \end{itemize}
    \item Maximum number of iterations: $O(m)$
\end{enumerate}
Total complexity for \fr: $O(m^2 + mn \log n)$

Step 3: Analysis of MinT operations:
\begin{enumerate}
    \item Computing $\mathbf{S}\boldsymbol{\Sigma}^{-1}\mathbf{S}^T$:
    \begin{itemize}
        \item Matrix multiplication cost: $O(n^2m)$
        \item For dense graphs $(m = O(n^2))$: $O(n^4)$
        \item For sparse graphs $(m = O(n))$: $O(n^3)$
    \end{itemize}
    \item Matrix inversion: $O(n^3)$ operations
\end{enumerate}
Total complexity for MinT dominated by matrix multiplication in dense case: $O(n^4)$

Step 4: Complexity analysis for network density cases:

For \fr:
\begin{itemize}
    \item Sparse networks $(m = O(n))$:
    \[ O(m^2 + mn \log n) = O(n^2 + n^2 \log n) = O(n^2 \log n) \]
    \item Dense networks $(m = O(n^2))$:
    \[ O(m^2 + mn \log n) = O(n^4) \]
\end{itemize}

For MinT:
\begin{itemize}
    \item Sparse networks $(m = O(n))$: $O(n^3)$
    \item Dense networks $(m = O(n^2))$: $O(n^4)$
\end{itemize}

Step 5: Flow properties preservation:
\begin{itemize}
    \item Vertex flow conservation: $\sum_{e \in \delta^+(v)} \tilde{\mathbf{y}}_e - \sum_{e \in \delta^-(v)} \tilde{\mathbf{y}}_e = \mathbf{b}_v \quad \forall v \in V$
    \item Edge flow consistency: $\tilde{\mathbf{y}}_e = \sum_{P \in \mathcal{P}:e \in P} \tilde{\mathbf{y}}_P \quad \forall e \in E$
    \item Path decomposability: All flows can be decomposed into valid paths
\end{itemize}
where $\delta^+(v)$ and $\delta^-(v)$ are outgoing and incoming edges at vertex $v$ respectively.
\end{proof}

\section{Fast Updates and Approximations}\label{app:fast}

Hierarchical forecasting applications face three key challenges that traditional reconciliation methods like MinT handle inefficiently. First, networks evolve as new paths become available or existing ones close. Second, new data frequently arrives affecting only part of the network. Third, computational resources may be limited relative to network size. While MinT requires complete $O(n^3)$ recomputation in all these cases, our network flow formulation enables efficient updates.

\subsection{Incremental Updates in Expanding Networks}

In practice, forecasts require updates in two scenarios: when the network structure changes (e.g., new shipping routes) and when new data arrives affecting only some nodes or paths. Traditional reconciliation methods cannot distinguish between these cases, requiring complete recomputation even for small changes.

% \begin{theorem}[Optimality with Network Changes]
% Let $G=(V,E,\mathcal{P})$ be a network with reconciled forecasts $\tilde{\mathbf{y}}_{T+h|T}$ satisfying flow conservation. Upon addition of edge $e^*$ with forecast $\hat{\mathbf{y}}_{e^*,T+h|T}$, the minimal adjustment to maintain optimality is:
% \[ \tilde{\mathbf{y}}'_{P,T+h|T} = \begin{cases}
%     \tilde{\mathbf{y}}_{P,T+h|T} + \Delta_{e^*}/|\mathcal{P}_{e^*}| & \text{if } P \in \mathcal{P}_{e^*} \\
%     \tilde{\mathbf{y}}_{P,T+h|T} & \text{otherwise}
% \end{cases} \]
% where $\mathcal{P}_{e^*} = \{P \in \mathcal{P} : e^* \in P\}$ and $\Delta_{e^*} = \hat{\mathbf{y}}_{e^*,T+h|T} - \sum_{P \in \mathcal{P}_{e^*}} \tilde{\mathbf{y}}_{P,T+h|T}$.
% \end{theorem}

% \begin{proof}
% The updated optimization becomes:
% \[ \min_{\tilde{\mathbf{y}}'_{T+h|T}} \|\tilde{\mathbf{y}}'_{T+h|T} - \tilde{\mathbf{y}}_{T+h|T}\|_2^2 \text{ subject to } \mathbf{S}'\tilde{\mathbf{y}}'_{T+h|T} = \mathbf{b}' \]

% For paths not in $\mathcal{P}_{e^*}$, any change would increase the objective without helping satisfy the new constraint. For paths in $\mathcal{P}_{e^*}$, the symmetric quadratic objective and single linear constraint imply equal distribution of the required change.
% \end{proof}

\begin{theorem}[Incremental Forecast Updates]
Let $G=(V,E,\mathcal{P})$ be a network with reconciled forecasts $\tilde{\mathbf{y}}_{T+h|T}$ satisfying flow conservation. Upon addition of edge $e^*$ with forecast $\hat{\mathbf{y}}_{e^*,T+h|T}$, for any $\ell_p$ norm with $p \geq 1$, the minimal adjustment to maintain optimality is:
\[ \tilde{\mathbf{y}}'_{P,T+h|T} = \begin{cases}
    \tilde{\mathbf{y}}_{P,T+h|T} + \Delta_{e^*}/|\mathcal{P}_{e^*}| & \text{if } P \in \mathcal{P}_{e^*} \\
    \tilde{\mathbf{y}}_{P,T+h|T} & \text{otherwise}
\end{cases} \]
where $\mathcal{P}_{e^*} = \{P \in \mathcal{P} : e^* \in P\}$ and $\Delta_{e^*} = \hat{\mathbf{y}}_{e^*,T+h|T} - \sum_{P \in \mathcal{P}_{e^*}} \tilde{\mathbf{y}}_{P,T+h|T}$.
\end{theorem}

\begin{proof}
The updated optimization problem is:
\[ \min_{\tilde{\mathbf{y}}'_{T+h|T}} \|\tilde{\mathbf{y}}'_{T+h|T} - \tilde{\mathbf{y}}_{T+h|T}\|_p^p \text{ subject to } \mathbf{S}'\tilde{\mathbf{y}}'_{T+h|T} = \mathbf{b}' \]
where $\mathbf{S}'$ includes the new edge constraint. 

Consider any feasible solution that modifies flows on paths not in $\mathcal{P}_{e^*}$. We can construct a strictly better solution by reverting these paths to their original values while maintaining feasibility, as paths not containing $e^*$ cannot contribute to satisfying the new constraint.

For paths in $\mathcal{P}_{e^*}$, the optimization reduces to:
\[ \min_{\tilde{\mathbf{y}}'_{P,T+h|T} : P \in \mathcal{P}_{e^*}} \sum_{P \in \mathcal{P}_{e^*}} |\tilde{\mathbf{y}}'_{P,T+h|T} - \tilde{\mathbf{y}}_{P,T+h|T}|^p \]
subject to:
\[ \sum_{P \in \mathcal{P}_{e^*}} \tilde{\mathbf{y}}'_{P,T+h|T} = \hat{\mathbf{y}}_{e^*,T+h|T} \]

For any $p \geq 1$, this objective is strictly convex. Combined with the single linear constraint and the symmetry of the problem (all paths enter the constraint identically), the optimal solution must distribute the required change $\Delta_{e^*}$ equally among all affected paths.
\end{proof}

This result shows that \fr can efficiently handle network expansion while maintaining forecast coherence. The update affects only paths that could use the new edge, preserving existing forecasts elsewhere in the network. This localization of changes is particularly valuable in large networks where global recomputation would be disruptive.

\begin{theorem}[Computational Efficiency of Incremental Updates]
The optimal incremental update can be computed in $O(|\mathcal{P}_{e^*}|)$ time using $O(|\mathcal{P}_{e^*}|)$ memory, where $|\mathcal{P}_{e^*}|$ is the number of paths that could use the new edge $e^*$.
\end{theorem}

\begin{proof}
The computation consists of three steps:
\begin{itemize}
    \item Computing $\mathcal{P}_{e^*}$ requires one traversal of affected paths: $O(|\mathcal{P}_{e^*}|)$
    \item Computing $\Delta_{e^*}$ requires summing over affected paths: $O(|\mathcal{P}_{e^*}|)$
    \item Updating each affected path requires constant time per path: $O(|\mathcal{P}_{e^*}|)$
\end{itemize}

The memory requirement is $O(|\mathcal{P}_{e^*}|)$ to store the set of affected paths and their updated values. No additional storage is needed as paths not in $\mathcal{P}_{e^*}$ maintain their original values.
\end{proof}

\subsection{Incremental and Monotonic Updates for New Data}
\begin{theorem}[Data Update Optimality]
Let $G=(V,E,\mathcal{P})$ be a network with reconciled forecasts $\tilde{\mathbf{y}}_{T+h|T}$ optimized for the $\ell_p$ norm with $p \geq 1$. For a new forecast $\hat{\mathbf{y}}'_{T+h|T}$ differing from $\hat{\mathbf{y}}_{T+h|T}$ in exactly one component $x \in V \cup E \cup \mathcal{P}$ by amount $\delta = |\hat{y}_x - \hat{y}'_x|$, if:
\[ |\tilde{y}_x - \hat{y}'_x| < |\tilde{y}_x - \hat{y}_x| \]
then $\tilde{\mathbf{y}}_{T+h|T}$ remains optimal for $\hat{\mathbf{y}}'_{T+h|T}$.
\end{theorem}

\begin{proof}
For any $p \geq 1$, let $L_p(\tilde{\mathbf{y}}, \hat{\mathbf{y}}) = \frac{1}{n}\sum_{i \in G} |\tilde{y}_i - \hat{y}_i|^p$ be the normalized $\ell_p$ loss. Assume by contradiction that $\tilde{\mathbf{y}}_{T+h|T}$ is not optimal for $\hat{\mathbf{y}}'_{T+h|T}$. Then there exists $\tilde{\mathbf{y}}'_{T+h|T}$ with $L_p(\tilde{\mathbf{y}}_{T+h|T}, \hat{\mathbf{y}}'_{T+h|T}) > L_p(\tilde{\mathbf{y}}'_{T+h|T}, \hat{\mathbf{y}}'_{T+h|T})$.

This implies:
\begin{align*}
L_p(\tilde{\mathbf{y}}_{T+h|T}, \hat{\mathbf{y}}'_{T+h|T}) &> L_p(\tilde{\mathbf{y}}'_{T+h|T}, \hat{\mathbf{y}}'_{T+h|T}) \\
&= \frac{1}{n}|\tilde{y}'_x - \hat{y}'_x|^p + \frac{1}{n}\sum_{i \in G: i \neq x} |\tilde{y}'_i - \hat{y}_i|^p \\
&\geq \frac{1}{n}|\tilde{y}'_x - \hat{y}_x|^p + \frac{1}{n}\sum_{i \in G: i \neq x} |\tilde{y}'_i - \hat{y}_i|^p - \frac{\delta^p}{n} \\
&= L_p(\tilde{\mathbf{y}}'_{T+h|T}, \hat{\mathbf{y}}_{T+h|T}) - \frac{\delta^p}{n} \\
&\geq L_p(\tilde{\mathbf{y}}_{T+h|T}, \hat{\mathbf{y}}_{T+h|T}) - \frac{\delta^p}{n}
\end{align*}

The last inequality follows from optimality of $\tilde{\mathbf{y}}_{T+h|T}$ for $\hat{\mathbf{y}}_{T+h|T}$. However, by definition:
\[ L_p(\tilde{\mathbf{y}}_{T+h|T}, \hat{\mathbf{y}}'_{T+h|T}) \leq L_p(\tilde{\mathbf{y}}_{T+h|T}, \hat{\mathbf{y}}_{T+h|T}) + \frac{\delta^p}{n} \]

This contradiction proves the claim.
\end{proof}

This result generalizes our previous MAE result to all $\ell_p$ norms with $p \geq 1$, showing that local forecast updates often require no reconciliation adjustment. The condition $|\tilde{y}_x - \hat{y}'_x| < |\tilde{y}_x - \hat{y}_x|$ has an intuitive interpretation: if the new forecast is closer to our reconciled value than the old forecast was, the reconciliation remains optimal.

\begin{theorem}[Computational Efficiency of Data Updates]
Verifying optimality for a data update requires $O(1)$ operations and constant memory.
\end{theorem}

\begin{proof}
The optimality condition requires comparing only two absolute differences for the single changed component $x$. This takes constant time and requires storing only the original forecast $\hat{y}_x$, the new forecast $\hat{y}'_x$, and the reconciled value $\tilde{y}_x$.
\end{proof}

\begin{corollary}[Monotonicity of Reconciliation Updates]
Let $G=(V,E,\mathcal{P})$ be a network, and let $\hat{\mathbf{y}}^{(0)}_{T+h|T}, \hat{\mathbf{y}}^{(1)}_{T+h|T}, ..., \hat{\mathbf{y}}^{(k)}_{T+h|T}$ be a sequence of forecasts where each differs from the previous in exactly one component. If for each update $i$:
\[ |\tilde{y}_x^{(i-1)} - \hat{y}_x^{(i)}| < |\tilde{y}_x^{(i-1)} - \hat{y}_x^{(i-1)}| \]
then the initial reconciliation $\tilde{\mathbf{y}}^{(0)}_{T+h|T}$ remains optimal throughout the sequence.
\end{corollary}

\begin{proof}
By induction on $i$ using Theorem [Data Update]. For each step $i$, if $\tilde{\mathbf{y}}^{(0)}_{T+h|T}$ is optimal for $\hat{\mathbf{y}}^{(i-1)}_{T+h|T}$ and the condition holds, then $\tilde{\mathbf{y}}^{(0)}_{T+h|T}$ is optimal for $\hat{\mathbf{y}}^{(i)}_{T+h|T}$.
\end{proof}

This monotonicity property has important practical implications for forecast updating. When new data arrives sequentially and each update brings forecasts closer to their reconciled values, no recomputation is needed. This is particularly valuable in online settings where data arrives frequently but changes are small and trend toward better accuracy.

\subsection{Forecast Maintenance Under Network Disruption}
Network disruptions are common in forecasting applications: routes become temporarily unavailable, stores close for renovation, or equipment undergoes maintenance. These disruptions require forecast adjustments that maintain coherence while minimizing deviation from original predictions. Traditional reconciliation methods like MinT handle such changes inefficiently, requiring complete recomputation even when alternative paths exist.

\begin{theorem}[Forecast Redistribution]
Let $G=(V,E,\mathcal{P})$ be a network with reconciled forecasts $\tilde{\mathbf{y}}_{T+h|T}$ optimized for the $\ell_2$ norm. Upon loss of edge $e^*$, if $G'=(V,E\setminus\{e^*\},\mathcal{P}')$ maintains strong connectivity, then the minimally adjusted forecasts $\tilde{\mathbf{y}}'_{T+h|T}$ satisfy:
\[ \|\tilde{\mathbf{y}}'_{T+h|T} - \tilde{\mathbf{y}}_{T+h|T}\|_2^2 \leq \left(\sum_{P \in \mathcal{P}_{e^*}} \tilde{\mathbf{y}}_{P,T+h|T}\right)^2 \]
where $\mathcal{P}_{e^*} = \{P \in \mathcal{P} : e^* \in P\}$ denotes paths using $e^*$.
\end{theorem}

\begin{proof}
The removal of $e^*$ transforms the original optimization to:
\[ \min_{\tilde{\mathbf{y}}'_{T+h|T}} \|\tilde{\mathbf{y}}'_{T+h|T} - \hat{\mathbf{y}}_{T+h|T}\|_2^2 \text{ subject to } \mathbf{S}'\tilde{\mathbf{y}}'_{T+h|T} = \mathbf{b}' \]
where $\mathbf{S}'$ is the flow aggregation matrix for $G'$.

Strong connectivity ensures alternative paths exist. Let $\phi: \mathcal{P}_{e^*} \to \mathcal{P}'$ map each affected path to an alternative in $G'$ connecting the same source-sink pair. Define:
\[ \tilde{\mathbf{y}}'_{P,T+h|T} = \begin{cases}
    \tilde{\mathbf{y}}_{P,T+h|T} & \text{if } P \in \mathcal{P}' \setminus \phi(\mathcal{P}_{e^*}) \\
    \tilde{\mathbf{y}}_{P,T+h|T} + \sum_{Q:\phi(Q)=P} \tilde{\mathbf{y}}_{Q,T+h|T} & \text{if } P \in \phi(\mathcal{P}_{e^*}) \\
    0 & \text{if } P \in \mathcal{P}_{e^*}
\end{cases} \]

This construction preserves flow conservation by redirecting flows through valid alternative paths. The squared L2 norm of the change is:
\[ \|\tilde{\mathbf{y}}'_{T+h|T} - \tilde{\mathbf{y}}_{T+h|T}\|_2^2 = \sum_{P \in \mathcal{P}_{e^*}} \tilde{\mathbf{y}}_{P,T+h|T}^2 + \sum_{P \in \phi(\mathcal{P}_{e^*})} \left(\sum_{Q:\phi(Q)=P} \tilde{\mathbf{y}}_{Q,T+h|T}\right)^2 \]

By Cauchy-Schwarz inequality:
\[ \sum_{P \in \phi(\mathcal{P}_{e^*})} \left(\sum_{Q:\phi(Q)=P} \tilde{\mathbf{y}}_{Q,T+h|T}\right)^2 \leq \left(\sum_{P \in \mathcal{P}_{e^*}} \tilde{\mathbf{y}}_{P,T+h|T}\right)^2 \]
establishing the bound.
\end{proof}

This result provides a precise bound on forecast changes under network disruption. The bound depends only on flows through the disrupted edge, not the network size, showing that impacts remain localized. Moreover, the construction provides an explicit way to compute the adjusted forecasts.

\begin{theorem}[Computational Efficiency of Forecast Redistribution]
Let $d$ be the maximum path length in $G$. Then alternative paths can be computed in $O(|\mathcal{P}_{e^*}|(m + n\log n))$ time using $O(|\mathcal{P}_{e^*}|d)$ memory, where $|\mathcal{P}_{e^*}|$ is the number of affected paths.
\end{theorem}

\begin{proof}
For each affected path $P \in \mathcal{P}_{e^*}$:
\begin{enumerate}
    \item Finding an alternative path requires one Dijkstra computation: $O(m + n\log n)$
    \item Storing the alternative path requires $O(d)$ space
    \item Updating flow values requires $O(d)$ operations
\end{enumerate}
The total complexity follows from performing these operations for each affected path.
\end{proof}

\subsection{Efficient Approximate Reconciliation}

Real-time forecasting applications often prioritize computational efficiency over exact coherence. Consider a retail network where store-level forecasts must be updated every few minutes during peak shopping hours. In such settings, a small violation of coherence constraints may be acceptable if it enables faster updates. While MinT cannot provide such trade-offs, our network flow formulation admits a principled relaxation with provable error bounds.

\begin{theorem}[Approximate Forecast Reconciliation]
Let $G=(V,E,\mathcal{P})$ be a network with base forecasts $\hat{\mathbf{y}}_{T+h|T}$ optimized for the $\ell_2$ norm. For accuracy threshold $\varepsilon > 0$, there exists an $\varepsilon$-relaxed reconciliation $\tilde{\mathbf{y}}_{\varepsilon,T+h|T}$ satisfying:
\[ \left|\sum_{P \in \mathcal{P}:e \in P} \tilde{\mathbf{y}}_{\varepsilon,P,T+h|T} - \tilde{\mathbf{y}}_{\varepsilon,e,T+h|T}\right| \leq \varepsilon \quad \forall e \in E \]
with bounded deviation from exact reconciliation:
\[ \|\tilde{\mathbf{y}}_{\varepsilon,T+h|T} - \tilde{\mathbf{y}}_{T+h|T}\|_2 \leq \sqrt{\varepsilon|E|}\|\tilde{\mathbf{y}}_{T+h|T}\|_2 \]
\end{theorem}

\begin{proof}
Consider the relaxed optimization:
\[ \min_{\tilde{\mathbf{y}}_{\varepsilon,T+h|T}} \|\tilde{\mathbf{y}}_{\varepsilon,T+h|T} - \hat{\mathbf{y}}_{T+h|T}\|_2^2 \]
subject to:
\[ -\varepsilon \leq \sum_{P \in \mathcal{P}:e \in P} \tilde{\mathbf{y}}_{\varepsilon,P,T+h|T} - \tilde{\mathbf{y}}_{\varepsilon,e,T+h|T} \leq \varepsilon \quad \forall e \in E \]

The Lagrangian is:
\begin{align*}
\mathcal{L}(\tilde{\mathbf{y}}_{\varepsilon}, \boldsymbol{\lambda}^+, \boldsymbol{\lambda}^-) = &\|\tilde{\mathbf{y}}_{\varepsilon,T+h|T} - \hat{\mathbf{y}}_{T+h|T}\|_2^2 + \\
&\sum_{e \in E} \lambda_e^+(\sum_{P \in \mathcal{P}:e \in P} \tilde{\mathbf{y}}_{\varepsilon,P,T+h|T} - \tilde{\mathbf{y}}_{\varepsilon,e,T+h|T} - \varepsilon) + \\
&\sum_{e \in E} \lambda_e^-(-\sum_{P \in \mathcal{P}:e \in P} \tilde{\mathbf{y}}_{\varepsilon,P,T+h|T} + \tilde{\mathbf{y}}_{\varepsilon,e,T+h|T} - \varepsilon)
\end{align*}

The KKT conditions imply that for each edge $e \in E$, either the flow conservation constraint is exactly satisfied or the violation equals $\pm\varepsilon$. Let $\mathcal{E}_\varepsilon$ be the set of edges where the maximum violation occurs. Then:
\[ \|\tilde{\mathbf{y}}_{\varepsilon,T+h|T} - \tilde{\mathbf{y}}_{T+h|T}\|_2^2 = \sum_{e \in \mathcal{E}_\varepsilon} (\pm\varepsilon)^2 \leq \varepsilon^2|E| \]

Moreover:
\[ \|\tilde{\mathbf{y}}_{\varepsilon,T+h|T} - \hat{\mathbf{y}}_{T+h|T}\|_2^2 \leq \|\tilde{\mathbf{y}}_{T+h|T} - \hat{\mathbf{y}}_{T+h|T}\|_2^2 + \varepsilon^2|E| \]
The bound follows from the triangle inequality.
\end{proof}

This result shows that by allowing small violations of flow conservation, we can obtain solutions with provably bounded deviation from exact reconciliation. The error bound $\sqrt{\varepsilon|E|}$ provides precise guidance for setting $\varepsilon$ based on application requirements.

\begin{theorem}[Computational Efficiency of Approximate Reconciliation]
The $\varepsilon$-relaxed solution can be computed in $O(m \log(\frac{1}{\varepsilon}) \log n)$ time using $O(m)$ memory, where $m = |E|$ and $n = |V|$.
\end{theorem}

\begin{proof}
The problem is strongly convex due to the quadratic objective. The gradient of the Lagrangian at any point can be computed in $O(m)$ operations. For step size $\alpha$, projected gradient descent satisfies:
\[ \|\tilde{\mathbf{y}}^{(k+1)}_{\varepsilon,T+h|T} - \tilde{\mathbf{y}}^*_{\varepsilon,T+h|T}\|_2 \leq (1-\alpha\mu)\|\tilde{\mathbf{y}}^{(k)}_{\varepsilon,T+h|T} - \tilde{\mathbf{y}}^*_{\varepsilon,T+h|T}\|_2 \]
where $\mu = 2$ is the strong convexity parameter.

This geometric convergence implies $O(\log(\frac{1}{\varepsilon}))$ iterations to reach error $\varepsilon$. Each iteration requires:
\begin{enumerate}
    \item Gradient computation: $O(m)$
    \item Projection onto box constraints: $O(m)$
    \item Line search: $O(\log n)$ using binary search
\end{enumerate}

The total complexity follows. Memory requirements are $O(m)$ to store the current solution and gradient.
\end{proof}

This result provides practitioners with a principled way to trade off computational efficiency against forecast coherence. The error bound $\sqrt{\varepsilon|E|}$ gives precise guidance for setting $\varepsilon$ based on application requirements. Moreover, the $O(m \log(\frac{1}{\varepsilon}) \log n)$ runtime represents a significant improvement over MinT's $O(n^3)$, particularly for sparse networks where $m \ll n^2$.

\section{Experiments on Simulated Data}\label{app:exp-sim}
% Figures \ref{fig:overall-mae}, \ref{fig:path-mae}, \ref{fig:edge-mae} and \ref{fig:node-mae} , show the performance of \fr against other methods for MAE,  plotted on the log scale on the y-axis. For each metric, we observe that our method of flow reconciliation outperforms both bottom-up and MinT  by a very significant margin,  showing promise for  extension to bigger datasets and more complexes use cases. The error gap increases as the level of the hierarchy increases, meaning all models forecast node values in a similar manner, compared to a larger range across models for path values. 
% We also note that our method improves the base forecast, irrespective of the base forecast model, and the forecast hierarchy level. This is unexpected behavior for a reconciliation method in hierarchical forecasting. Consider Figure \ref{fig:overall-mae}. The bottom up method  consistently increases the error of the base forecast  after reconciliation. MinT shows better performance than bottom-up,  and predicts vales quite close to the base forecast,  and in a few cases shows a mild improvement. This can be explained by the enforced non negativity constraint,  which forces all negative forecast numbers to zero thus limiting outliers.  On the other hand,  Flow can be  at least three times more accurate  than the base forecast in predicting values on edges.
\begin{figure}
    \centering
    \includegraphics[scale=0.45]{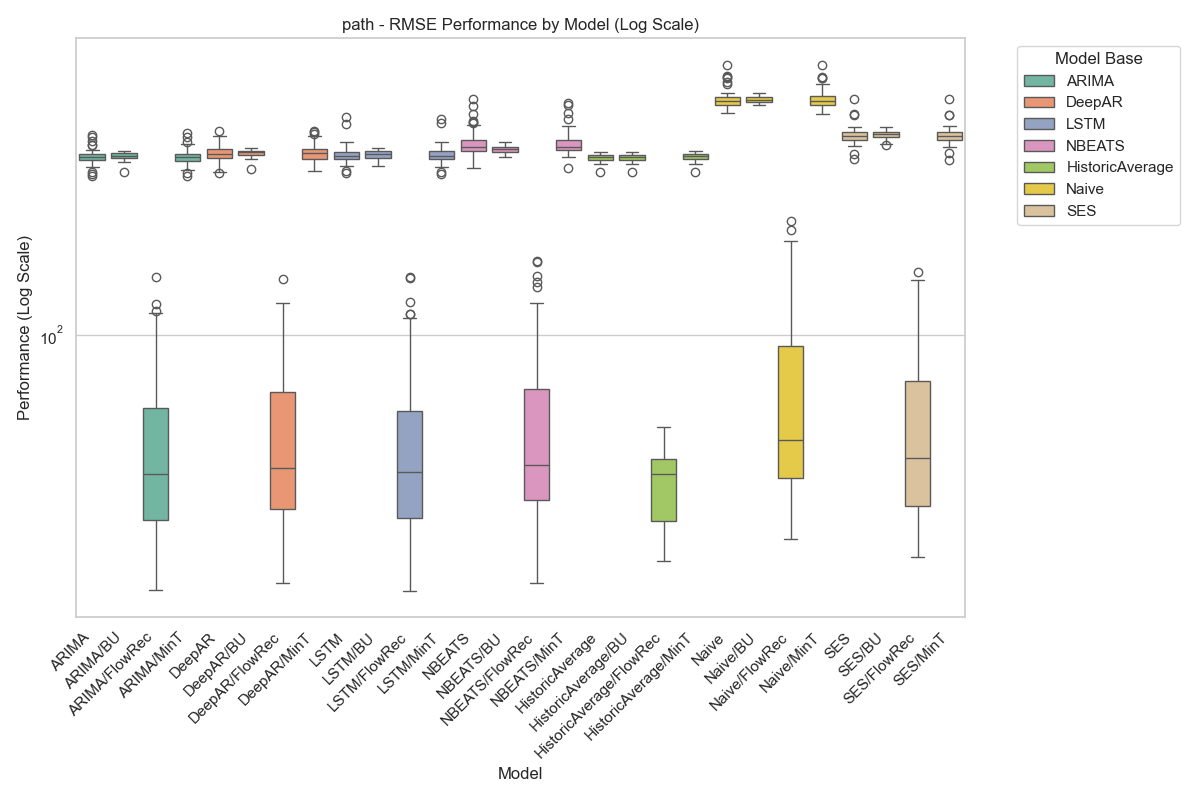}
    \caption{RMSE results for forecasting and reconciliation models across paths for simulated data, shown across seven different base forecasting methods. \fr shows the lowest RMSE (best), while bottom-up shows the highest. MinT shows similar performance to the unreconciled base forecast. }
    \label{fig:path-rmse}
\end{figure}

\begin{figure}
    \centering
    \includegraphics[scale=0.45]{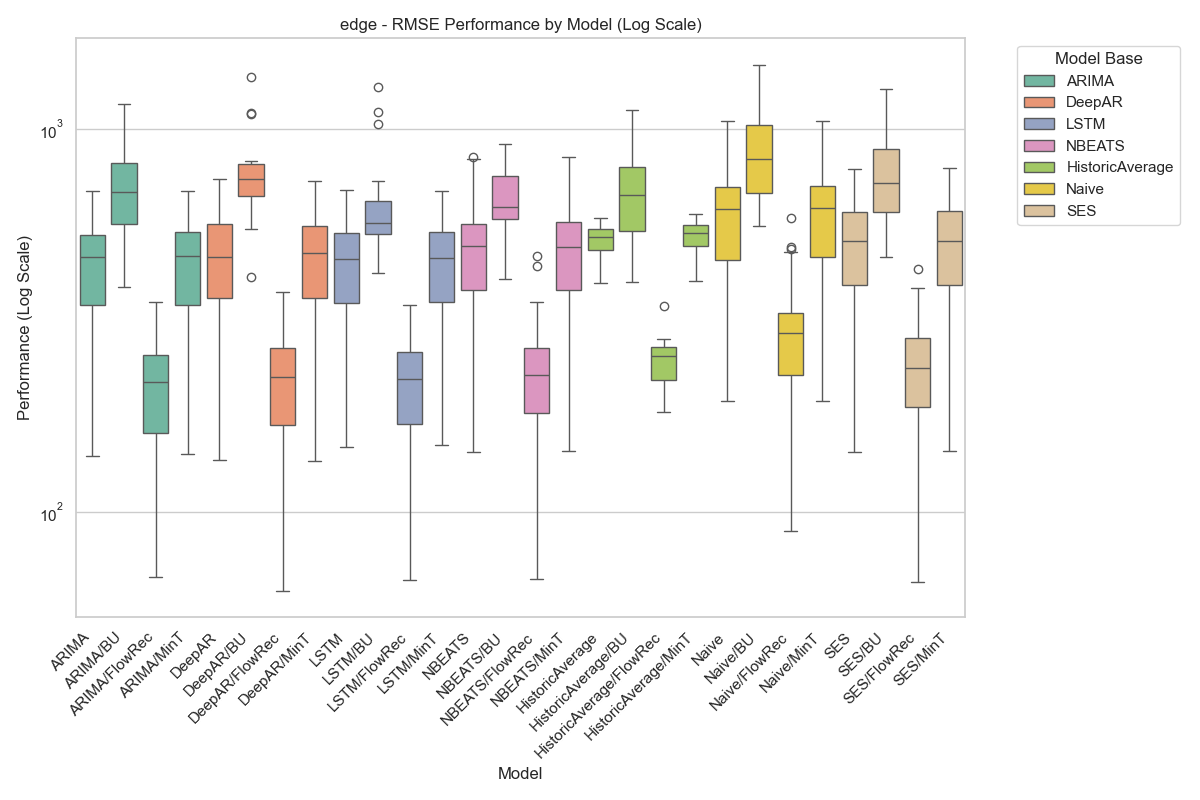}
    \caption{RMSE results for forecasting and reconciliation models across edges for simulated data, shown across seven different base forecasting methods. \fr shows the lowest RMSE (best), while bottom-up shows the highest. MinT shows similar performance to the unreconciled base forecast.}
    \label{fig:edge-rmse}
\end{figure}

\begin{figure}
    \centering
    \includegraphics[scale=0.45]{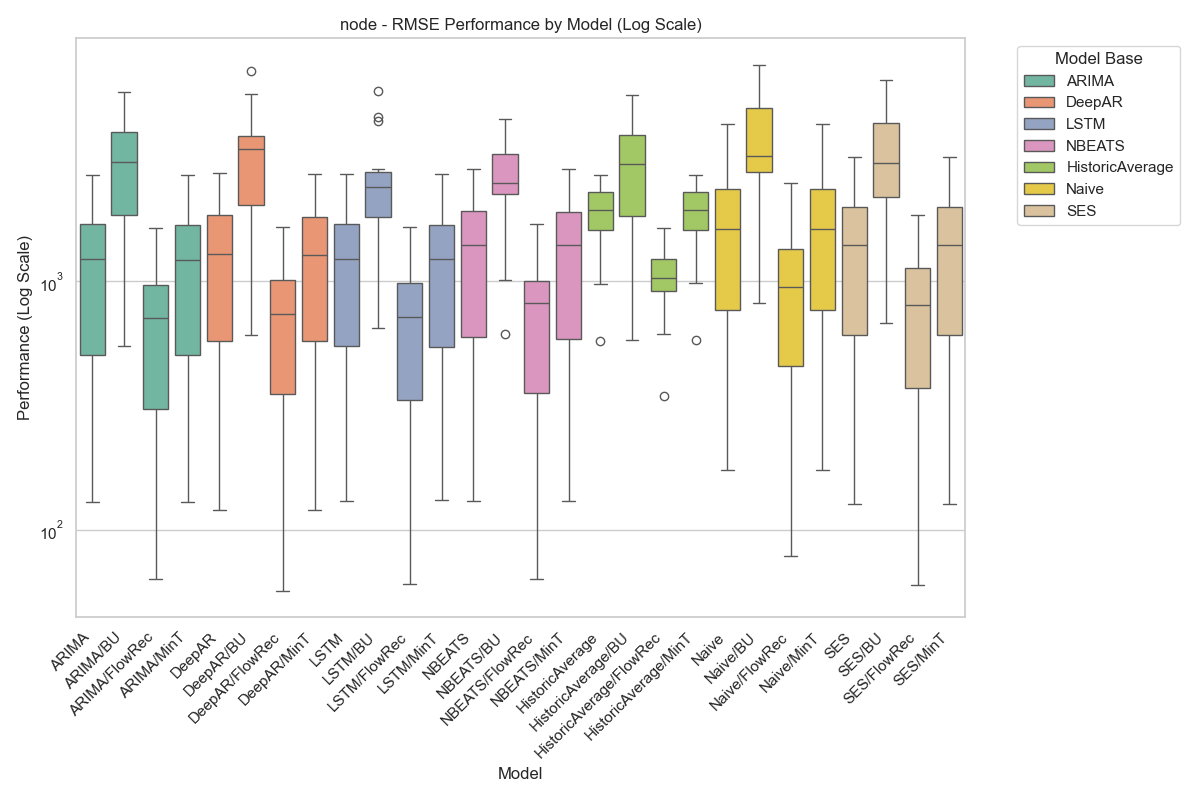}
    \caption{RMSE results for forecasting and reconciliation models across all nodes for simulated data, shown across seven different base forecasting methods. \fr shows the lowest RMSE (best), while bottom-up shows the highest. MinT shows similar performance to the unreconciled base forecast.}
    \label{fig:node-rmse}
\end{figure}

\begin{figure}
    \centering
    \includegraphics[scale=0.45]{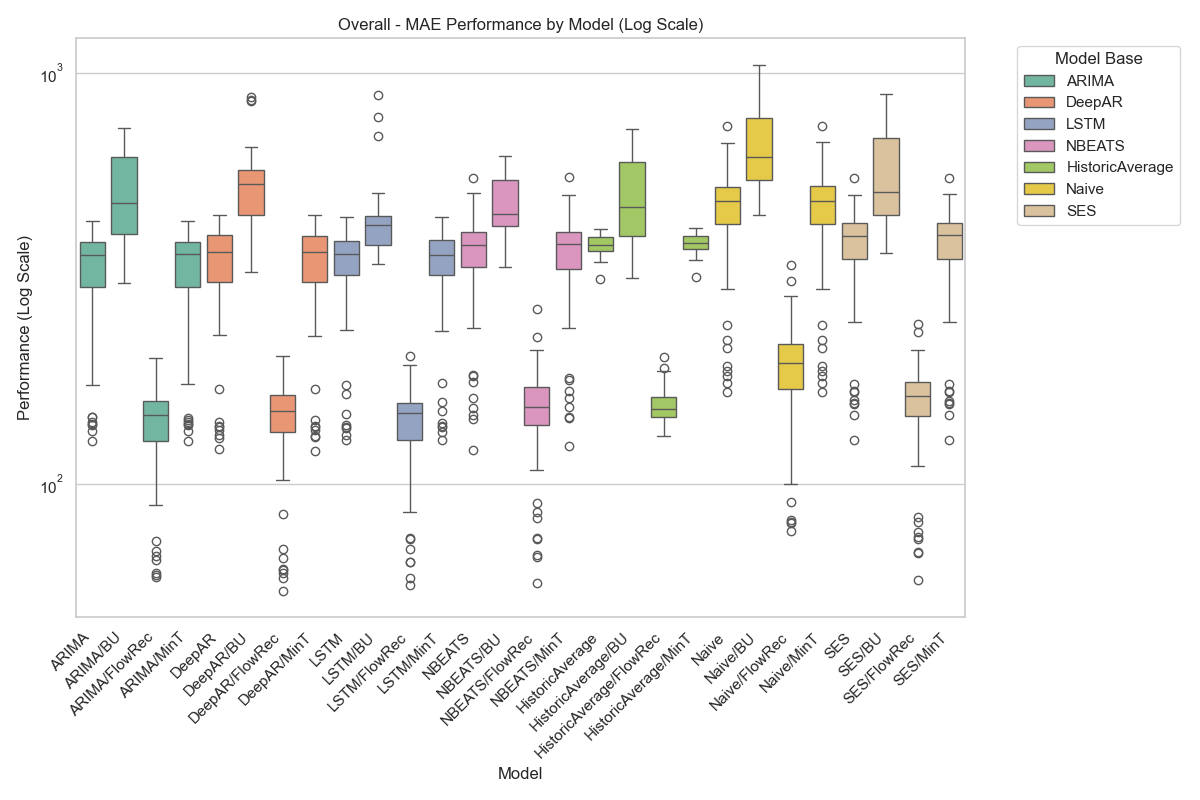}
    \caption{MAE results for forecasting and reconciliation models across all hierarchies for simulated data, shown across seven different base forecasting methods. \fr shows the lowest MAE (best), while bottom-up shows the highest. MinT shows similar performance to the unreconciled base forecast.}
    \label{fig:overall-mae}
\end{figure}

\begin{figure}
    \centering
    \includegraphics[scale=0.45]{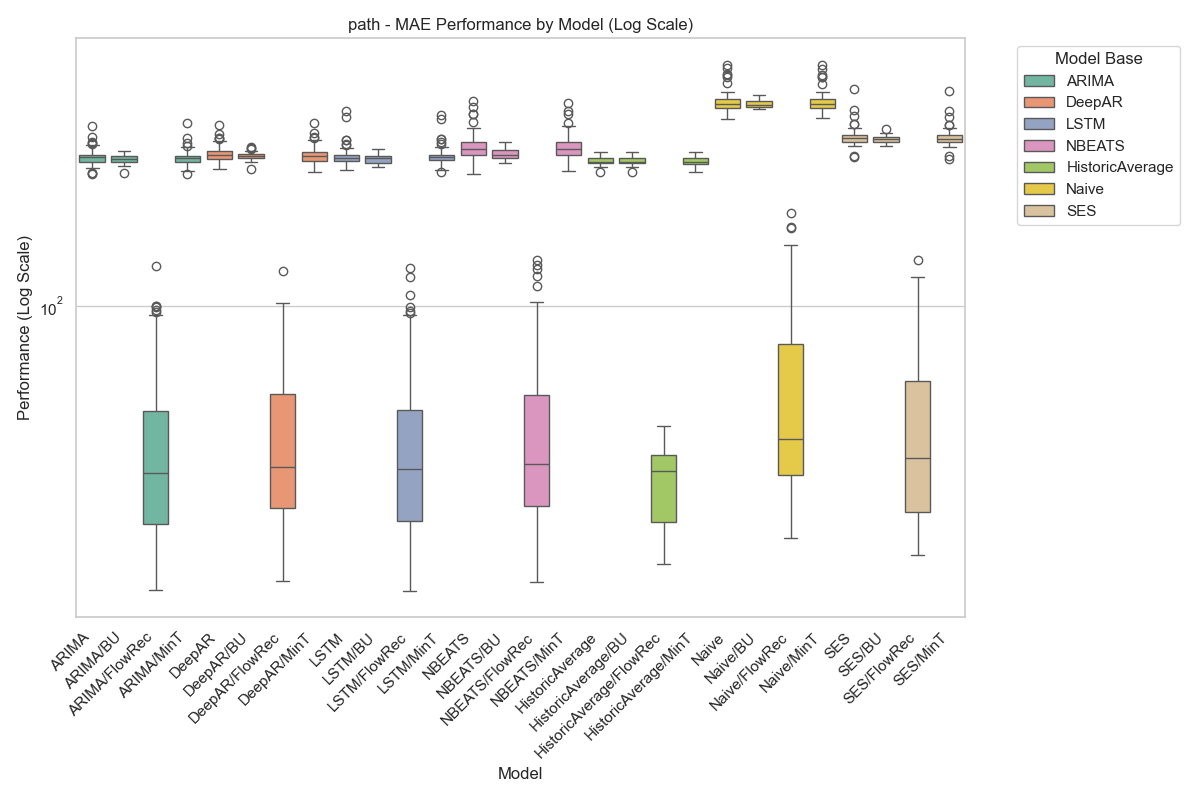}
    \caption{MAE results for forecasting and reconciliation models across paths for simulated data, shown across seven different base forecasting methods. \fr shows the lowest MAE (best), while bottom-up shows the highest. MinT shows similar performance to the unreconciled base forecast.}
    \label{fig:path-mae}
\end{figure}

\begin{figure}
    \centering
    \includegraphics[scale=0.45]{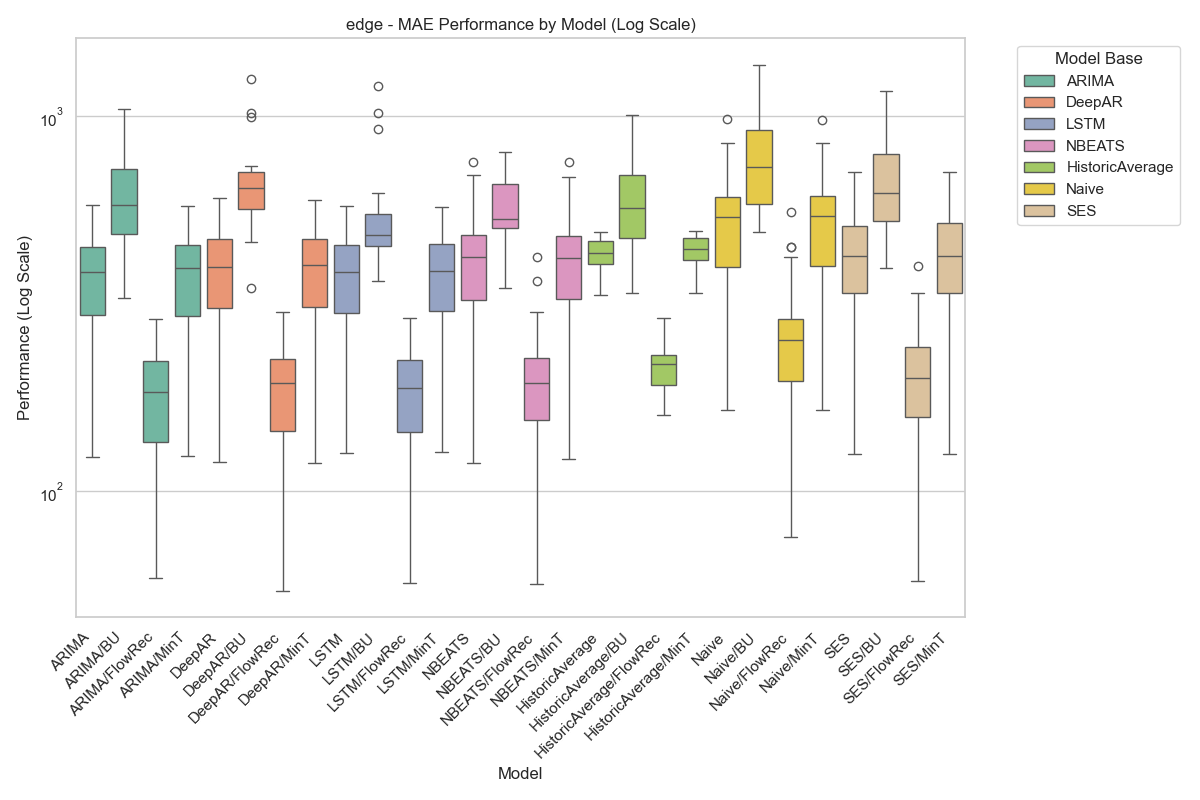}
    \caption{MAE results for forecasting and reconciliation models across edges for simulated data, shown across seven different base forecasting methods. \fr shows the lowest MAE (best), while bottom-up shows the highest. MinT shows similar performance to the unreconciled base forecast.}
    \label{fig:edge-mae}
\end{figure}

\begin{figure}
    \centering
    \includegraphics[scale=0.45]{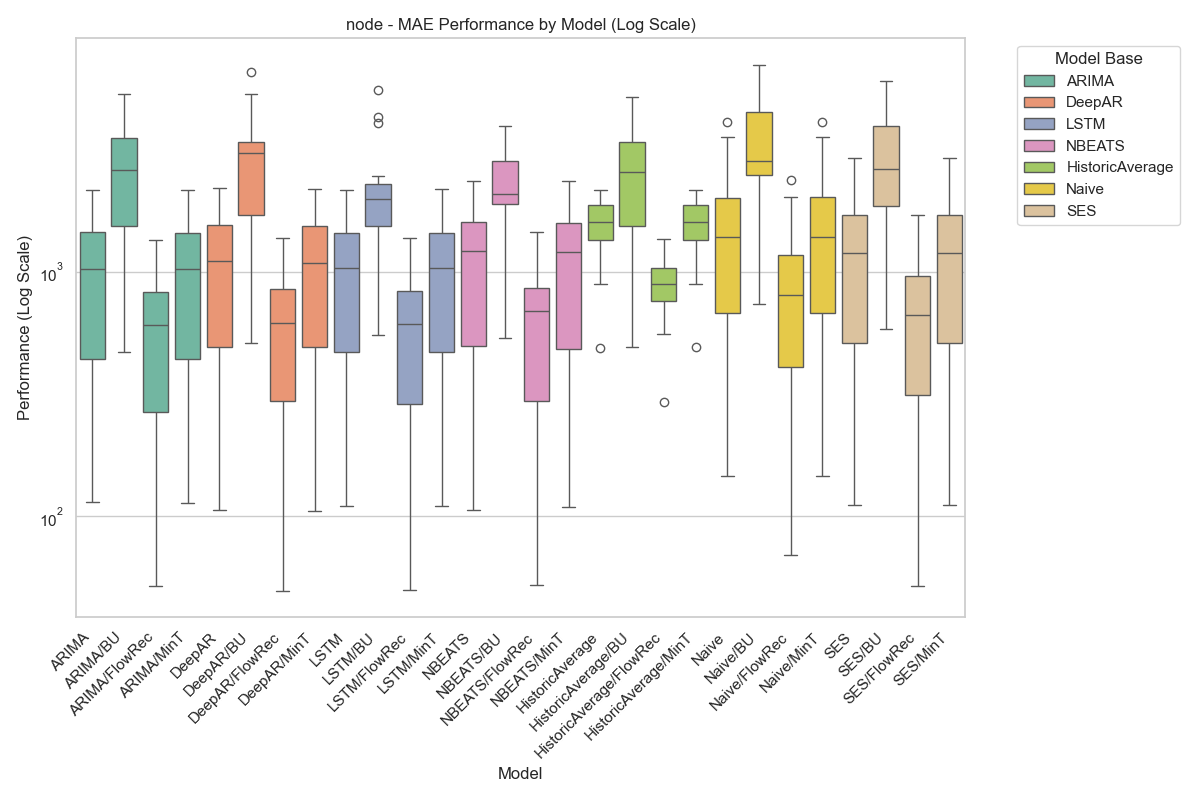}
    \caption{MAE results for forecasting and reconciliation models across nodes for simulated data, shown across seven different base forecasting methods. \fr shows the lowest MAE (best), while bottom-up shows the highest. MinT shows similar performance to the unreconciled base forecast.}
    \label{fig:node-mae}
\end{figure}

\begin{figure}
    \centering
    \includegraphics[scale=0.35]{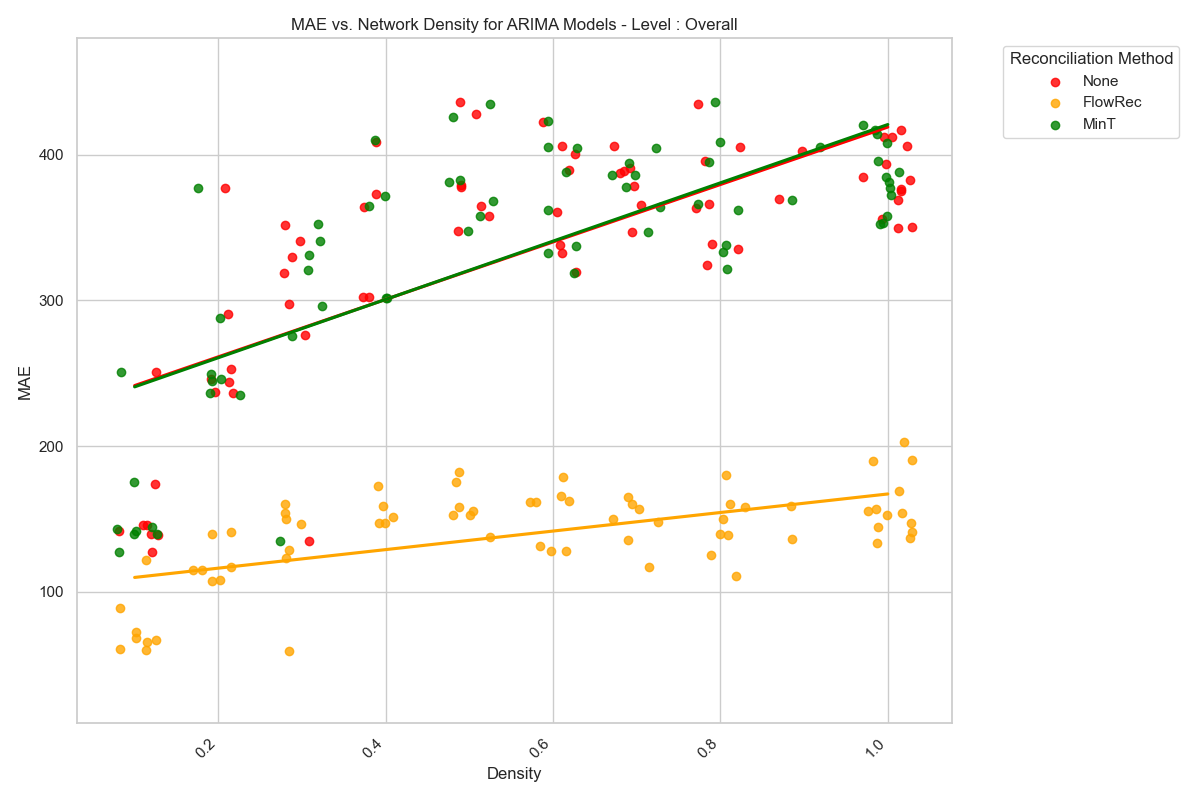}
    \caption{MAE results for ARIMA, plotted against the density (degree of nodes, where 1 means fully connected) of the simulated networks. }
    \label{fig:overall-density}
\end{figure}

\begin{figure}
    \centering
    \includegraphics[scale=0.35]{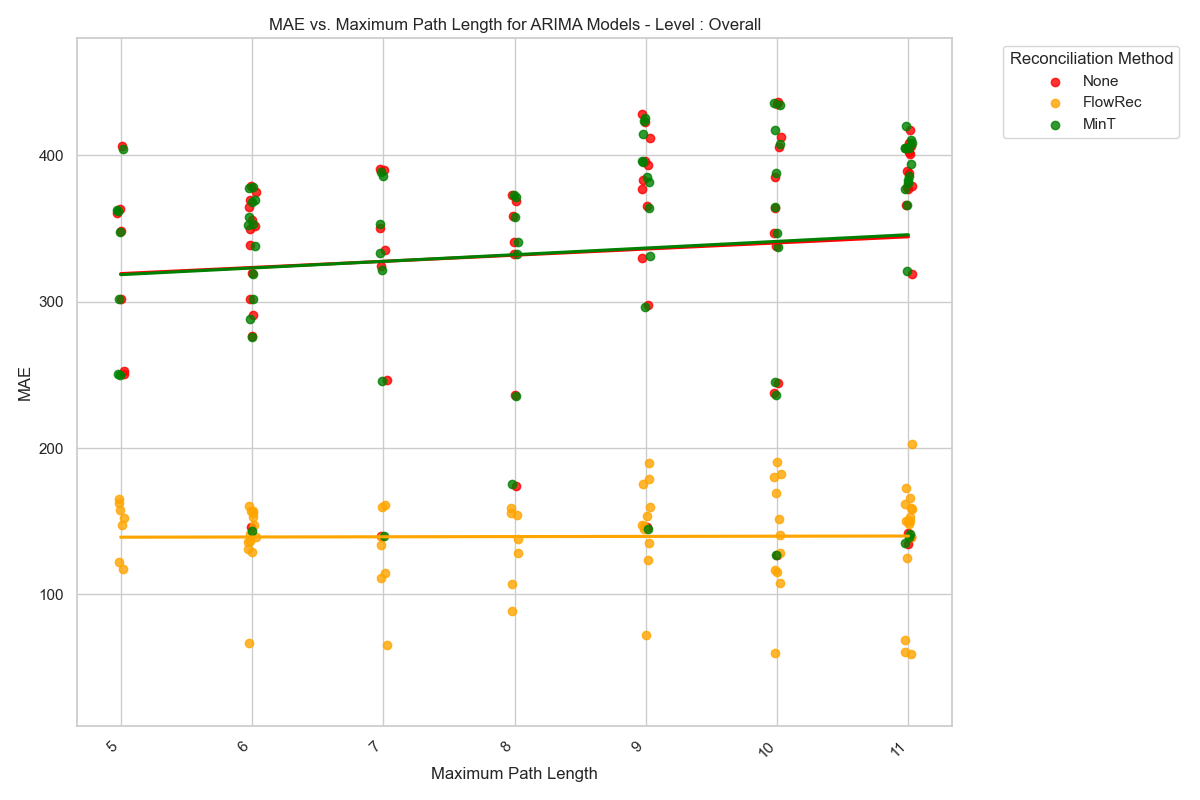}
    \caption{MAE results for ARIMA, plotted against the maximum path length of the simulated networks. }
    \label{fig:overall-length}
\end{figure}

\begin{figure}
    \centering
    \includegraphics[scale=0.35]{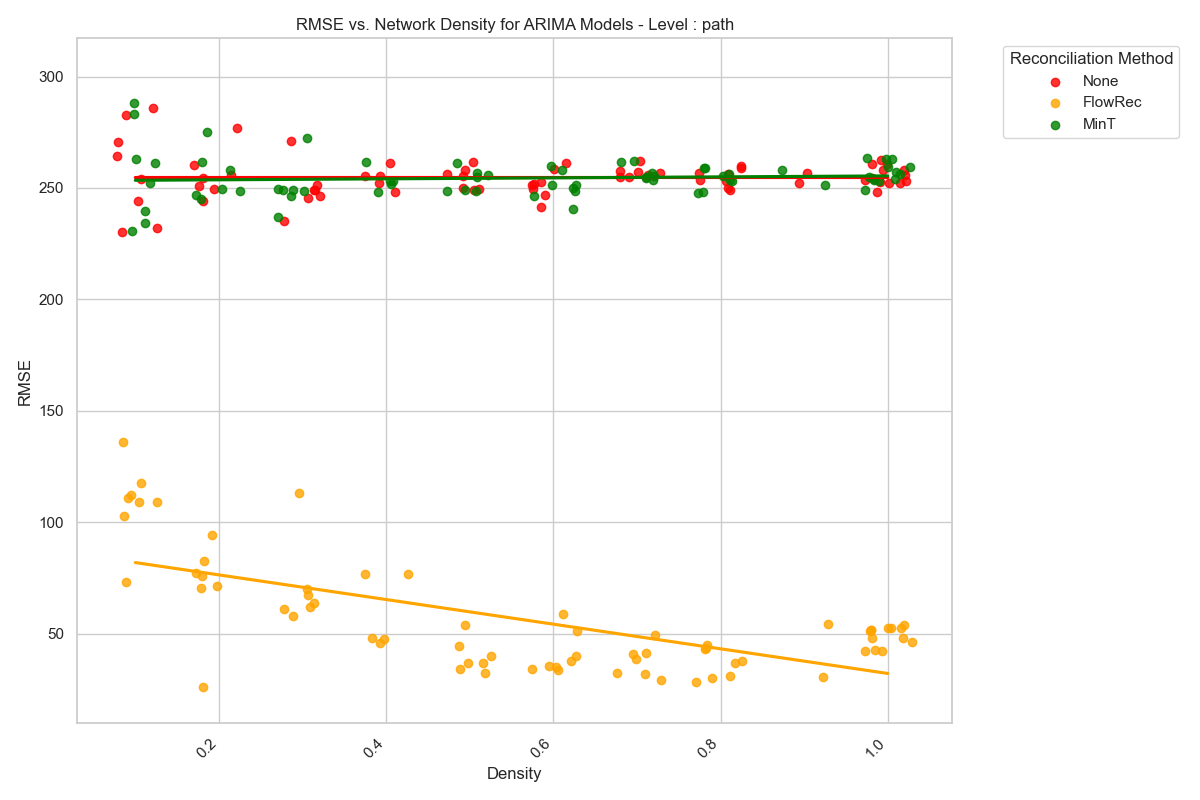}
    \caption{RMSE results for ARIMA, plotted against the density (degree of nodes, where 1 means fully connected) of the simulated networks for paths. }
    \label{fig:overall-density}
\end{figure}

\begin{figure}
    \centering
    \includegraphics[scale=0.35]{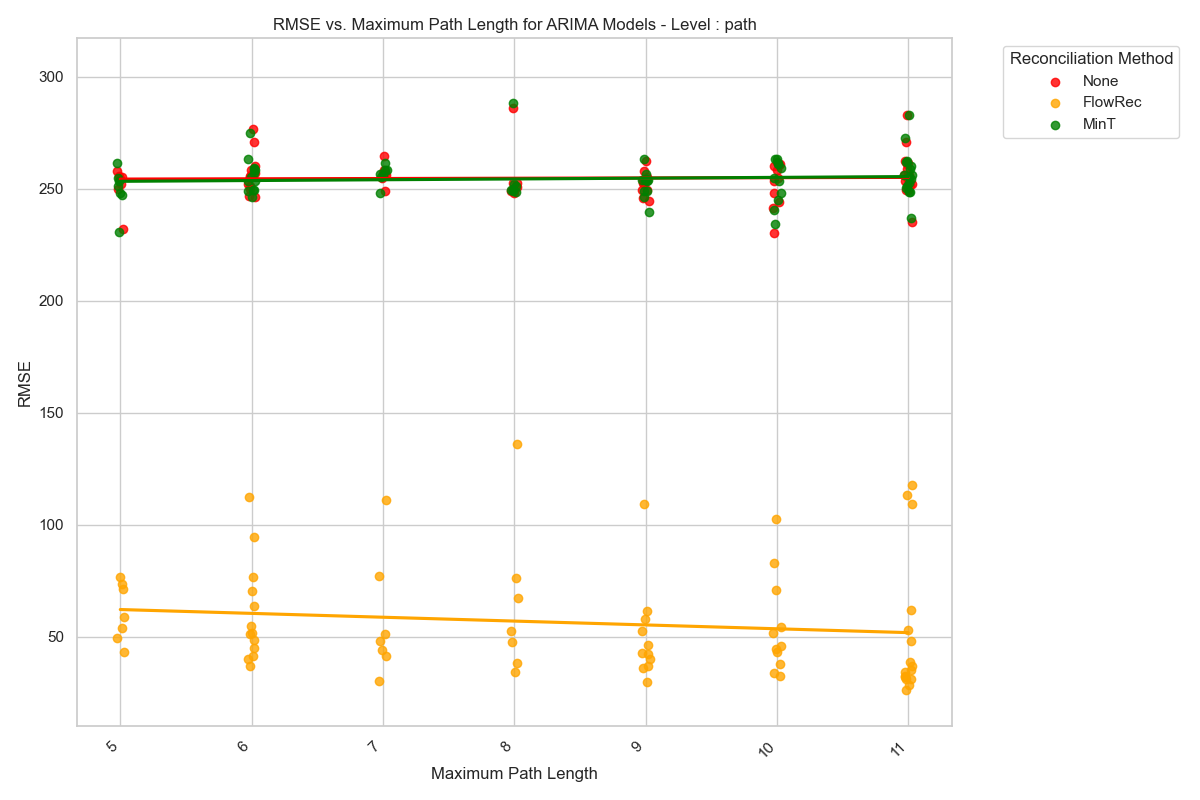}
    \caption{RMSE results for ARIMA for paths, plotted against the maximum path length of the simulated networks. }
    \label{fig:overall-length}
\end{figure}

\begin{figure}
    \centering
    \includegraphics[scale=0.35]{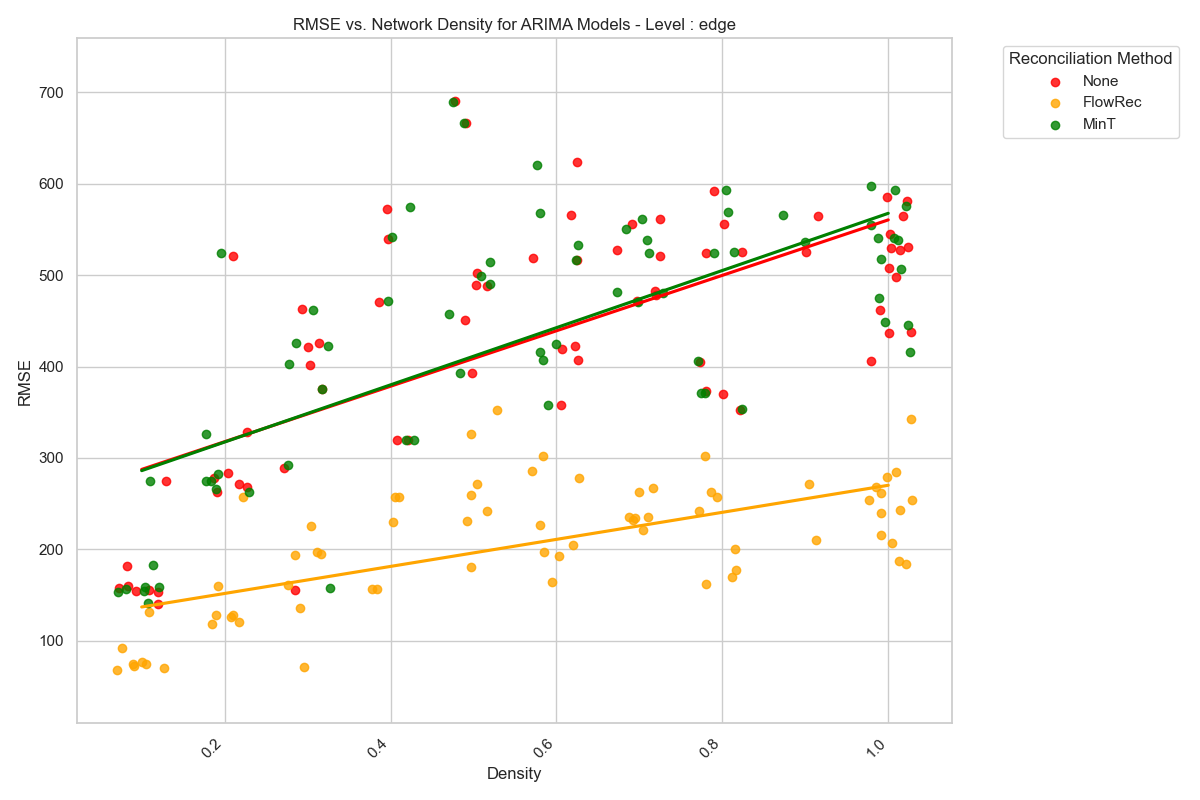}
    \caption{RMSE results for ARIMA for edges, plotted against the density (degree of nodes, where 1 means fully connected) of the simulated networks. }
    \label{fig:overall-density}
\end{figure}

\begin{figure}
    \centering
    \includegraphics[scale=0.35]{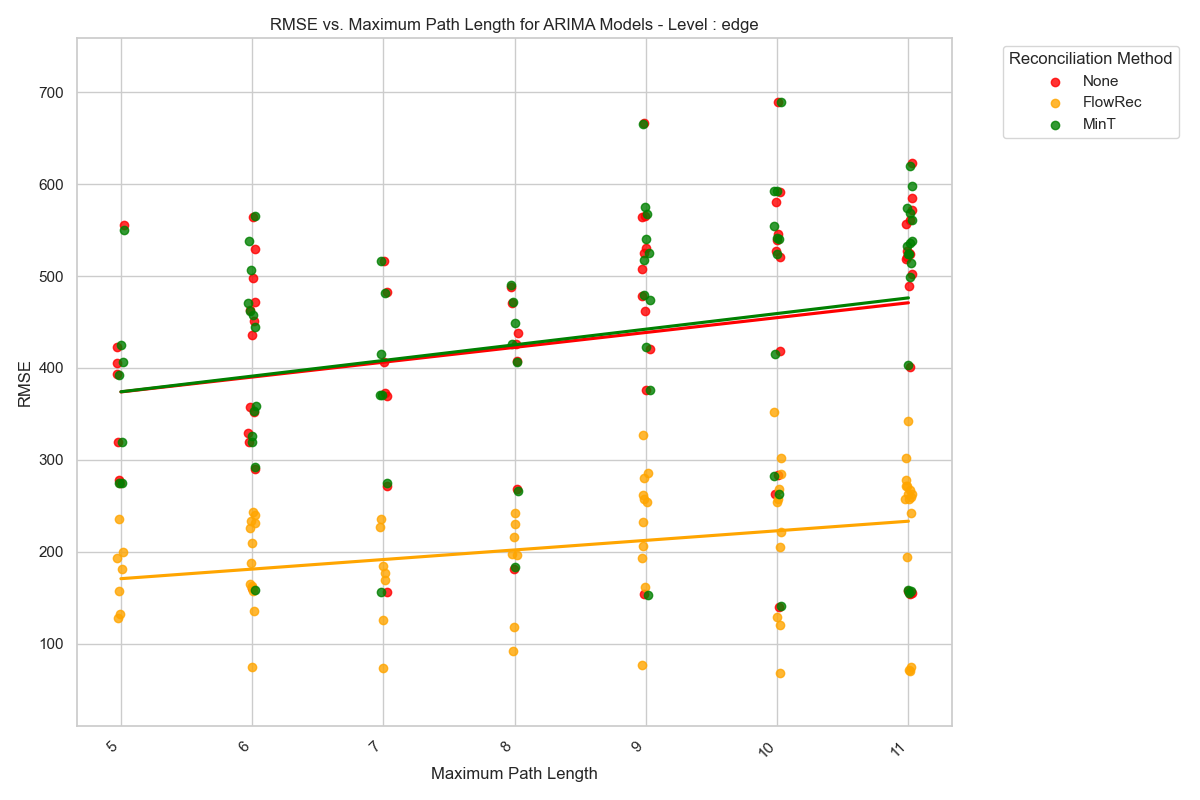}
    \caption{RMSE results for ARIMA for edges, plotted against the maximum path length of the simulated networks. }
    \label{fig:overall-length}
\end{figure}

\begin{figure}
    \centering
    \includegraphics[scale=0.35]{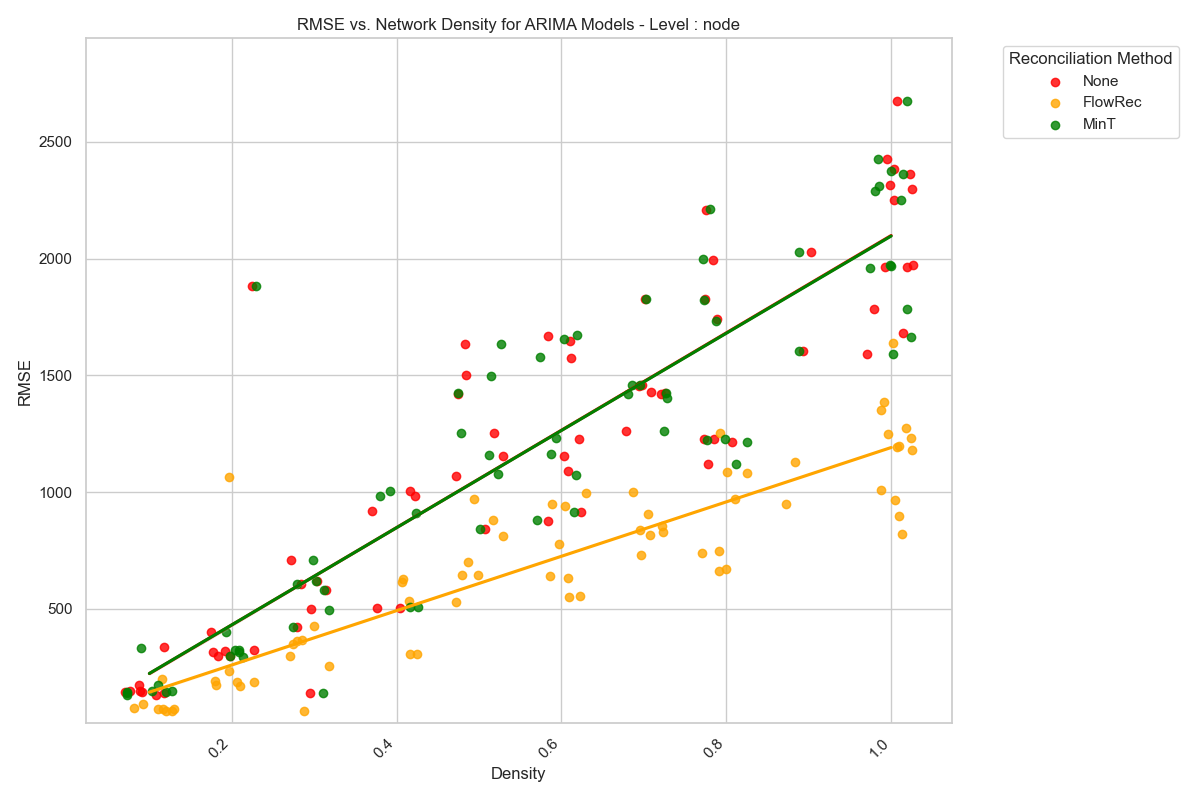}
    \caption{RMSE results for ARIMA for nodes, plotted against the density (degree of nodes, where 1 means fully connected) of the simulated networks. }
    \label{fig:overall-density}
\end{figure}

\begin{figure}
    \centering
    \includegraphics[scale=0.35]{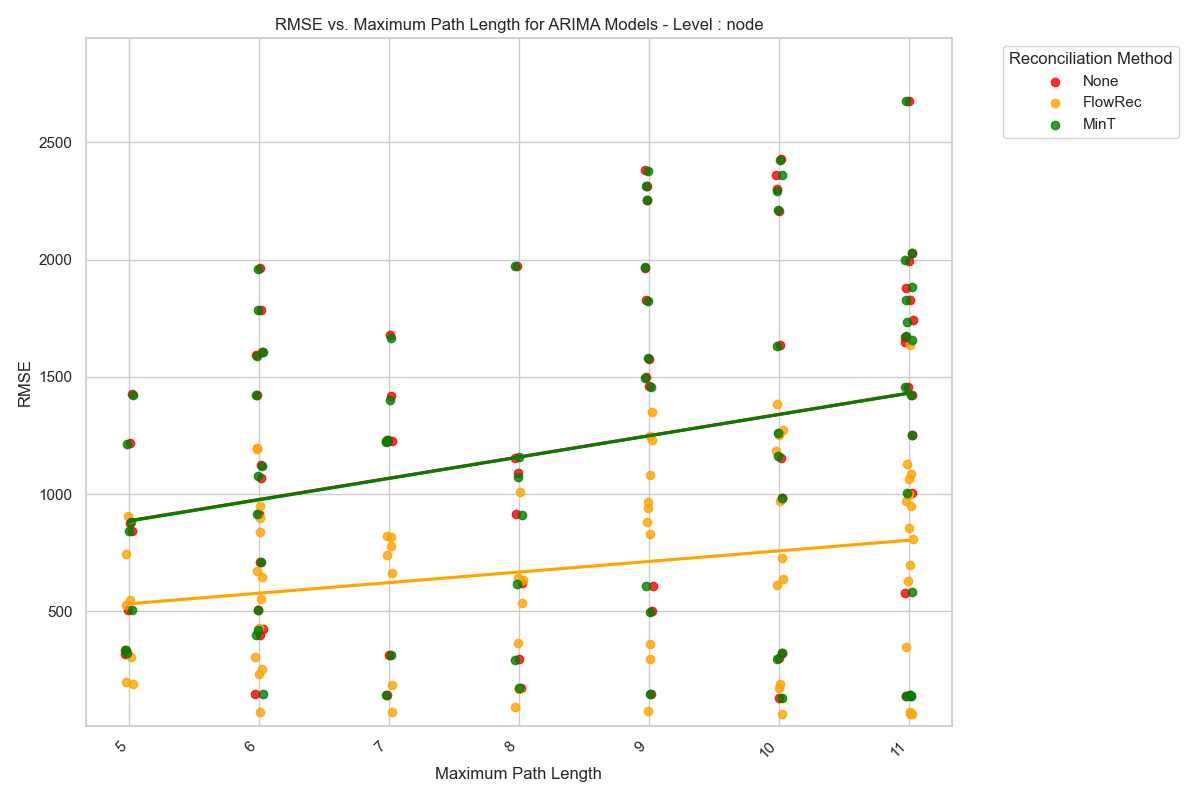}
    \caption{RMSE results for ARIMA for nodes, plotted against the maximum path length of the simulated networks. }
    \label{fig:overall-length}
\end{figure}

\begin{figure}
    \centering
    \includegraphics[scale=0.35]{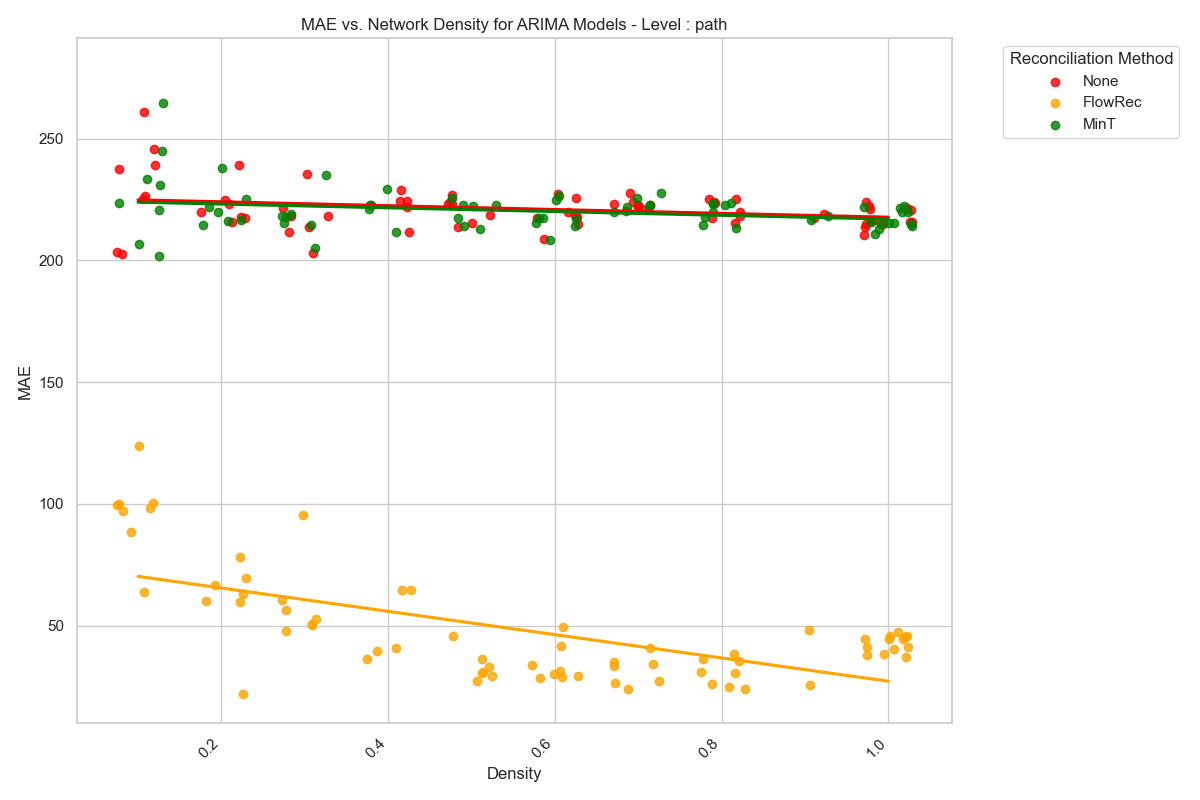}
    \caption{MAE results for ARIMA, plotted against the density (degree of nodes, where 1 means fully connected) of the simulated networks for paths. }
    \label{fig:overall-density}
\end{figure}

\begin{figure}
    \centering
    \includegraphics[scale=0.35]{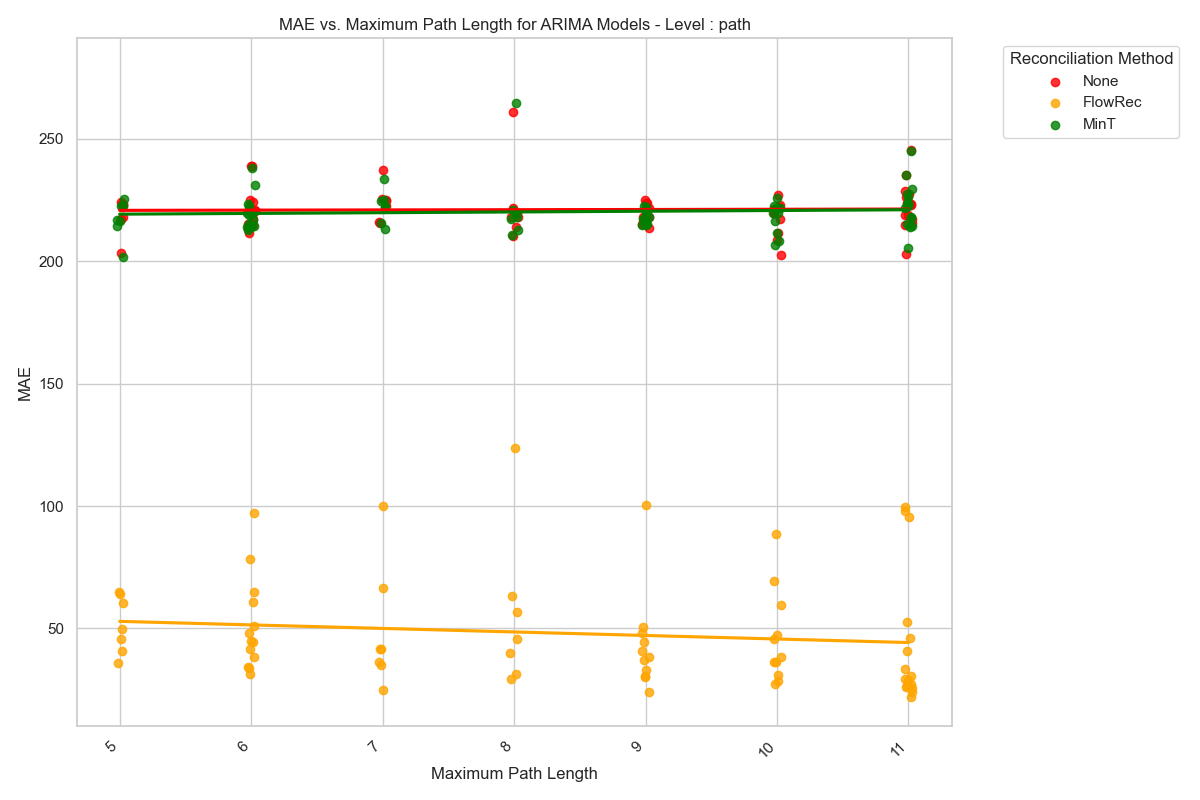}
    \caption{MAE results for ARIMA for paths, plotted against the maximum path length of the simulated networks. }
    \label{fig:overall-length}
\end{figure}

\begin{figure}
    \centering
    \includegraphics[scale=0.35]{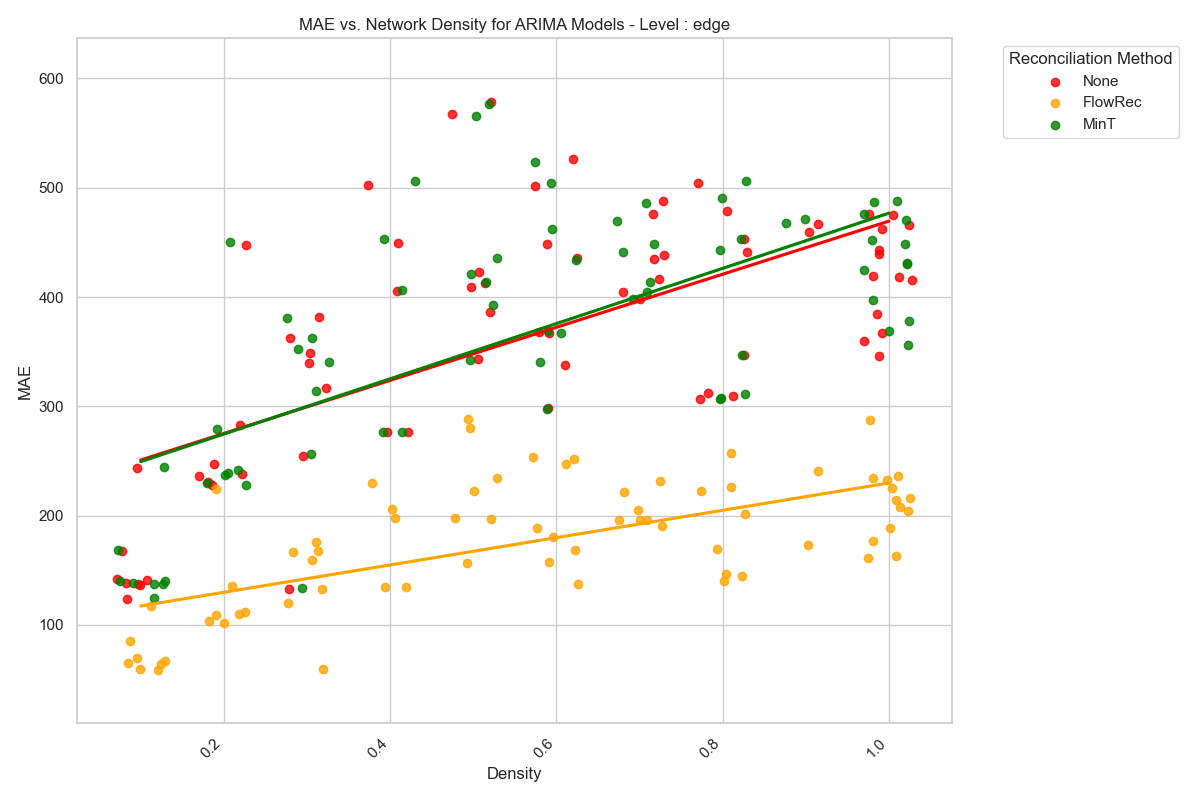}
    \caption{MAE results for ARIMA for edges, plotted against the density (degree of nodes, where 1 means fully connected) of the simulated networks. }
    \label{fig:overall-density}
\end{figure}

\begin{figure}
    \centering
    \includegraphics[scale=0.35]{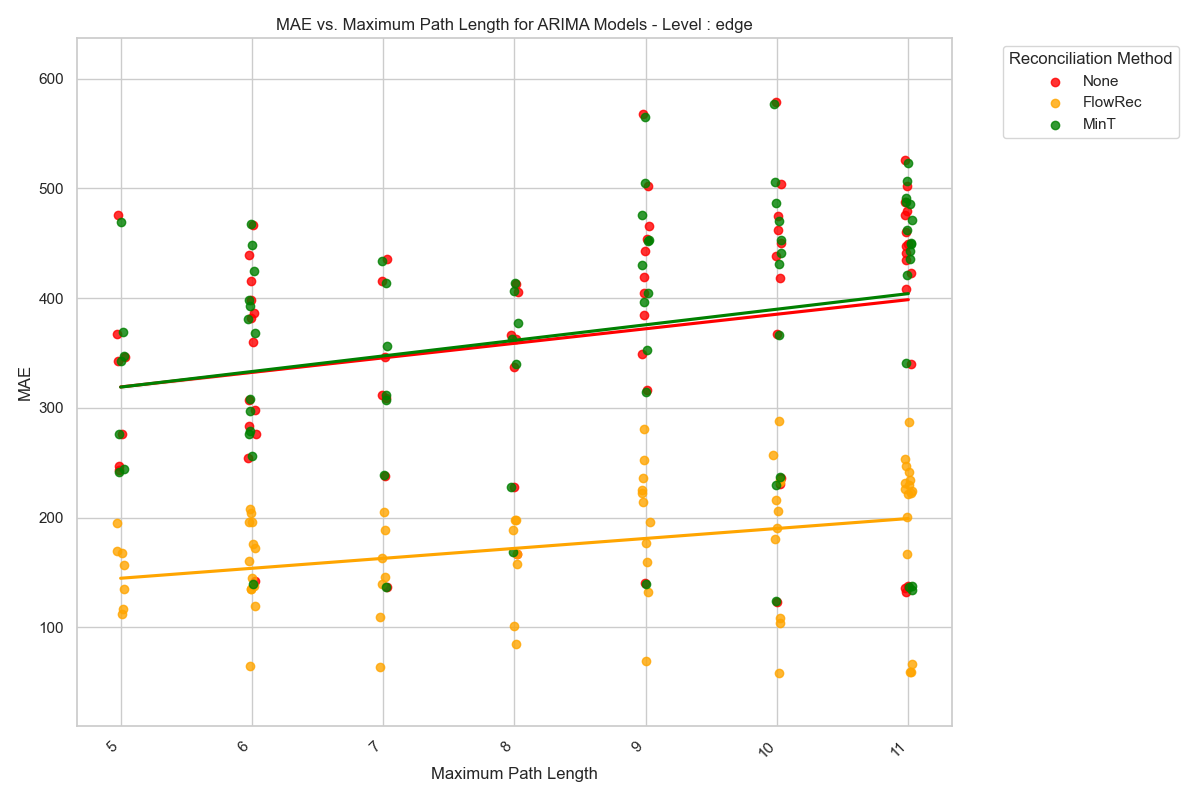}
    \caption{MAE results for ARIMA for edges, plotted against the maximum path length of the simulated networks. }
    \label{fig:overall-length}
\end{figure}

\begin{figure}
    \centering
    \includegraphics[scale=0.35]{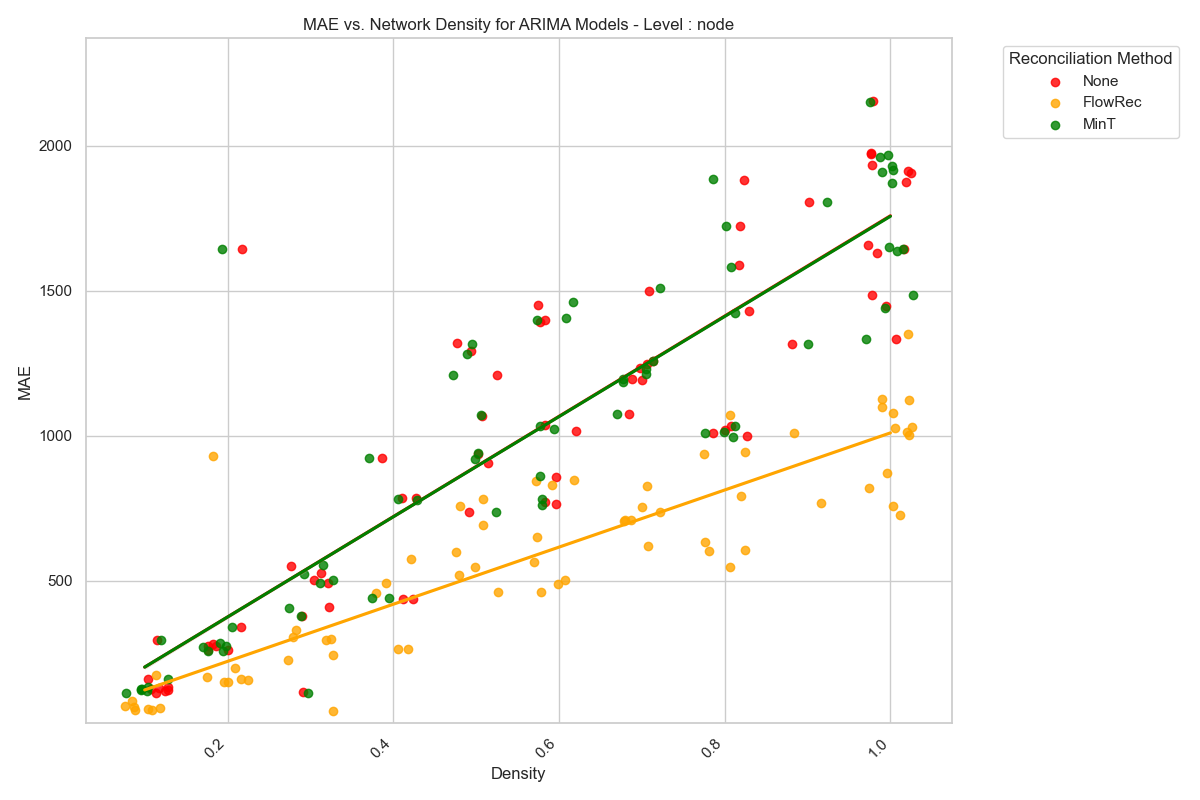}
    \caption{MAE results for ARIMA for nodes, plotted against the density (degree of nodes, where 1 means fully connected) of the simulated networks. }
    \label{fig:overall-density}
\end{figure}

\begin{figure}
    \centering
    \includegraphics[scale=0.35]{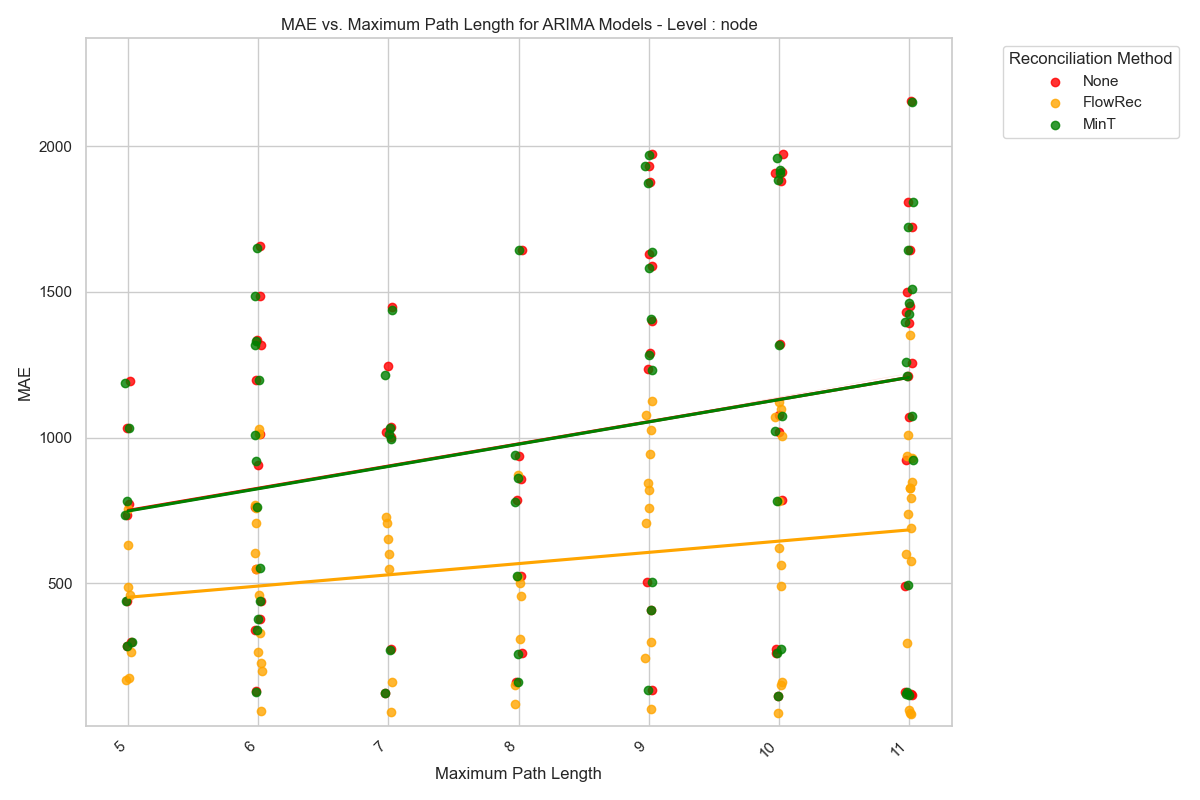}
    \caption{MAE results for ARIMA for nodes, plotted against the maximum path length of the simulated networks. }
    \label{fig:overall-length}
\end{figure}
\section{Experiments on Real Data}\label{app:exp-real}

\begin{figure}
    \centering
    \includegraphics[scale=0.4]{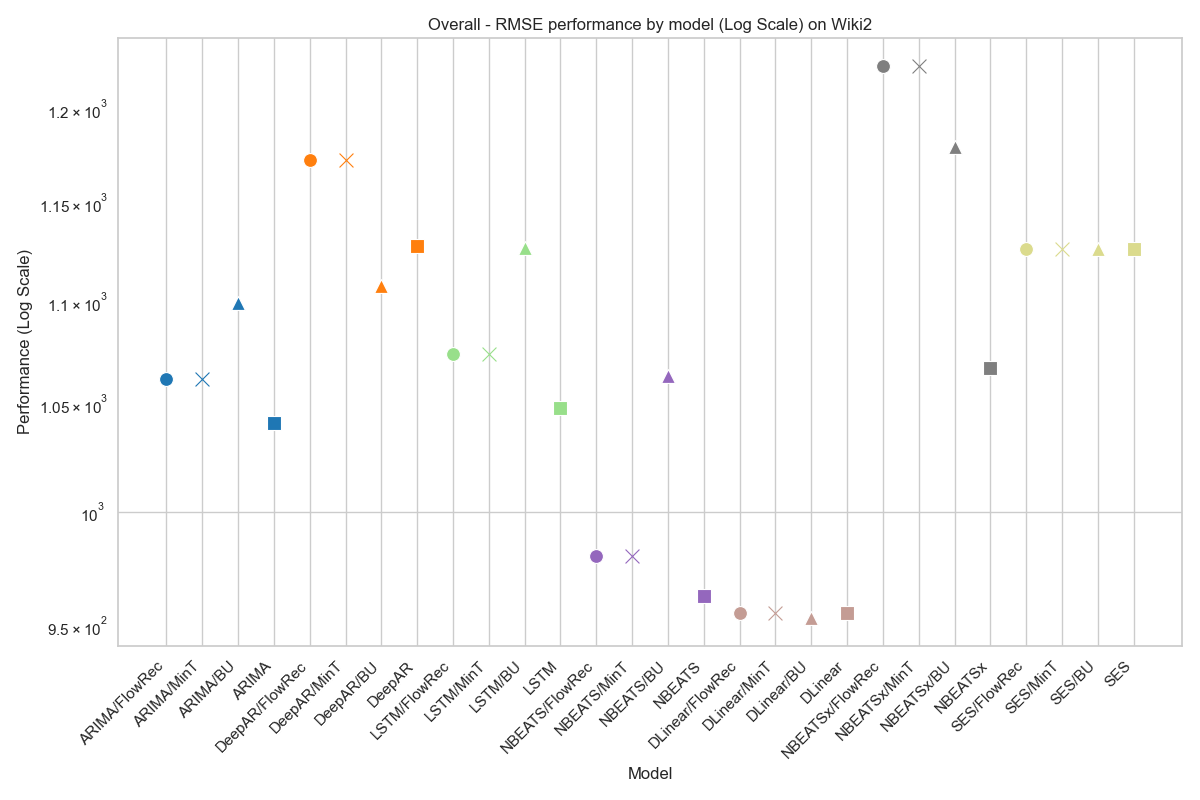}
    \caption{RMSE results for forecasting and reconciliation models across all hierarchies.}
    \label{fig:rmse-wiki}
\end{figure}

\begin{figure}
    \centering
    \includegraphics[scale=0.4]{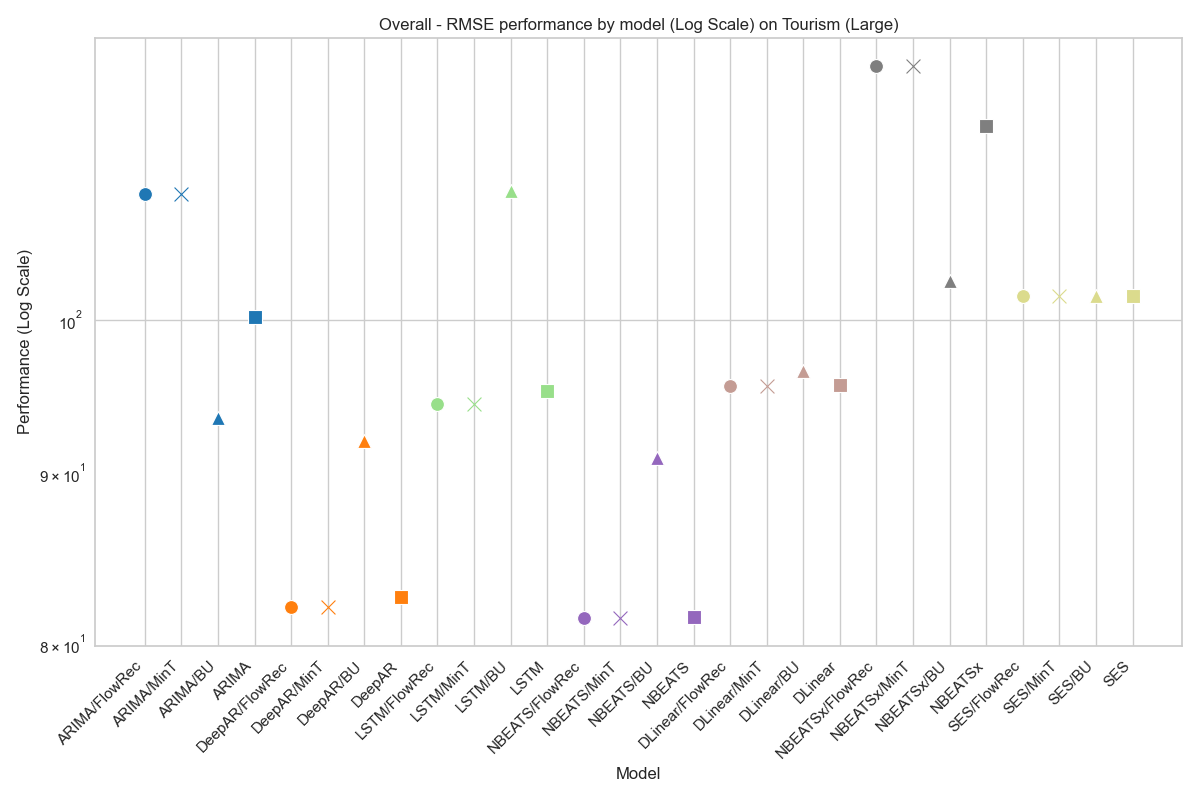}
    \caption{RMSE results for forecasting and reconciliation models across all hierarchies.}
    \label{fig:rmse-tourl}
\end{figure}

\begin{figure}
    \centering
    \includegraphics[scale=0.45]{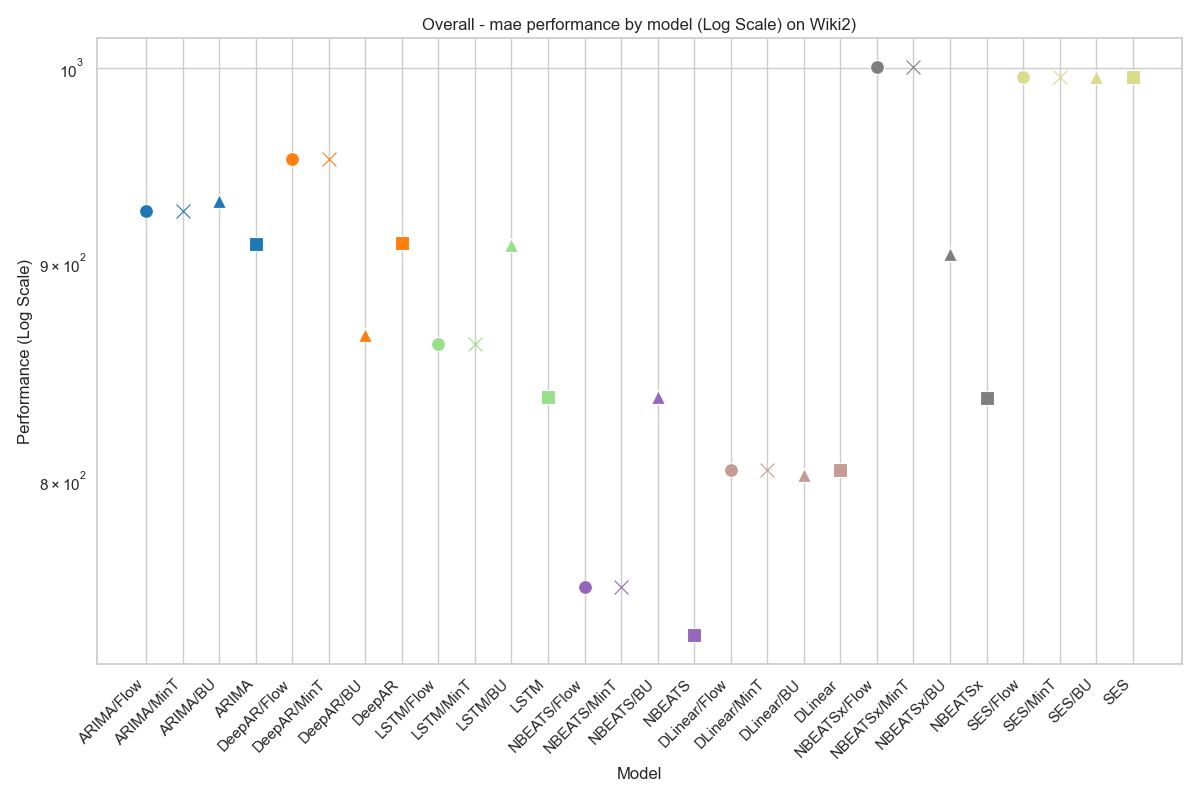}
    \caption{RMSE results for forecasting and reconciliation models across all hierarchies.}
    \label{fig:mae-wiki}
\end{figure}

\begin{figure}
    \centering
    \includegraphics[scale=0.45]{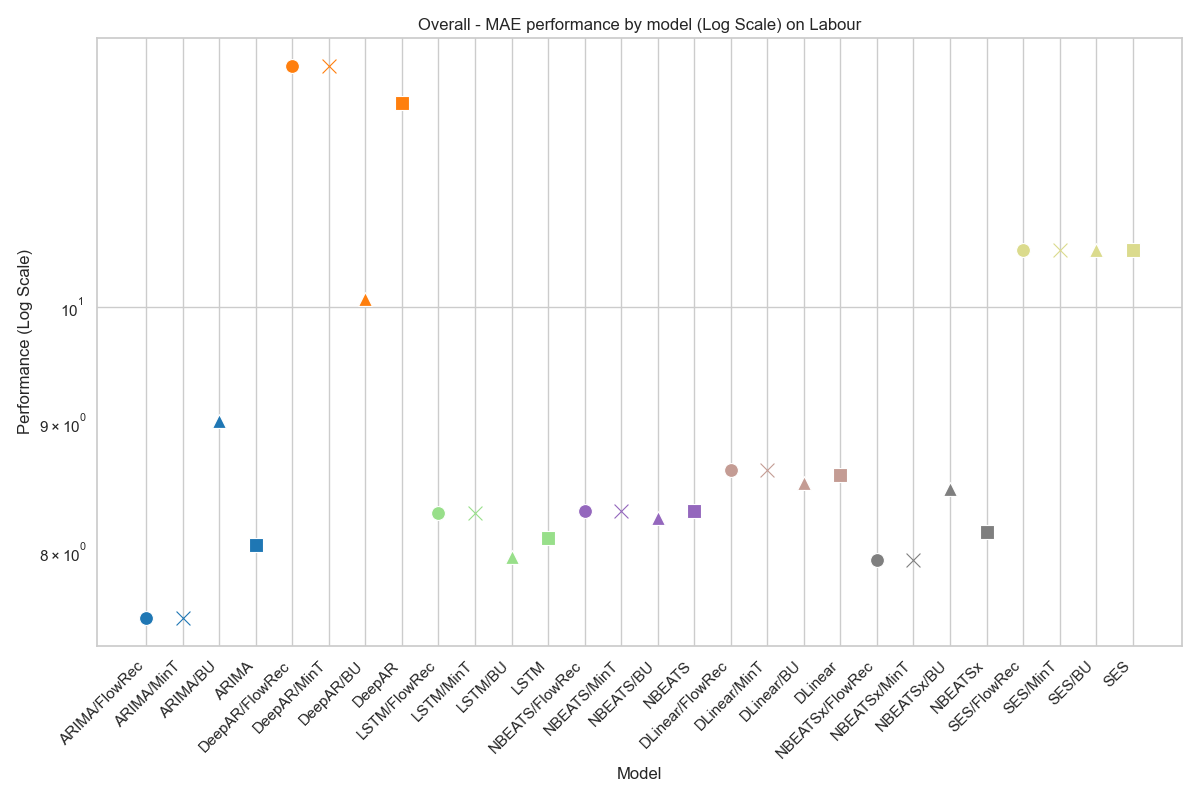}
    \caption{RMSE results for forecasting and reconciliation models across all hierarchies.}
    \label{fig:mae-labour}
\end{figure}

\begin{figure}
    \centering
    \includegraphics[scale=0.45]{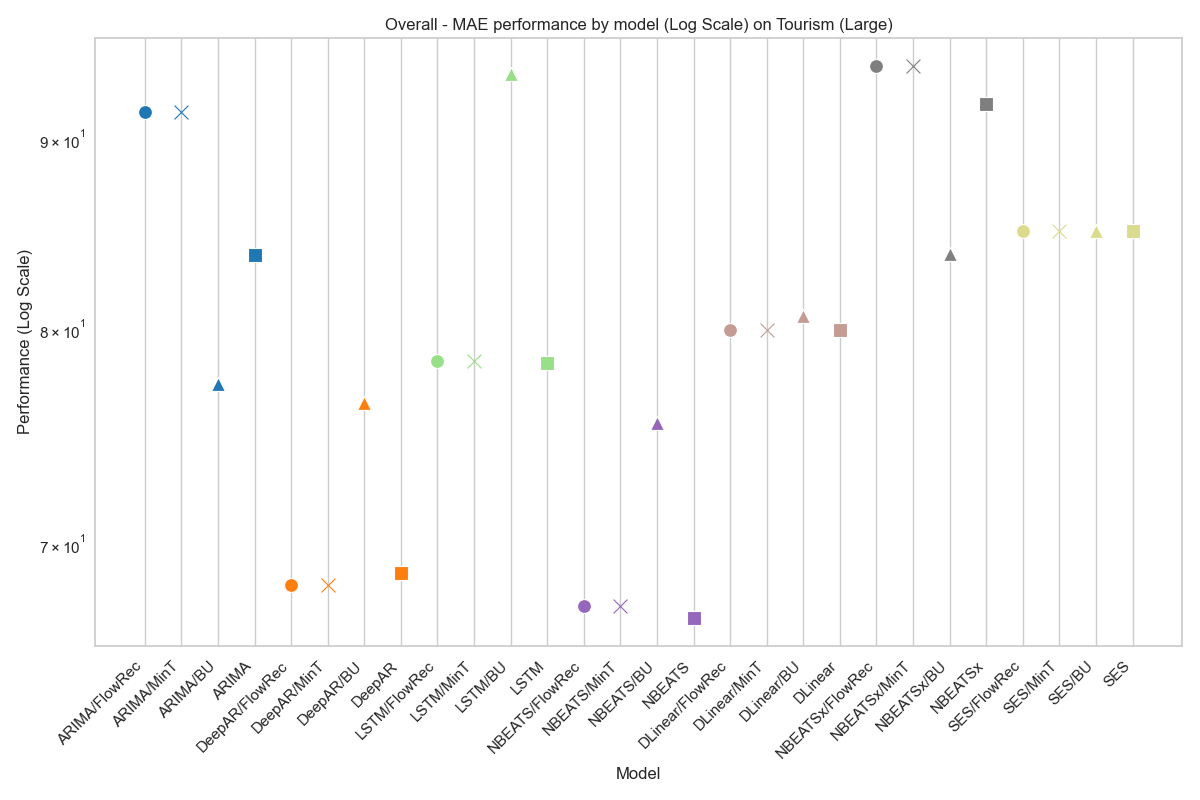}
    \caption{RMSE results for forecasting and reconciliation models across all hierarchies.}
    \label{fig:mae-tourl}
\end{figure}

\begin{figure}
    \centering
    \includegraphics[scale=0.45]{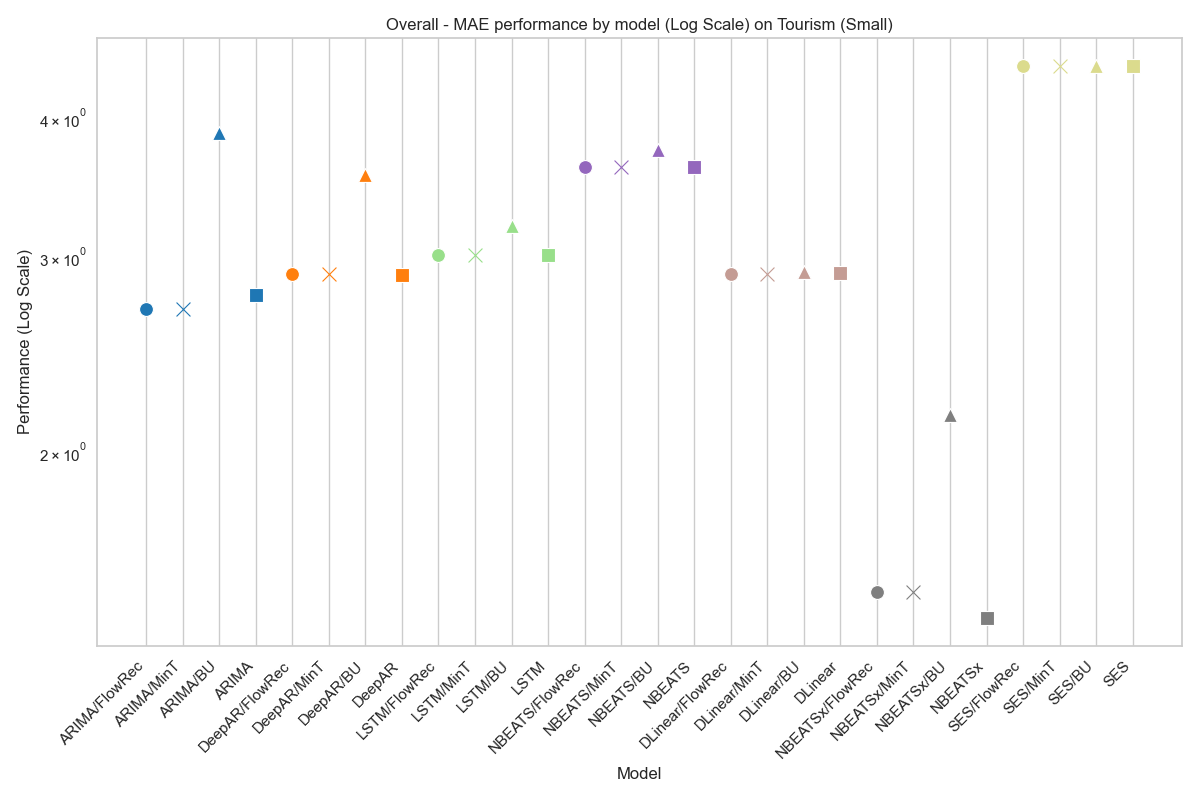}
    \caption{RMSE results for forecasting and reconciliation models across all hierarchies.}
    \label{fig:mae-tours}
\end{figure}
\end{document}